\RequirePackage{fix-cm}
\RequirePackage{amsmath}
\documentclass[fleqn,twocolumn]{svjour3}          
\smartqed  
\usepackage{microtype}

\usepackage{amsmath,amssymb,amsthm}
\usepackage{mathtools}
\usepackage{graphicx}  
\usepackage{amsfonts}
\usepackage{multirow}
\usepackage{silence}
\usepackage{subcaption}
\usepackage{natbib}
\usepackage{longtable}
\usepackage{array}
\usepackage{pifont}
\usepackage{txfonts}
\usepackage{hyperref,lineno}
\usepackage{color}
\DeclareGraphicsExtensions{.jpg,.eps,.png,.pdf,.mps,.gif}
\usepackage{url}
\usepackage{epstopdf}
\usepackage{xcolor}
\usepackage{breqn}
\usepackage{booktabs}
\usepackage{longtable}

\usepackage{subfiles}
\usepackage{wrapfig}
\usepackage{bm}

\usepackage{esvect}
\usepackage{comment}

\makeatletter
\newcommand\footnoteref[1]{\protected@xdef\@thefnmark{\ref{#1}}\@footnotemark}
\makeatother

\newcommand{\be}{\begin{equation}}
\newcommand{\ee}{\end{equation}}
\newcommand{\bbm}{\begin{bmatrix}}
\newcommand{\ebm}{\end{bmatrix}}

\DeclareSymbolFont{perso}{U}{cmr}{bx}{n}
\DeclareMathSymbol{\varbPsi}{\mathalpha}{perso}{"09}
\DeclareMathSymbol{\varbPhi}{\mathalpha}{perso}{"08}
\DeclareMathSymbol{\varbOmega}{\mathalpha}{perso}{"0A}
\DeclareMathSymbol{\varbdollar}{\mathalpha}{perso}{"24} 

\usepackage[normalem]{ulem}
\usepackage{soul}

\def\Q{{\mathbb{Q}}}
\def\C{{\mathbb{C}}}

\hypersetup{final}

\begin{document}

\title{Singularity Analysis for the Perspective-Four and Five-Line Problems}
\thanks{
This work was supported by the project PROMPT funded by the RFI AtlanSTIC2020, and by the French ANR project SESAME (funding ID: ANR-18-CE33-0011).
}


\author{Jorge Garc\'ia Font\'an \and Abhilash Nayak \and S\'ebastien Briot \and Mohab Safey El Din }


\institute{  Corresponding author: J. Garc\'ia Font\'an \at
            Sorbonne Universit\'e, LIP6, Equipe PolSys, Paris, France\\
             \email{Jorge.Garcia-Fontan@lip6.fr} \\
              ORCID Id: 0000-0003-4187-9624 
\and
              A. Nayak \at
             Centre national de la Recherche Scientifique (CNRS), Laboratoire des Sciences du Num\'erique de Nantes (LS2N), UMR CNRS 6004, Nantes, France \\
             \email{Abhilash.Nayak@ls2n.fr} \\
              ORCID Id: 0000-0001-6228-203X
\and
             S. Briot \at
              Centre national de la Recherche Scientifique (CNRS), Laboratoire des Sciences du Num\'erique de Nantes (LS2N), UMR CNRS 6004, Nantes, France \\
              \email{Sebastien.Briot@ls2n.fr}             \\
              ORCID Id: 0000-0002-0419-6042%
\and
             M. Safey El Din \at
             Sorbonne Universit\'e, CNRS, LIP6, Equipe PolSys, Paris, France\\
             \email{Mohab.Safey@lip6.fr} \\
              ORCID Id: 0000-0001-9463-1257 
}

\date{Received: date / Accepted: date}

\maketitle

\begin{abstract}
{This paper deals with image-based visual servoing and pose estimation by observing four and five lines. Our main interest is to determine the relative configurations of the camera and the observed lines that lead to problems in control and stability. Since it is equivalent to finding the singularities of the corresponding Jacobian matrix, we use tools from computational algebraic geometry to seek configurations such that all of its minors vanish simultaneously. By choosing a suitable basis for this matrix, we revisit the problem in the case of three lines to show that one type of the singularities is when the camera lies on the hyperboloid of one sheet uniquely defined by the lines. This result is further exploited to prove that the one-dimensional singularities, if any, in the case of $n$ lines appear when the camera lies on the transversals to the observed lines. Thus, by forcing the transversals to be complex, we can avoid the aforementioned type of singularities in the case of four lines although the algebra shows that there can always be up to 10 inevitable singular locations of the camera for the other type of singularity. For five lines, we find out that there are no singularities in the generic case. The singularities are also characterized for four and five lines with orthogonality and parallelism constraints. Furthermore, a visual servoing library is used to conduct some simulated experiments to substantiate the theoretical results. As expected, we observe problems in control in the vicinity of a singularity as well as increased errors in pose estimation.}

\end{abstract}    

\section{Introduction}
\label{sec:intro}




{A standard problem in computer vision, which has many applications in augmented reality~\citep{Marchand16a} and robotics (especially in visual servoing~\citep{Hutchinson_tutovisualservoTRO_1996}) is the estimation of the pose based on the features projected in the camera image. When the 2D image is a set of $n$ points that are projections of their 3D counterparts on the image plane, the problem is known as P$n$P~(Perspective-$n$-Point) and has been dealt with extensively in literature~\citep{gao2003complete,Horaud1998,Kneip_CVPR11,Wu2006pnp}. Similarly, when the features observed by the camera are $n$ straight lines, the problem is referred to as P$n$L~(Perspective-$n$-Line)~\citep{Dhome_IEEE_1989,Xu2017,Wang2020}.}


In particular, P$n$P and P$n$L involve determining the six pose parameters of a camera (degrees of freedom corresponding to its position and orientation) belonging to the Special Euclidean group SE(3) knowing the 3D-2D correspondences of the $n$ observed points and lines, respectively. {By taking the time derivatives of the parameters involved in  P$n$P or P$n$L, we obtain the so called motion-field equations~\citep{LonguetHiggins80} that are crucial to visual servoing in robotics. They involve the mapping between the time derivative of camera pose parameters belonging to se(3), being the Lie algebra of SE(3)~(3D vector space of translational and orientational velocities of the camera) and the relative velocities of the projected features on the image plane, 
through the \emph{image Jacobian} or \emph{interaction matrix}~\citep{Chaumette_Tutorial_Part1,Chaumette_tutorial_Part2,Chaumette_Handbook_2008}}.

Indeed, the problem of determining the singularities of the interaction matrix is crucial, {especially for the following reasons: 
\begin{itemize}
    \item In visual servoing tasks, we face potential accuracy and controllability issues of the robot when the camera is in the vicinity of a singularity~\citep{hutchinson1996tutorial}.
    \item The singularities are known to considerably worsen the pose reconstruction accuracy~\citep{Pascual-Escudero2021}. Moreover, their are known to influence the number of solutions to pose estimation problem as shown by~\cite{Rieck2013} and~\cite{Zhang2006WhyIT} in the case of P3P.
\end{itemize}}

{However, determining those singularities is computationally heavy since it involves solving the complex algebraic systems arising due to the loss of rank of the interaction matrix. As a result, the singularity analysis in the past has been limited to simple image features, such as the observation of three points in space (P3P).}
For this problem, a well-known result is that a singularity occurs if the three points are aligned or if the camera lies on the cylinder that contains the three points and is perpendicular to the plane they define~\citep{michel1993singularities}. {This result and tools from algebraic geometry were recently used by~\cite{Pascual-Escudero2021} to show that, in P4P, there are always two to six camera configurations where the corresponding interaction matrix becomes rank-deficient. }

In the case of P$n$L, most of the research has been focused on finding solutions to P3L~\citep{Dhome_IEEE_1989,Xu2017,Wang2020} without looking at the singularity problem, to the best of our knowledge. However, recently, the singularities in P3L were determined using a tool called the $\textit{hidden robot}$ which was introduced in~\citep{briot2013minimal}. It proved to be efficient in determining the singularities of vision-based controllers applied to parallel robots and broader classes of visual servo controllers~\citep{Briot_2015_TRO,QuattroVisualServoBriotIROS13,Briot_TRO_2017,Briot_RAL17}. With this method, it is possible to compute a simplified basis for the rows of the interaction matrix, based on the realization that they can be understood as a system of lines with Pl\"ucker coordinates. For the problem of visual servoing using three image lines, the $\textit{hidden robot}$ concept was used in~\cite{Briot_RAL17} to show that the singularities appear when the camera lies on a quadric or a cubic surface.

In this paper, we first provide a geometric interpretation of the results of~\cite{Briot_RAL17} and we exploit this interpretation in order to perform the singularity analysis for the generic P4L and P5L. {In order to do so, we first identify a basis of the rows of the interaction matrix. Then, the conditions for degeneracy of this basis are obtained thanks to the use of Gr\"obner Bases~\citep{cox2013ideals}.} Furthermore, we also provide the singularity loci when the lines are subjected to orthogonality and parallelism constraints. We focus on such type of lines because they often appear in a structured environment, e.g. edge tracking, navigation in urban areas, in corridors or any buildings, and are as a result commonly used as visual features for robot control tasks~\citep{andreff2002visual}. {Finally, simulations of camera motions near singularities are performed to illustrate their impact on robot control from visual servoing and on the pose estimation problem.}


The structure of the paper is as follows: Section~\ref{section:recalls} recalls the form of the interaction matrix related to image lines and provides the computation of a simplified basis for its rows. Section~\ref{sec:singP3L} revisits the singularities in P3L~\citep{Briot_RAL17} and puts forth their geometric interpretation. Sections~\ref{sec:singP4L} and~\ref{sec:singP5L} give a complete analysis of singularities in P4L and P5L, respectively, with a focus on the special cases where the observed lines are bound by orthogonality and parallelism constraints. Section~\ref{sec:sim} presents experimental results from simulations based on the exposed singularities. Section~\ref{sec:conclu} draws conclusions.

\section{Row basis of the interaction matrix \label{section:recalls}}

In this section, we deal with the interaction matrix related to PnL which corresponds to the observation of $n$ lines in space, and the computation of a new basis for the space spanned by its rows. It results in a simplified system of equations leading to a new interpretation of the results obtained in~\citep{Briot_RAL17}.

\subsection{Recalls on the interaction model related to image lines} \label{subsec:intmat}

In what follows, and without loss of generality, we will use the standard pin-hole camera model with focal length equal to 1, and the $\mathbf{z}$-axis oriented along the optical axis. However, any other camera model based on projective geometry could be used~\citep{michel1993singularities}.

There are several possibilities for describing a line in 3D space. Here we will use the Pl\"ucker representation which is free of representation singularities, and will be helpful later for expressing a basis of the interaction matrix.

Let us start by reviewing the geometric description of lines by Pl\"ucker coordinates. A 3D line $\mathcal{L}_i$ can be characterized in the camera frame by a Pl\"ucker vector $[\mathbf{u}_i^T \ \mathbf{l}_i^T ]^T \in \mathbb{R}^6$, where $\mathbf{u}_i$ is the direction of the line, and $\mathbf{l}_i = \mathbf{x} \times \mathbf{u}_i$, with $\mathbf{x}=\vv{CM_i}$ the position vector from the focal point $C$ to any point $M_i$ on the line.
\begin{figure}[t]
	\centering
	\includegraphics[width=0.9\linewidth]{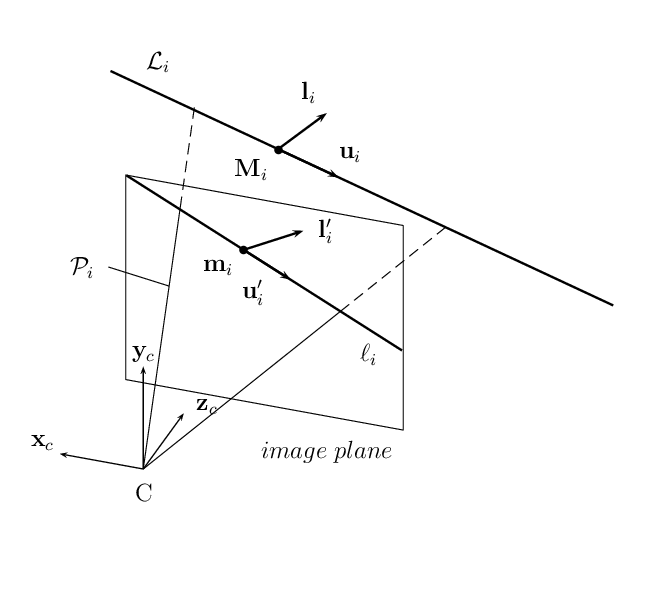}
	\caption{Observation of a line.\label{fig:Line_projection.png}}
\end{figure}

The 3D line $\mathcal{L}_i$ projects on the image plane of the camera on a 2D line $\ell_i$ (see Fig.~\ref{fig:Line_projection.png}) with Pl\"ucker coordinates $[{\mathbf{u}_i'}^T \ {\mathbf{l}_i'}^T]$, where $\mathbf{u}_i'$ is the image line direction, and $\mathbf{l}_i' = \mathbf{x}' \times \mathbf{u}_i'$ for any point on $\ell_i$ {with position vector $\mathbf{x}'$}. The coordinates of $\mathcal{L}_i$ and of its projection are related by the perspective equations~\citep{chaumette1990relation}:
\begin{equation}
  \label{eq:plucker2dline}
    \mathbf{u}_i' = \begin{bmatrix}
    u_{xi}'\\  u_{yi}'\\ u_{zi}'
    \end{bmatrix} =
    \begin{bmatrix}
    L_{yi}/\Delta_i\\ -L_{xi}/\Delta_i\\ 0
    \end{bmatrix}; \quad
    \mathbf{l}_i' = \begin{bmatrix}
    l_{xi}'\\  l_{yi}'\\ l_{zi}'
    \end{bmatrix} = \begin{bmatrix}
    L_{xi}/\Delta_i\\ L_{yi}/\Delta_i \\ L_{zi}/\Delta_i
    \end{bmatrix}
\end{equation}
where $\mathbf{l}_i=[L_{xi} \ L_{yi} \ L_{zi}]^T$ and $\Delta_i = \sqrt{L_{xi}^2 + L_{yi}^2}$ is a depth factor. The image line $\ell_i$ is fully determined from the three coordinates $l_{xi}'$, $l_{yi}'$ and $l_{zi}'$, so it suffices to use $\mathbf{l}_i'$ as the vector of features for the line in the visual servo scheme. 

A relative camera-object velocity, represented by the velocity twist $\boldsymbol{\tau}_c = [\mathbf{v}_c^T \ \boldsymbol{\omega}_c^T]^T$ expressed in the camera frame, will induce a variation of the image feature coordinates $\dot{l}_{xi}'$, $\dot{l}_{yi}'$ and $\dot{l}_{zi}'$, according to the relation
\begin{equation}
  \dot{\mathbf{l}}_i' = \mathbf{M}_{i} \boldsymbol{\tau}_c.
\end{equation}
The interaction matrix $\mathbf{M}_{i}$, which was derived in~\citep{chaumette1990relation}, has dimension $(3\times 6)$, but maximum rank $2$. Hence, we can control at maximum two degrees of freedom of the camera with each image line, and at least three lines are necessary to fully constrain the system~\citep{andreff2002visual}.

For the observation of $n>1$ lines, the interaction matrix $\mathbf{M}_{(n)}$, relating a change in the vector of features $\dot{\mathbf{s}} = [\dot{\mathbf{l}}_1^{'T}, \dots, \dot{\mathbf{l}}_n^{'T}]^T$ to the camera velocity, is obtained by stacking the three rows of $\mathbf{M}_{i}$ corresponding to each line $\mathcal{L}_i$:
\begin{equation}
\label{eq:interaction_matrix_full}
  \mathbf{M}_{(n)} = [\mathbf{M}_{1}^T, \ \mathbf{M}_{2}^T, \dots, \mathbf{M}_{n}^T]^T.
\end{equation}

\subsection{Revisiting the interaction matrix as a system of Pl\"ucker lines} \label{subsec:Plu}

Differentiating the vector $\mathbf{l}_i'$ in~\eqref{eq:plucker2dline} with respect to time we obtain:
\begin{equation}
\label{eq:ldot}
  \dot{\mathbf{l}}_i' = \frac{1}{\Delta_i^3} \begin{bmatrix} L_{yi}^2 & -L_{xi}L_{yi} & 0 \\
-L_{xi}L_{yi} & L_{xi}^2 & 0\\
-L_{xi}L_{zi} & -L_{yi}L_{zi} & \Delta_i^2
  \end{bmatrix} \begin{bmatrix} \dot{L}_{xi} \\ \dot{L}_{yi} \\ \dot{L}_{zi} \end{bmatrix} = \begin{bmatrix}  \mathbf{p}_{1}^T \\ \mathbf{p}_{2}^T \\ \mathbf{p}_{3}^T \end{bmatrix} \dot{\mathbf{l}}_i.
\end{equation}

The variation of the coordinates $\mathbf{l}_i$ of the 3D line is linked to the velocity twist of the camera by the relation: $\dot{\mathbf{l}}_i = \mathbf{v}_c \times \mathbf{u}_i + \boldsymbol{\omega}_c \times \mathbf{l}_i$~\citep{chaumette1990relation, rives1987estimation}, which can be expressed in matrix form as
\begin{equation}
\label{eq:Ldot}
  \dot{\mathbf{l}}_i = \begin{bmatrix} \left[ \mathbf{u}_i \right]_{\times}^T \ \left[ \mathbf{l}_i \right]_{\times}^T \end{bmatrix} \boldsymbol{\tau}_c.
\end{equation}
Here $\left[ \mathbf{u}_i \right]_{\times}$ and $\left[ \mathbf{l}_i \right]_{\times}$ denote the cross-product matrices associated to vectors $\mathbf{u}_i$ and $\mathbf{l}_i$. Introducing~\eqref{eq:Ldot} in~\eqref{eq:ldot}, we arrive at an expression for the interaction matrix:
\begin{equation}
\label{eq:MPlucker}
  \mathbf{M}_i = \begin{bmatrix} (\mathbf{u}_i \times \mathbf{p}_1)^T & (\mathbf{l}_i \times \mathbf{p}_1)^T \\
(\mathbf{u}_i \times \mathbf{p}_2)^T & (\mathbf{l}_i \times \mathbf{p}_2)^T\\
(\mathbf{u}_i \times \mathbf{p}_3)^T & (\mathbf{l}_i \times \mathbf{p}_3)^T \end{bmatrix}.
\end{equation}

By analysing the matrix in~\eqref{eq:ldot}, we see that the first and second rows are related by $L_{xi}\mathbf{p}_{1}+L_{yi}\mathbf{p}_{2} = 0$. Note that the row vectors $\mathbf{p}_{j}$ are all orthogonal to $\mathbf{l}_i$, and that $\mathbf{p}_{3}$ is linearly independent from $\mathbf{p}_{1}$, $\mathbf{p}_{2}$ as long as $\Delta_i \neq 0$. Since $\mathbf{l}_i$ is defined as $\mathbf{l}_i = \vv{CM_i} \times \mathbf{u}_i$, the vectors $\{\mathbf{p}_{1}, \mathbf{p}_{2}, \mathbf{p}_{3}\}$ span the same subspace as $\{\vv{CM_i},\mathbf{u}_i\}$, namely, the plane $\mathcal{P}_i$ containing the line and the focal point $C$, and whose normal has direction $\mathbf{l}_i$.

As a consequence, the vectors $\mathbf{u}_i \times \mathbf{p}_j$ and $\mathbf{l}_i \times \mathbf{p}_j$ in~\eqref{eq:MPlucker} are mutually orthogonal for each $j$ and therefore, the rows of the interaction matrix $\mathbf{M}_i$ can also be regarded as the coordinates of a system of Pl\"ucker lines.

\subsection{Change of basis for the rows of the interaction matrix} 
\label{subsec:newbasis}

The result that the rows of the matrix in~\eqref{eq:MPlucker} define a system of lines was generalized in~\citep{Briot_RAL17} to obtain a new basis $\boldsymbol{\xi}_i$ for the rows of $\mathbf{M}_i$. 

To start, the fact that vectors $\{\mathbf{p}_1, \mathbf{p}_2, \mathbf{p}_3\}$ in~\eqref{eq:ldot} span the same subspace as $\{\vv{CM_i}, \mathbf{u}_i\}$ implies that there exist two vectors $\boldsymbol{\mu}_{i1}$ and $\boldsymbol{\mu}_{i2}$, collinear with $\vv{CM_i}$ and with $\mathbf{u}_i$ respectively, which are linear combinations of $\mathbf{p}_1$, $\mathbf{p}_2$ and $\mathbf{p}_3$:
\begin{equation}
\boldsymbol{\mu}_{i1} = \sum_{j=1}^{3} a_j \cdot \mathbf{p}_j \, \propto \, \vv{CM_i}, 
\quad 
\boldsymbol{\mu}_{i2} = \sum_{j=1}^{3} b_j \cdot \mathbf{p}_j \, \propto \, \mathbf{u_i}.
\end{equation}

Assuming that the factor $\Delta_i$ in~\eqref{eq:ldot} is non-zero, there exists a transformation matrix $\mathbf{H}_i$, which is always of rank $2$,
\begin{equation}
\mathbf{H}_i = \begin{bmatrix}
a_1 & a_2 & a_3 \\ b_1 & b_2 & b_3 ,
\end{bmatrix}
\end{equation}
such that the product $\mathbf{H}_i\cdot \mathbf{M}_i$ is, after some manipulation,
\begin{equation}
\label{eq:introHM}
\mathbf{H}_i\cdot \mathbf{M}_i = \begin{bmatrix} (\mathbf{u}_i \times \boldsymbol{\mu}_{i1})^T & \left( \vv{CM_i} \times ( \mathbf{u}_i \times \boldsymbol{\mu}_{i1} ) \right)^T \\
\mathbf{0}_{1\times 3} & \left( \mathbf{u}_i \times ( \boldsymbol{\mu}_{i2} \times \vv{CM_i} ) \right)^T \end{bmatrix}.
\end{equation}


The matrix~\eqref{eq:introHM} is in fact proportional to the basis $\boldsymbol{\xi}_i$ for the rows of $\mathbf{M}_i$, obtained differently in~\citep{Briot_RAL17}, and which can be expressed as
\begin{equation}
\label{eq:basiseta}
\boldsymbol{\xi}_i = \begin{bmatrix} \boldsymbol{\xi}_{i1} \\ \boldsymbol{\xi}_{i2} \end{bmatrix} = \begin{bmatrix} \mathbf{f}_{i1}^T & (\vv{QM_i}\times \mathbf{f}_{i1})^T \\ \mathbf{0}_{1\times 3} & \mathbf{m}_{i2}^T \end{bmatrix},
\end{equation}
where 
\begin{equation}
\label{eq:fandmdirections}
\mathbf{f}_{i1} \propto \mathbf{u}_i \times \vv{CM_i}, \quad \mathbf{m}_{i2} \propto \mathbf{u}_i \times \mathbf{f}_{i1},
\end{equation}
and $Q$ is any point in space. Note that, by definition, $\mathbf{f}_{i1}$ is any vector normal to the plane $\mathcal{P}_i$ which contains the line and the focal point $C$, and $\mathbf{m}_{i2}$ is any vector orthogonal to both $\mathbf{f}_{i1}$ and the line direction $\mathbf{u}_i$. 

Let us remark that, in~\eqref{eq:basiseta}, the vector $\boldsymbol{\xi}_{i1}$ represents a straight line passing through the point $M_i$ with direction $\mathbf{f}_{i1}$, while $\boldsymbol{\xi}_{i2}$ can be regarded as a line at infinity (alias an ideal line) in the projective space, with direction $\mathbf{m}_{i2}$. 

Some observations arise from the fact that the basis $\boldsymbol{\xi}_i$ is spanned by a system of lines:
\begin{itemize}
\item Degeneracy of a system of lines is independent of the choice of point $Q$ (appearing in~\eqref{eq:basiseta}) at which the lines are expressed. However, when computing the analytical expressions, all lines must be given in reference to the same point. Note however that vector $\mathbf{f}_{i1}$ is still dependent on the location of $C$.
\item The conditions for degeneracy of any system of lines depend only on the relative configuration of the lines~\citep{Briot_RAL17, kanaan2009singularity}. Specifically, they are independent of the frame where the Pl\"ucker vectors are expressed, and therefore of the relative orientation of the object and camera frames.
\end{itemize}

In particular the second remark will be useful to simplify the computations in the following sections by assuming a constant zero orientation for the camera.

The new basis $\boldsymbol{\xi}_i$ is a valid representation so long as the depth factor $\Delta_i = \sqrt{L_{xi}^2+L_{yi}^2}$ appearing in~\eqref{eq:plucker2dline} is non-zero. This excludes only $2$ camera configurations: \textit{1)} when line $\mathcal{L}_i$ is fully contained in the plane $Z=0$ \textit{of the camera frame} and \textit{2)} when the focal point $C$ lies on the line $\mathcal{L}_i$; in both situations the coordinates $L_{xi}$ and $L_{yi}$ vanish. These are degenerate cases for which the projection mapping in~\eqref{eq:plucker2dline} is ill-defined, so we will not consider them in the sections that follow.

Finally, based on the previous results, a basis $\boldsymbol{\xi}_{(n)}$ for the full interaction matrix $\mathbf{M}_{(n)}$ is obtained by stacking the rows of~\eqref{eq:basiseta} for each line $i$:
\begin{equation} \label{eq:zetan}
\boldsymbol{\xi}_{(n)} = [\boldsymbol{\xi}_1^T \ \boldsymbol{\xi}_2^T \dots \boldsymbol{\xi}_n^T]^T \in \mathbb{R}^{2n\times 6}.
\end{equation}

For instance, a basis for the interaction matrix $\mathbf{M}_{(3)}$ corresponding to P3L is given by $\boldsymbol{\xi}_{(3)} = [\boldsymbol{\xi}_{11}^T \ \boldsymbol{\xi}_{12}^T \ \boldsymbol{\xi}_{21}^T \ \boldsymbol{\xi}_{22}^T \ \boldsymbol{\xi}_{31}^T \ \boldsymbol{\xi}_{32}^T]^T$. The singularities of this matrix were analysed algebraically in~\citep{Briot_RAL17}. These results are revisited in the next section with a geometric point of view, which will then be used to analyse the singularities both algebraically and geometrically in the observation of more than three lines. 

\section{Revisiting singularities in P3L} \label{sec:singP3L}

\subsection{Parametrization} \label{subsec:paramP3L}
Since the set of lines in the three dimensional projective space $\mathbb{P}^3$ is a four-dimensional manifold, a line can be defined using four independent parameters. In the Pl\"ucker representation of lines as described in Section~\ref{subsec:intmat}, only four out of the six Pl\"ucker coordinates are independent. {Hence, we need $12$ parameters to define three generic lines.}
The object frame $\mathcal{F}_o: (O, \mathbf{x}_o, \mathbf{y}_o, \mathbf{z}_o)$ is fixed relative to the three lines, with its axes defining an orthonormal, right-handed basis. The first line can be placed on the $\mathbf{x}_o$ axis and the second line parallel to the plane $z_o=0$ and intersecting the $\mathbf{z}_o$ axis. {Any other line can be defined using its points of intersection with any two of the three planes $x_o=0$, $y_o=0$ or $z_o=0$. So, the third line is defined using its points of intersection with the plane $x_o=0$ and $z_o=0$. This leaves us $7$ parameters to define three lines using the following two points $M_i$ and $N_i$ on each line.}
\begin{align}
		&\overrightarrow{OM_1}=[0,0,0]^T,\;\overrightarrow{ON_1}=[1,0,0]^T, \nonumber \\
		&\overrightarrow{OM_2}=[0,0,d_1]^T,\;\overrightarrow{ON_2}=[r_1,r_2,d_1]^T, \nonumber \\
		&\overrightarrow{OM_3}=[d_2,d_3,0]^T,\;\overrightarrow{ON_3}=[0,r_3,r_4]^T. 
\end{align} 
Since the rows of the interaction matrix defined in Section~\ref{subsec:newbasis} consist of some affine and ideal lines, defined using the direction vector of the lines, the parametrization can be simplified by considering the direction vector $\mathbf{u}_i=\overrightarrow{OM_i}-\overrightarrow{ON_i},\,i=1,2,3$ and changing the parameters $-r_1=s_1$, $-r_2=s_2$, $-r_4=s_4$, $d_3-r_3=s_3$ as follows:
\begin{align} 
		&\overrightarrow{OM_1}=[0,0,0]^T,\;\mathbf{u}_1=[1,0,0]^T, \nonumber \\
		&\overrightarrow{OM_2}=[0,0,d_1]^T,\;\mathbf{u}_2=[s_1,s_2,0]^T, \nonumber \\
		&\overrightarrow{OM_3}=[d_2,d_3,0]^T,\;\mathbf{u}_3=[d_2,s_3,s_4]^T. \label{eq:param}
\end{align} 
We also define the position of the focal point $C$ in frame $\mathcal{F}_o$ by the vector ${}^o\vv{OC} = [X \ Y \ Z]^T$ and define the camera frame $\mathcal{F}_c:\{C, \mathbf{x}_c,\mathbf{y}_c,\mathbf{z}_c\}$ centred at $C$, with $\mathbf{x}_c,\mathbf{y}_c,\mathbf{z}_c$ also an orthonormal basis. It was noted in Section~\ref{subsec:newbasis} that the singularity conditions of the problem are independent of the relative orientation of the object and camera frames. Hence, for the computations we will assume that $\mathcal{F}_c$ can be obtained from $\mathcal{F}_o$ by a direct translation by the vector ${}^o\vv{OC}$.

The singularity loci will be given in terms of the location of the focal point $C$ relative to the fixed object frame $\mathcal{F}_o$, that is as a set of expressions involving variables $X$, $Y$ and $Z$. If $\mathcal{F}_o$ and $\mathcal{F}_c$ do not share the same orientation, we can retrieve the position of the origin $O$ expressed in the camera frame $\mathcal{F}_c$ by introducing a new set of variables $X'$, $Y'$, $Z'$ such that
\begin{equation}
{}^c\vv{CO} = \begin{bmatrix}X' \\ Y' \\ Z'\end{bmatrix} = - {}^c\mathbf{R}_o \cdot {}^o\vv{OC} = - {}^c\mathbf{R}_o \cdot \begin{bmatrix}X \\ Y \\ Z\end{bmatrix},
\end{equation}
where the matrix ${}^c\mathbf{R}_o$ represents the relative camera-object orientation. In the following, we assume ${}^c\mathbf{R}_o$ is the identity matrix.

\subsection{Geometric interpretation of singularities in P3L} \label{subsec:geomP3L}

Singularities in the observation of three lines have already been determined algebraically~\citep{Briot_RAL17}. In this section, we provide the geometric interpretation of those results which will aid us in determining the singularities in P4L and P5L. By choosing the point $Q$ in~\eqref{eq:basiseta} as the camera center $C$, we have
\begin{equation}
	\xi_i= \left[ \begin{matrix}
		\xi_{i1} \\ \xi_{i2} \end{matrix} \right] = \left[ \begin{matrix} \mathbf{f}_{i1}^T & (\overrightarrow{CM_i} \times \mathbf{f}_{i1})^T \\ \mathbf{0}_{(1 \times 3)} & \mathbf{m}_{i2}^T
	\end{matrix} \right],
\end{equation}
where $\overrightarrow{CM_i}=\overrightarrow{OM_i}-\overrightarrow{OC}$ with $\overrightarrow{OC}=[X,Y,Z]^T$ being the position vector of the camera center. Consequently, a basis for the rows of the interaction matrix $\mathbf{M}_{(3)}$ can be obtained under the form $\boldsymbol{\xi}_{(3)} = [\boldsymbol{\xi}_{11}^T \ \boldsymbol{\xi}_{12}^T \ \boldsymbol{\xi}_{21}^T \ \boldsymbol{\xi}_{22}^T \ \boldsymbol{\xi}_{31}^T \ \boldsymbol{\xi}_{32}^T]^T$. It was proven in \citep{Briot_RAL17} that the degeneracy of the basis $\boldsymbol{\xi}_{(3)}$ and thus of the interaction matrix $\mathbf{M}_{(3)}$ occurs if and only if one of the following two conditions is satisfied:
\begin{equation}
\label{eq_conditionG}
  G = \mathbf{f}_{11}\cdot (\mathbf{f}_{21} \times \mathbf{f}_{31}) = 0
\end{equation}
or
\begin{equation}
\label{eq_conditionH}
  H = \mathbf{m}_{12}\cdot (\mathbf{m}_{22} \times \mathbf{m}_{32}) = 0,
\end{equation}
that is, when the three vectors $\mathbf{f}_{i1}$, $i=1,2,3$ (respectively, $\mathbf{m}_{i2}$) defined in \eqref{eq:fandmdirections} are parallel to the same plane. The product $G\cdot H$ is in fact the determinant of the matrix $\boldsymbol{\xi}_{(3)}$. In the case of three general lines observed in space, it was shown in~\citep{Briot_RAL17} that a singularity appears when the camera lies either on a quadric or on a cubic surface, defined respectively by~\eqref{eq_conditionG} and~\eqref{eq_conditionH}. 
The quadratic factor~\eqref{eq_conditionG} using the parametrization~\eqref{eq:param} is as follows:
\begin{dmath}\label{eq:hyp3}
	G=d_{{1}}s_{{2}}s_{{4}}XY-s_{{2}} \left( d_{{1}}s_{{3}}+d_{{3}}s_{{4}}
	\right) XZ-d_{{1}}s_{{1}}s_{{4}}{Y}^{2}+ \left( d_{{1}}s_{{1}}s_{{3}}
	+d_{{2}}s_{{2}}s_{{4}} \right) YZ-d_{{1}}s_{{4}} \left( d_{{2}}s_{{2}}
	-d_{{3}}s_{{1}} \right) Y+d_{{2}}s_{{2}} \left( d_{{3}}-s_{{3}}
	\right) {Z}^{2}-d_{{1}}d_{{2}}s_{{2}} \left( d_{{3}}-s_{{3}} \right) 
	Z.
\end{dmath}
\begin{figure}[t]
	\begin{center}
		\includegraphics[width=0.8\linewidth]{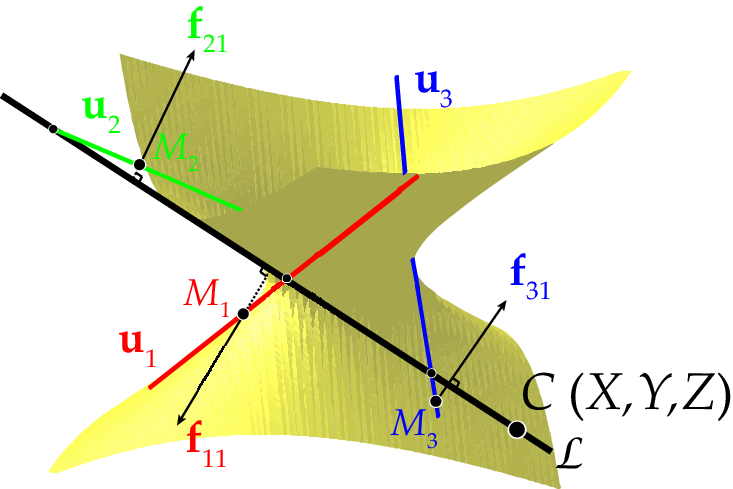}
		\caption{One of the singularities in P3L is when the camera center $C$ lies on the hyperboloid formed by the three observed lines. \label{fig:3lhyp}}
	\end{center}
\end{figure}
It is a hyperboloid of one sheet, leading to the following theorem.
\begin{theorem}
 The quadratic factor corresponding to singularities in P3L implies that the camera center lies on the hyperboloid of one sheet uniquely defined by the three observed lines.
\end{theorem}
\begin{proof}
The quadric factor is the determinant of the upper left $(3 \times 3)$ matrix of the basis $\boldsymbol{\xi}_{(3)}$. Let the kernel of the matrix $\left[ \begin{matrix} \mathbf{f}_{11} & \mathbf{f}_{21} & \mathbf{f}_{31} \end{matrix} \right] $ be the direction vector of a line $\mathcal{L}$. Since the affine lines are represented according to~\eqref{eq:fandmdirections}, the line $\mathcal{L}$ must lie in a plane containing $\overrightarrow{CM_i}$ and $\mathbf{u}_i$, say $\mathcal{P}_i$. Consequently, $\mathcal{L}$ has to intersect all $\mathbf{u}_i$ and therefore the observed three lines~(see Fig.~\ref{fig:3lhyp}). As a result, it belongs to the complementary regulus of a hyperboloid of one sheet defined by the regulus of the observed lines~\citep[Chapter~2]{Odenhal}. $\mathcal{L}$ is called the \emph{transversal} line. Moreover, $\mathcal{L}$ must be the intersection of planes $\mathcal{P}_1$, $\mathcal{P}_2$ and $\mathcal{P}_3$. Since $C$ belongs to $\mathcal{P}_i\, \forall \, i$, it should lie in their intersection too and hence $\mathcal{L}$ has to contain $C$. 
\end{proof}
\begin{remark}\label{rem:kern}
	When the camera center $C$ lies on the hyperboloid defined by the three observed lines, the kernel of the interaction matrix $\boldsymbol{\xi}_{(3)}$ is an ideal line whose moment vector is the same as the direction vector of the line passing through $C$ and intersecting the observed lines. As a result, we face problems in control for infinitesimal translations of the camera along the transversal line.
\end{remark}
\begin{remark}\label{rem:singll}
	When the camera center $C$ lies on the hyperboloid defined by the three observed lines, the finite lines constituting the rows of the interaction matrix are all parallel to the same plane whose normal vector is along the transversal. Then, the six lines $\boldsymbol{\xi}_{i1}$ and $\boldsymbol{\xi}_{i2}$, $i=1,2,3$ are said to be in a singular linear line complex~\citep[Chapter~3]{Pottmann_Wallner}. 
\end{remark}
The latter remark assures that the kernel of the interaction matrix represents the Pl\"ucker coordinates of a line. In terms of screw theory~\citep{Hunt1987}, it is always a screw of infinite pitch.%

Unfortunately, the geometric interpretation is not straightforward when $C$ lies on the cubic surface $H=0$ of~\eqref{eq_conditionH} leading to singularities. In this case, the lines $\boldsymbol{\xi}_{i1}$ and $\boldsymbol{\xi}_{i2}$ belong to a regular linear line complex~\citep[Chapter~3]{Pottmann_Wallner}. In terms of screw theory~\citep{Hunt1987}, the kernel is no longer a line but a screw, meaning that the controllability issues arise when the camera performs this instantaneous screw motion. Additionally, unlike the hyperboloid, the cubic surface is not uniquely defined by the three observed lines. This is due to a classic result from Geometry by Arthur Cayley and George Salmon who showed in 1849 that there are 27 lines on a cubic surface~(refer to~\cite{Lazarus2014} for a proof from an algebraic geometry point of view). Therefore, computational algebraic techniques will be employed to deal with this singularity in the case of P4L.

\section{Singularities in P4L} \label{sec:singP4L}

\subsection{Parametrization}
{Following Section~\ref{subsec:paramP3L}, the first three lines are defined according to~\eqref{eq:param}.
The fourth line is defined using its two points of intersection $M_4$ and $N_4$ with the planes $x_o=0$ and $y_o=0$, respectively:  
\begin{align} 
\overrightarrow{OM_4}=[0,d_4,d_5]^T,\;\overrightarrow{ON_4}=[r_5,0,r_6]^T.
\end{align}
Thus, the direction vector of the fourth line is given by $\mathbf{u}_4=\overrightarrow{OM_4}-\overrightarrow{ON_4}$. After replacing $-r_5$ and $d_5-r_6$ by $s_5$ and $s_6$, respectively, we have
\begin{align} \label{eq:param_l4}
\overrightarrow{OM_4}=[0,d_4,d_5]^T,\;\mathbf{u}_4=[s_5,d_4,s_6]^T.
\end{align}
As mentioned in Section~\ref{subsec:paramP3L}, the assumption that the relative orientation between the camera frame $\mathcal{F}_c$ and the object frame $\mathcal{F}_o$ is zero remains valid in the following analysis.}

\subsection{Singularity analysis}

From the remarks made in Section~\ref{subsec:newbasis}, a basis $\boldsymbol{\xi}_i = [{\xi}_{i1}^T \ {\xi}_{i2}^T]^T$ for the rows of the interaction matrix for each line is computed as in~\eqref{eq:basiseta}, with point $Q$ taken as the camera center $C$, and with vectors $\mathbf{f}_{i1}$ and $\mathbf{m}_{i2}$ given by~\eqref{eq:fandmdirections}.

Singularities of the interaction matrix $\mathbf{M}_{(4)}$ of the four lines appear when the $(8\times 6)$ matrix
\begin{equation}
\label{eq:eta_4}
\boldsymbol{\xi}_{(4)} = [\boldsymbol{\xi}_1^T \ \boldsymbol{\xi}_2^T \ \boldsymbol{\xi}_3^T \ \boldsymbol{\xi}_4^T]^T,
\end{equation}
formed by stacking the rows in~\eqref{eq:basiseta} for all lines, becomes rank-deficient. This is the case if and only if all the $28$ maximal minors of~\eqref{eq:eta_4} vanish simultaneously.

All the entries of the matrix $\boldsymbol{\xi}_{(4)}$ are polynomials in the
variables $\{X,Y,Z\}$ representing the camera location, with coefficients which
are polynomials in the parameters $\boldsymbol{\eta}=\{s_1,s_2,s_3,s_4,s_5,s_6,d_1,d_2,d_3,d_4,d_5\}$. The minors
of $\boldsymbol{\xi}_{(4)}$ then form a system of $28$ polynomials $p_i$, which are said to generate an \textit{ideal} (refer to~\citep{cox2013ideals} for
more information about ideals), denoted by $\mathcal{I}_{28}=\langle p_1, p_2,
\dots, p_{28} \rangle$. An ideal is defined as follows: let $f_1, \dots, f_s$ in
$\Q[x_1, \dots, x_n]$ be a set of polynomials in the variables $x_1, \dots, x_n$
with coefficients in the field of rational numbers $\Q$; the ideal generated by
$f_1, \dots, f_s$ is the set
\begin{equation}
\langle f_1, \dots, f_s \rangle = \left\{\sum_{i}^{s} h_i \, f_i \ | \ h_i \in K[x_1, \dots, x_n]\right\}
\end{equation}
of all polynomials which are algebraic combinations of its generators. By this definition, any common solution of the system $p_1 = \dots = p_{28} = 0$ formed by the minors of $\boldsymbol{\xi}_{(4)}$ is also a solution of any polynomial in the ideal $\mathcal{I}_{28}$.

In geometric terms, the locus of complex solutions of a system of polynomial
equations is called a \textit{variety}. The (complex) solutions of all the
polynomials in $\mathcal{I}_{28}$ define a variety
$\mathbf{V}(\mathcal{I}_{28})\subset \C^{14}$, which consists of all points where
the matrix $\boldsymbol{\xi}_{(4)}$ becomes rank-deficient.

To get a better insight of $\mathbf{V}(\mathcal{I}_{28})$, we describe below how it can be split into subsets, i.e. written as the union of solution sets of simpler polynomial systems of equations.

Consider the $28$ maximal submatrices of size $(6\times 6)$ of
$\boldsymbol{\xi}_{(4)}$. There are three types of them:
\begin{equation}  \label{eq:subm1}
\boldsymbol{\xi}^{ij}_{1234}=
\begin{bmatrix}
		\mathbf{f}_{i1}^T & (\overrightarrow{CM_i} \times \mathbf{f}_{i1})^T\\
		\mathbf{f}_{j1}^T & (\overrightarrow{CM_j} \times \mathbf{f}_{j1})^T\\
		\mathbf{0}_{(1 \times 3)} & \mathbf{m}_{12}^T\\
		\mathbf{0}_{(1 \times 3)} & \mathbf{m}_{22}^T\\
		\mathbf{0}_{(1 \times 3)} & \mathbf{m}_{32}^T\\
		\mathbf{0}_{(1 \times 3)} & \mathbf{m}_{42}^T
	\end{bmatrix},
\end{equation}
\begin{equation}  \label{eq:subm2}
\boldsymbol{\xi}^{ijk}_{lmn}=
\begin{bmatrix}
		\mathbf{f}_{i1}^T & (\overrightarrow{CM_i} \times \mathbf{f}_{i1})^T\\
		\mathbf{f}_{j1}^T & (\overrightarrow{CM_j} \times \mathbf{f}_{j1})^T\\
		\mathbf{f}_{k1}^T & (\overrightarrow{CM_k} \times \mathbf{f}_{k1})^T\\
		\mathbf{0}_{(1 \times 3)} & \mathbf{m}_{l2}^T\\
		\mathbf{0}_{(1 \times 3)} & \mathbf{m}_{m2}^T\\
		\mathbf{0}_{(1 \times 3)} & \mathbf{m}_{n2}^T
	\end{bmatrix},
\end{equation}
\begin{equation} \label{eq:subm3}
\boldsymbol{\xi}^{1234}_{lm}=
\begin{bmatrix}
		\mathbf{f}_{11}^T & (\overrightarrow{CM_i} \times \mathbf{f}_{11})^T\\
		\mathbf{f}_{21}^T & (\overrightarrow{CM_j} \times \mathbf{f}_{21})^T\\
		\mathbf{f}_{31}^T & (\overrightarrow{CM_k} \times \mathbf{f}_{31})^T\\
		\mathbf{f}_{41}^T & (\overrightarrow{CM_l} \times \mathbf{f}_{41})^T\\
		\mathbf{0}_{(1 \times 3)} & \mathbf{m}_{l2}^T\\
		\mathbf{0}_{(1 \times 3)} & \mathbf{m}_{m2}^T\\
	\end{bmatrix},
\end{equation}
where $i,j,k$ and $l,m,n$ range every triplet of numbers in $\{1,2,3,4\}$.

There are six submatrices of type $\boldsymbol{\xi}^{ij}_{1234}$ and their determinants are zero since out of four lines at infinity, only three can be independent and hence, the matrix always has a rank of at most 5.

The second category of block-triangular submatrices are composed of row vectors that represent three affine lines and three lines at infinity. There are $\binom{4}{3} \times \binom{4}{3} =16$ of them. The determinants of $\boldsymbol{\xi}^{ijk}_{lmn}$ are of the form $G_{ijk} \cdot H_{lmn}$, with:
\begin{equation}
\label{eq:Gijk}
G_{ijk} = \mathbf{f}_{i1} \cdot (\mathbf{f}_{j1} \times \mathbf{f}_{k1})
, 
H_{lmn} = \mathbf{m}_{l2} \cdot (\mathbf{m}_{m2} \times \mathbf{m}_{n2}).
\end{equation}
The cases where the subindices $\{i,j,k\}$ and $\{l,m,n\}$ coincide correspond to the singularity conditions~\eqref{eq_conditionG} and~\eqref{eq_conditionH} in P3L for each triplet of lines taken individually. Thus, it is useful to note that for a singularity of the four lines, a necessary, but not sufficient condition is that each triplet of lines is in turn in a singular configuration. The determinants of matrices $\boldsymbol{\xi}^{ijk}_{lmn}$ generate an ideal $\mathcal{I}_{16} = \langle G_{123} H_{123}, \dots ,G_{234} H_{234} \rangle$ which is contained in the larger ideal $\mathcal{I}_{28}$.

There are six submatrices $\boldsymbol{\xi}^{1234}_{lm}$ of the third category. Their determinants are of degree $5$ in the ring $\Q[X,Y,Z]$. Let them generate an ideal $\mathcal{K}$.

It follows that the union of ideals $\mathcal{I}_{16}$ and $\mathcal{K}$ yields $\mathcal{I}_{28}$. Dually, the intersection of their varieties yields $\mathbf{V}(\mathcal{I}_{28})$ (see~\citep{cox2013ideals} for more information on correspondence between ideals and varieties):
\begin{equation}
\label{eq:KunionI16}
  \mathcal{I}_{28} = \mathcal{I}_{16} \cup \mathcal{K} \Longrightarrow \mathbf{V}(\mathcal{I}_{28}) = \mathbf{V}(\mathcal{I}_{16}) \cap \mathbf{V}(\mathcal{K}),
\end{equation}
{We first thoroughly analyze the ideal $\mathcal{I}_{16}$ to show how it can be further decomposed into two sub-ideals and then incorporate the analysis of $\mathcal{K}$.} Note that all mathematical derivations shown below can be followed in the attached Maple file.
 Since a solution of the polynomials in $\mathcal{I}_{28}$ must also be a solution for the polynomials in $\mathcal{I}_{16}$, we say that the variety $\mathbf{V}(\mathcal{I}_{28})$ is contained in $\mathbf{V}(\mathcal{I}_{16})$:
\begin{equation}
\label{eq:VI28inVI16}
  \mathcal{I}_{16} \subseteq \mathcal{I}_{28} \Longrightarrow \mathbf{V}(\mathcal{I}_{28})\subseteq \mathbf{V}(\mathcal{I}_{16}),
\end{equation}
although $\mathbf{V}(\mathcal{I}_{16})$ may contain points outside $\mathbf{V}(\mathcal{I}_{28})$.

The ideal $\mathcal{I}_{16}$ can be factorized as the product of two simpler ideals: 
{\begin{equation}
\label{eq:VI16_VGpVH}
\mathcal{I}_{16} = \mathcal{G}\times \mathcal{H},
\end{equation}
where
\begin{align}
    \mathcal{G} = & \langle \mathcal{G}_{123},\, \mathcal{G}_{124},\, \mathcal{G}_{134},\, \mathcal{G}_{234} \rangle, \label{eq:idealG_4lines} \\
    \mathcal{H} = & \langle \mathcal{H}_{123},\, \mathcal{H}_{124},\, \mathcal{H}_{134},\, \mathcal{H}_{234} \rangle. \label{eq:idealH_4lines}
\end{align}}
It implies that the variety $\mathbf{V}(\mathcal{I}_{16})$ is the union of two smaller sets: $\mathbf{V}(\mathcal{I}_{16})=\mathbf{V}(\mathcal{G})\cup \mathbf{V}(\mathcal{H})$. That is, the polynomials in $\mathcal{I}_{16}$ vanish whenever $G_{ijk}=0$ for all $i,j,k$; or when $H_{lmn}=0$ for all $l,m,n$. {As a consequence, we can rewrite~\eqref{eq:KunionI16} as
\begin{align}
    \label{eq:KuGiH}
 \mathbf{V}(\mathcal{I}_{28}) & = \left(\mathbf{V}(\mathcal{G}) \cup \mathbf{V}(\mathcal{H})\right) \cap \mathbf{V}(\mathcal{K}) \nonumber \\
  & = \left(\mathbf{V}(\mathcal{G}) \cap \mathbf{V}(\mathcal{K})\right) \cup \left(\mathbf{V}(\mathcal{H}) \cap \mathbf{V}(\mathcal{K})\right).
\end{align}
The variety $\mathbf{V}(\mathcal{G})$ defined by the ideal in~\eqref{eq:idealG_4lines} in $\C[X,Y,Z]$ describes the intersection between four hyperboloids, while $\mathbf{V}(\mathcal{H})$ defined by~\eqref{eq:idealH_4lines} is the intersection of four cubic surfaces.}

We can analyze the sub-varieties on the right hand side of~\eqref{eq:KuGiH} separately. First, we can check if  $\mathbf{V}(\mathcal{K})$ or a component of it lies in $\mathbf{V}(\mathcal{G})$ or $\mathbf{V}(\mathcal{H})$. For instance, if $\mathbf{V}(\mathcal{G}) \subset \mathbf{V}(\mathcal{K})$ then the analysis is much simpler since the intersection between those varieties would just yield $\mathbf{V}(\mathcal{G})$. 
So, to check if the varieties share any component, we need to know the following definitions.

Ideals can have different bases: there are many different
polynomials which generate the same family of polynomial combinations. A particularly useful basis for an ideal $\mathcal{I}$ is the Gr\"obner Basis~\citep{cox2013ideals}. It is defined as a set of polynomials $gb = \{g_1, \dots, g_s\}$ in $\mathcal{I}$ whose leading monomials $LM(g_i)$ generate the same ideal as all the leading monomials in $I$, that is, $\langle LM(g_1), \dots, LM(g_s) \rangle = \langle LM(\mathcal{I}) \rangle$. A Gr\"obner Basis $gb$ is only dependent on any set of generators of $I$ and an ordering of the monomials ``$\succ$", which must be specified \textit{a priori}, and is unique up to certain algorithmic reductions.
One can refer to~\citep{cox2013ideals} for detailed information on the properties of Gr\"obner Bases. A useful consequence of this definition is that any polynomial $f$ can be written as an algebraic combination of the elements of $gb$ plus a residue term $\overline{f}^{gb}$ which is unique: $f=\sum_1^s h_i\ g_i + \overline{f}^{gb}$. This residue $\overline{f}^{gb}$ is called the \textit{normal form} of $f$ w.r.t. the basis $gb$, and it vanishes if
and only if the polynomial $f$ belongs in the ideal $\mathcal{I}$. Hence, when $\overline{f}^{gb}=0$, the solutions of the system $g_1=\dots=g_s=0$ are also solutions of $f=0$. We can use this to find the Gr\"obner bases of ideals $\mathcal{G}$ and $\mathcal{H}$ and then to find the normal form of polynomials in $\mathcal{K}$ w.r.t to those bases. This will help us find which of the solutions of the ideal $\mathcal{I}_{16}= \mathcal{G} \times \mathcal{H}$ are also solutions of the larger ideal $\mathcal{I}_{28}$. In what follows, the analysis is done separately for varieties $\mathbf{V}(\mathcal{G}) \cap \mathbf{V}(\mathcal{K})$ and $\mathbf{V}(\mathcal{H}) \cap \mathbf{V}(\mathcal{K})$ of~\eqref{eq:KuGiH}.

\subsection{{Analysis of the variety $\mathbf{V}(\mathcal{G}) \cap \mathbf{V}(\mathcal{K})$}} \label{subsec:VG}

Let $gb_G=\{g_1, \dots, g_s\}$ be a Gr\"obner Basis of the ideal $\mathcal{G}$ with respect to \textit{pure lexicographical} monomial ordering~{(denoted by $\succ_{lex}$; see~\citep{cox2013ideals} for information about using different monomial orderings in the computation of Gr\"obner Bases)} with $X\succ_{lex} Y\succ_{lex} Z\succ_{lex} s_1\succ_{lex} \dots \succ_{lex} s_6\succ_{lex} d_1\succ_{lex} \dots \succ_{lex} d_5$. Here, we are assuming that our polynomials lie in the ring $\mathbb{Q}[X,Y,Z, \boldsymbol{\eta}]$ of polynomials in variables $X,Y,Z$ and $\boldsymbol{\eta}$ with rational coefficients, where $\boldsymbol{\eta}=\{d_1,\ldots,d_5,s_1,\ldots,s_6 \}$. By considering $s_1,\dots,d_5$ as variables, our Gr\"obner Bases computations will be valid under any specialization of the parameters.

{Knowing $gb_{G}$, we can compute the normal form of the polynomials $p_i \in \mathcal{K},\, i=1,\ldots,6$ to check if the varieties $\mathbf{V}(\mathcal{G})$ and $\mathbf{V}(\mathcal{K})$ share any component. It can be done using a multivariate polynomial division algorithm (see the attached Maple file) and we find that they vanish in every case: $\overline{p_i}^{\langle g_1, \dots, g_s\rangle} = 0$.} 
{The consequence is that any common solution of the system $g_1 = \dots = g_s = 0$ is also a solution of the system $p_1 = \dots = p_{6} = 0$ and is therefore a singular point of the matrix $\boldsymbol{\xi}_{(4)}$. In terms of algebraic
varieties, this can be written as
\begin{equation}
\label{eq:VGinVI28}
\mathbf{V}(\mathcal{G})\subseteq \mathbf{V}(\mathcal{K}) \textrm{ and hence } \mathbf{V}(\mathcal{G}) \cap \mathbf{V}(\mathcal{K}) = \mathbf{V}(\mathcal{G}) .
\end{equation}
As a result, Equation~\eqref{eq:KuGiH} can be updated as
\begin{equation}
    \label{eq:VGunionrest}
 \mathbf{V}(\mathcal{I}_{28}) = \mathbf{V}(\mathcal{G})  \cup \left(\mathbf{V}(\mathcal{H}) \cap \mathbf{V}(\mathcal{K})\right),
\end{equation}
implying that one of the singularities in P4L is when the camera center lies on the intersection of the four hyperboloids given by the variety $\mathbf{V}(\mathcal{G})$.}

{Additionally}, this result can be geometrically interpreted as follows. The basis of the interaction matrix $\mathbf{M}_{(4)}$ in this case consists of four affine lines and four ideal lines. {As we know from Section~\ref{sec:singP3L}, one of the singularities in P3L is when the three affine lines belong to a singular linear line complex~(meaning that they are parallel to the same plane; see Remark~\ref{rem:singll}). When this happens, the kernel of the interaction matrix is a line at infinity (see Remark~\ref{rem:kern}) and the camera center lies on a line that intersects all the three observed lines. Similarly, in the case of P4L, $\mathbf{V}(\mathcal{G})$ results in four affine lines of $\boldsymbol{\xi}_{(4)}$ being parallel to the same plane so that its kernel is a line at infinity. Hence, we can expect that a singularity occurs when the camera center lies on a line that intersects the four observed lines.
In fact, this is true for singularities in P$n$L for any $n \geq 3$.
\begin{theorem}\label{thm:transversals}
     A singular configuration of P$n$L is when the camera center $C$ lies on the line(s) intersecting the observed lines.
\end{theorem}
\begin{proof}
	Let the observed lines be $\mathcal{L}_1,\mathcal{L}_2,\ldots,\mathcal{L}_n$. Given a line $\mathcal{L}$ on which $C$ lies, the distance between lines $\mathcal{L}$ and $\mathcal{L}_i$ is given by
	\begin{equation}
		d_{\mathcal{L}\mathcal{L}_i}=\left\Vert \frac{(\hat{\bf l} \times \hat{{\bf l}}_i)^T}{||\hat{\bf l} \times \hat{{\bf l}}_i||} (\overrightarrow{OC}-\overrightarrow{OM_i}) \right\Vert,
	\end{equation}
where $M_i$ is a point on line $\mathcal{L}_i$, $\hat{\bf l}$ and $\hat{\bf l}_i$ are unit direction vectors of lines $\mathcal{L}$ and $\mathcal{L}_i$, respectively. $\mathcal{L}$ and $\mathcal{L}_i$ intersect when 
\begin{align}
      & (\mathbf{l} \times \mathbf{l}_i )^T(\overrightarrow{OC}-\overrightarrow{OM_i})=0, \\
      & (\mathbf{l}_i \times \overrightarrow{CM_i} )^T\mathbf{l}=0.
\end{align}
	It follows from~\eqref{eq:fandmdirections} that $\mathbf{f}_{i1}^T\mathbf{l}=0,\, i=1,\ldots,n$. Thus, for $n$ lines, the matrix $\boldsymbol{\xi}_{(n)}$ in~\eqref{eq:zetan} has a kernel which represents a line at infinity with Pl\"ucker coordinates $(\mathbf{0},\mathbf{l})$ and hence it is singular. 
\end{proof}
For $n=3$, the locus of lines intersecting the three observed lines is a hyperboloid of one sheet. For $n=4$, in the generic case, it is two lines and the observed lines belong to a linear line congruence~\citep{Pottmann_Wallner}. For $n=5$, it is just a line and the observed lines belong to a linear line complex~\citep{Pottmann_Wallner}.}

\begin{figure}[t]
	\centering
	\includegraphics[width=0.8\linewidth]{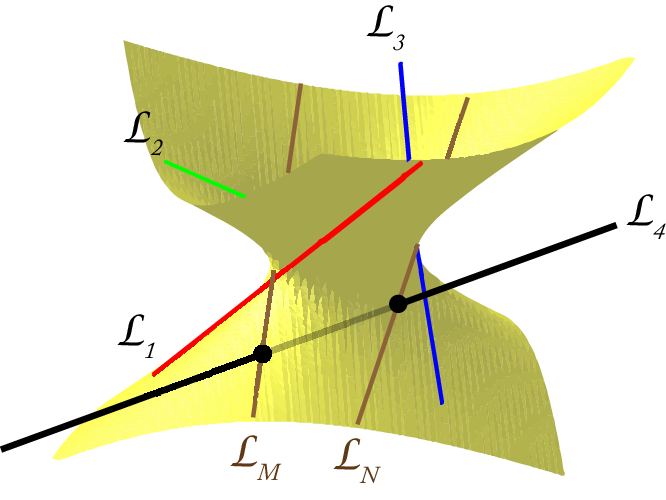}
	\caption[simple]{Four observed lines $\mathcal{L}_i,\,i=1,2,3,4$ in a hyperbolic congruence leading to two singular lines $\mathcal{L}_M$ and $\mathcal{L}_N$.\label{fig:cong}}
\end{figure}

{Moreover, the variety of the ideal $\mathcal{G}$ in~\eqref{eq:idealG_4lines} consists of two lines which are also the intersections of the four hyperboloids $\mathcal{G}_{123}=\mathcal{G}_{124}=\mathcal{G}_{134}=\mathcal{G}_{134}=0$. These transversal lines can be real or complex.}

In the real domain $\mathbb{R}[X,Y,Z]$, the intersection of the four hyperboloids i.e. $\mathbf{V}(\mathcal{G})$ can be an empty set, a line or two lines. If we consider the hyperboloid defined by the first three observed lines, assuming the fourth line does not lie entirely on the hyperboloid, it can intersect the hyperboloid in $0,1$ or $2$ points. Then, the four lines are said to be in an elliptic, a parabolic or a hyperbolic line congruence, respectively~\citep{Pottmann_Wallner}. A case of hyperbolic congruence is shown in Fig.~\ref{fig:cong}. A line passing through the point of intersection and lying on the hyperboloid intersects all four lines. By forcing the four observed lines to be in an elliptic congruence, we can avoid these singularities {as follows}.

By finding the Gr\"obner basis of the ideal $\mathcal{G}$ with respect to the monomial ordering $Y\succ Z \succ X$, we obtain four polynomials, of which the first polynomial is as follows:
\begin{align} \label{eq:discr}
	&a_2Y^2+a_1Y+a_0=0, \\
	& \textrm{ where }  \nonumber \\
	&a_2=d_{{1}}s_{{4}} \left( d_{{2}}s_{{2}}s_{{6}}-d_{{3}}s_{{1}}s_{{6}}-d_{{
			4}}d_{{5}}s_{{1}}+d_{{4}}s_{{1}}s_{{6}}+d_{{5}}s_{{2}}s_{{5}} \right), \nonumber \\
   &a_1= ( d_{{1}}d_{{2}}d_{{3}}s_{{2}}s_{{6}}-d_{{1}}d_{{2}}d_{{4}}s_{{2
   }}s_{{4}}-d_{{1}}d_{{2}}s_{{2}}s_{{3}}s_{{6}}+d_{{1}}d_{{3}}d_{{4}}s_{
   	{1}}s_{{4}} \nonumber \\
   &\quad +d_{{1}}d_{{4}}d_{{5}}s_{{1}}s_{{3}}-d_{{1}}d_{{4}}s_{{1}}s
   _{{3}}s_{{6}}-d_{{1}}d_{{4}}s_{{2}}s_{{4}}s_{{5}}-d_{{1}}d_{{5}}s_{{2}
   }s_{{3}}s_{{5}}\nonumber \\
& \quad +d_{{2}}d_{{4}}d_{{5}}s_{{2}}s_{{4}}-d_{{2}}d_{{4}}s_{{
   		2}}s_{{4}}s_{{6}}-d_{{3}}d_{{5}}s_{{2}}s_{{4}}s_{{5}} ) Z, \nonumber \\
   &a_0=-d_{{4}}s_{{2}} ( d_{{1}}d_{{2}}d_{{3}}-d_{{1}}d_{{2}}s_{{3}}-d_{
   	{1}}s_{{3}}s_{{5}}-d_{{2}}d_{{3}}d_{{5}}+d_{{2}}d_{{3}}s_{{6}}\nonumber \\
   & \quad +d_{{2}}
   d_{{5}}s_{{3}}-d_{{2}}s_{{3}}s_{{6}}-d_{{3}}s_{{4}}s_{{5}} ) {Z}
   ^{2}. \nonumber
\end{align}

{Since the variety of $\mathcal{G}$ represents two real or complex lines, the quadratic element~\eqref{eq:discr} must factorize into two linear polynomials which represent the planes containing the two lines that are transversal to the four observed lines (the remaining elements of the Gr\"obner basis can be used to deduce the equations of the two lines that constitute $\mathbf{V}(\mathcal{G})$; see Appendix~A for an example). These planes and hence the transversals lying on them are either real or complex depending on the sign of the discriminant of~\eqref{eq:discr}. The equation~\eqref{eq:discr} has no real solutions if the discriminant
\begin{equation}
    \Delta=a_1^2-4a_2a_0=Z^2f(\boldsymbol{\eta})<0,
\end{equation}
where $f(\boldsymbol{\eta})$ is a function of the parameters $\boldsymbol{\eta}$. Since $Z^2 \geq 0$, a necessary condition that there are no real transversals intersecting the four observed lines can be given by an inequality solely in terms of the Pl\"ucker coordinates of the four observed lines. This helps us avoid the one dimensional singularities due to $\mathbf{V}(\mathcal{G})$. This is quite useful as we will see in the next section that the remaining singularities due to $\mathbf{V}(\mathcal{H}) \cap \mathbf{V}(\mathcal{K})$ are only of dimension zero, implying that they are isolated points.}

{{\bf All in all, in the case of P4L, we have a singularity when $C$ lies on any transversal line that intersects the four observed lines. The transversal appears as the intersection of the hyperboloids defined by four triplets of the observed lines.}

{{\bf Moreover, by forcing the four observed lines to be in an elliptic congruence, we can make sure that their transversals are not real and therefore avoid the one dimensional singularities.}}

\subsection{{Analysis of the variety $\mathbf{V}(\mathcal{H}) \cap \mathbf{V}(\mathcal{K})$}} \label{subsec:VH}

The analysis of the component {$\mathbf{V}(\mathcal{H}) \cap \mathbf{V}(\mathcal{K})$} is slightly more involved (see the attached Maple file). We obtain a Gr\"obner Basis $gb_H=\{h_1, \dots, h_t\}$ for the ideal $\mathcal{H}$ w.r.t. the ordering with $X \succ_{lex} Y \succ_{lex} Z$,
and compute the normal form of the minors $p_i \in \mathcal{K}$ with respect to it. Note that we now consider polynomials in the ring $\mathbb{Q}(\boldsymbol{\eta})[X,Y,Z]$ of polynomials in $X$, $Y$ and $Z$ alone. This time, the residues are polynomials in $X$, $Y$ and $Z$ with coefficients that depend on the parameters:
\begin{equation}
\label{eq:4lines_NMH}
\overline{p_i}^{\langle h_1, \dots, h_t\rangle} = f_i \neq 0 \in \mathbb{Q}(\boldsymbol{\eta})[X,Y,Z], \ i=1,\ldots,6.
\end{equation}
{Since the residues $f_i$ do not vanish, unlike~\eqref{eq:VGinVI28}, $\mathbf{V}(\mathcal{H}) \not \subseteq \mathbf{V}(\mathcal{K})$. It implies that any common solution of the system $h_1 = \dots = h_t = 0$ is not a solution of the system $p_1 = \dots = p_{6} = 0$.
However, the analysis of $\mathbf{V}(\mathcal{H}) \cap \mathbf{V}(\mathcal{K})$ can be simplified by noting that $\mathbf{V}(\mathcal{K})$ contains the four observed lines and their two transversals. This is because the matrix $\boldsymbol{\xi}^{1234}_{lm}$ in~\eqref{eq:subm3} loses rank if $C$ lies on the observed four lines or their transversals. As a consequence, $\mathbf{V}(\mathcal{H}) \cap \mathbf{V}(\mathcal{K})$ might contain points on the four observed lines and their two transversals. As we know from Section~\ref{subsec:VG} that they are the singularity loci corresponding to $\mathbf{V}(\mathcal{G})$, we would like to remove them from the variety $\mathbf{V}(\mathcal{H}) \cap \mathbf{V}(\mathcal{K})$. We know that these six lines must lie in the union of four hyperboloids $\mathbf{V}(\mathcal{G}_{123}) \cup \mathbf{V}(\mathcal{G}_{124}) \cup \mathbf{V}(\mathcal{G}_{134}) \cup \mathbf{V}(\mathcal{G}_{234})$ that appear in~\eqref{eq:idealG_4lines}. 
Therefore, we can remove each hyperboloid from $\mathbf{V}(\mathcal{H}) \cap \mathbf{V}(\mathcal{K})$ to obtain the remaining singularities. In algebraic geometry terms, removing one variety from the other amounts to finding the difference of the varieties~\citep{cox2013ideals}.}

The \emph{set difference} of two affine varieties is generally not an affine variety but an open subset of a variety. It cannot be written as the set of solutions of a system of polynomial equations (it is not an affine variety). The smallest affine variety which contains it, is called the \emph{Zariski closure} of the difference, denoted with an overline. Loosely speaking, taking the Zariski closure amounts to patching up the holes in the open set. Therefore, in this case, we need to find the following Zariski closure of the difference:
\begin{equation}\label{eq:VF}
	\mathbf{V}(\mathcal{F})=\bigcap_{i=1}^{4} \overline{ \left(\mathbf{V}(\mathcal{K}) \cap \mathbf{V}(\mathcal{H}) \right) \setminus \mathbf{V}(S_i)},
\end{equation}
{where $S_i$ is an element of $\{\mathcal{G}_{123}$, $\mathcal{G}_{124}$, $\mathcal{G}_{134}$, $\mathcal{G}_{234}\}$.}

Using the correspondence between ideals and varieties, an ideal defining $\overline{V \setminus W}$ where $V$ and $W$ are affine varieties is obtained as the \emph{saturation} of an ideal $I$ defining $V$ with an ideal $J$ defining $W$. It is denoted by $I:J^\infty$. The colon symbolizes the \emph{quotient} of one ideal w.r.t. the other, and it removes factors from the polynomials in the first ideal which appear as polynomials themselves in the second ideal. The infinity symbol corresponds to considering all products of arbitrarily many polynomials in the second ideal. It can be done in a computational environment such as Maple or Singular and we obtain:
\begin{equation}\label{eq:idealF}
	\mathcal{F}=\bigcup_{i=1}^{4}\left(\mathcal{H} \cup \mathcal{K} \right) :S_i^{\infty},
\end{equation}

Due to a large number of variables leading to heavy computations, we did not succeed in determining $\mathcal{F}$ in~\eqref{eq:idealF} using the above mentioned approach. Therefore, Appendix~A shows an example where $\mathbf{V}(\mathcal{H}) \cap \mathbf{V}(\mathcal{K})$ is analysed for some specialization of the parameters $s_1,\ldots, s_6$, $d_1, \ldots, d_5$. 

Since we are dealing here with polynomial systems, we know that for 
almost all values of the parameters, the specialized systems have all the same number
of complex solutions~\citep{cox2013ideals}. More precisely, there
exists a polynomial $B$ depending on the parameters, such that when specializing the
parameters outside the zero set of $B$, the number of complex solutions to the system
that we obtain remains invariant.

In our analysis, we have observed that when specializing the parameters to random values and removing those solutions lying on the lines, one 
always obtains $10$ complex solutions. This indicates that for generic values of the parameters (outside this zero set of polynomial B), there are at most $10$ isolated 
singularities in the case of P4L. {Appendix~A shows one such example where a random specialization of parameters yields $10$ complex solutions of which $6$ are real.}

Now that the generic case is treated, let us deal with a more specific case. Indeed, it is often the case that the observed lines in an environment are constrained with orthogonality and/or parallelism. {This} special case is considered in the next section and the singularities are determined with the proposed approach without specializing any parameters.

\subsection{Singularities in P4L with orthogonality and parallelism} \label{subsec:P4Lop}

We consider three mutually orthogonal lines $\mathcal{L}_1$, $\mathcal{L}_2$ and $\mathcal{L}_3$, and a fourth one $\mathcal{L}_4$ with direction parallel to $\mathcal{L}_1$. The parametrization~\eqref{eq:param} and~\eqref{eq:param_l4} cannot be used in this context since we need the lines to only intersect one of the planes $x_o=0$, $y_o=0$ or $z_o=0$. The object frame $\mathcal{F}_o: (O, \mathbf{x}_o, \mathbf{y}_o, \mathbf{z}_o)$ is fixed relative to the four lines, with its axes defining an orthonormal, right-handed basis, and such that $\mathbf{x}_o$ is collinear to $\mathcal{L}_1$ and $\mathcal{L}_4$; $\mathbf{y}_o$ is collinear to $\mathcal{L}_2$, and $\mathbf{z}_o$ is collinear to $\mathcal{L}_3$. With this parametrization in the object frame, the direction vector $\mathbf{u}_i$ of the four lines and the coordinates of points $M_i$ belonging to each of them are given by:
\begin{align} \label{eq:param2}
		&\overrightarrow{OM_1}=[0,0,0]^T,\;\mathbf{u}_1=[1,0,0]^T, \nonumber \\
		&\overrightarrow{OM_2}=[0,0,d_1]^T,\;\mathbf{u}_2=[0,1,0]^T, \nonumber \\
		&\overrightarrow{OM_3}=[d_2,d_3,0]^T,\;\mathbf{u}_3=[0,0,1]^T, \nonumber \\
		&\overrightarrow{OM_4}=[0,d_4,d_5]^T,\;\mathbf{u}_4=[1,0,0]^T.
\end{align}

{Following the analysis done in the preceding section, the varieties $\mathbf{V}(\mathcal{G})$ and $\mathbf{V}(\mathcal{H}) \cap \mathbf{V}(\mathcal{K})$ will be analyzed separately.}

{In this context, the ideal $\mathcal{G}$ in~\eqref{eq:idealG_4lines} is calculated as follows:}
\begin{dmath}
    \mathcal{G}  = \langle XY(d_{{1}}-d_{{5}})-XZ(d_{{3}}-d_{{4}})-X(d_{{1}}d_{{4}}-d_{{3}}d_{{5}})+YZd_{{2}}-Yd_{{1}}d_{{2}}-Zd_{{2}}d_{{4}}+d_{{1}}d_{{2}}d_{{4}}, \,
 \left( -d_{{3}}+Y \right)  \left( Yd_{{5}}-Zd_{{4}} \right) , \, - \left( -d_{{1}}+Z \right)  \left( Yd_{{5}}-Zd_{{4}} \right) , \, XYd_{{1}}-XZd_{{3}}+YZd_{{2}}-Yd_{{1}}d_{{2}} \rangle.
\end{dmath}
{According to Theorem~\ref{thm:transversals}, we expect the positive dimensional singularities corresponding to $\mathbf{V}(\mathcal{G})$ to be the transversals that intersect the four observed lines. It can be verified by finding the Gr\"obner basis $gb_{G}$ of $\mathcal{G}$ w.r.t. $Z\succ_{lex} Y \succ_{lex} X$. It consists of four elements $\{g_1,g_2,g_3,g_4\}$, where the first element factors as follows:}
\begin{align} \label{eq:gbGv}
    g_1=\left( Z-d_1 \right)\left( Yd_5-Zd_4 \right).
\end{align}
By substituting the factors into the other elements of $gb_{G}$, $\mathcal{G}_v$ can be decomposed into two sub-ideals {$\mathcal{M}$ and $\mathcal{N}$} whose varieties correspond to the two transversal lines intersecting all the four observed lines:
\begin{align} \label{eq:transversal_special}
    \mathcal{G} & =\mathcal{M} \cap \mathcal{N}, \textrm{ where }\nonumber \\
    \mathcal{M} & =\langle Z-d_1, Y-d_3  \rangle, \nonumber \\
      \mathcal{N} & =\langle Yd_5-Zd_4, X(d_1d_4-d_3d_5)+Yd_2d_5-d_1d_2d_4  \rangle.
\end{align}
{In the generic case of P4L, we showed that the one dimensional singularities can be avoided by choosing the lines such that they satisfy~\eqref{eq:discr}. It is a condition on the discriminant of the quadratic polynomial that appears as the first element of the Gr\"obner basis of $\mathcal{G}$. Similarly, here, the first polynomial of $gb_{G}$ in~\eqref{eq:gbGv} is quadratic in $Z$ and its discriminant is $(Yd_5-d_1d_4)^2$, which is always positive.} Hence the two transversal lines given by~\eqref{eq:transversal_special} are always real and the singularity cannot be avoided when the four observed lines adhere to the orthogonality and parallelism conditions of this section.

{To analyse the remaining singularities, we need to determine the variety $\mathbf{V}(\mathcal{H}) \cap \mathbf{V}(\mathcal{K})$.
To do so, the ideal $\langle \mathcal{H}, \mathcal{K} \rangle$ is considered (it is too large to be displayed here) and its Gr\"obner basis calculated:}
\begin{dmath} \label{eq:gbHvfi}
gb_{HK}=\{Z \left( -d_{{5}}+Z \right)  \left( Zd_{{4}}-d_{{3}}d_{{5}} \right) 
 \left( -d_{{1}}+Z \right) 
,\ Yd_{{5}}-Zd_{{4}},-d_{{1}}d_{{2}}d_{{4}} \left( d_{{1}}d_{{4}}+d_{{3}}d_{{5}}-2\,d_{{4}}
d_{{5}} \right)  \left( d_{{1}}d_{{4}}-d_{{3}}d_{{5}} \right) +
 \left( d_{{1}}d_{{4}}+d_{{3}}d_{{5}}-2\,d_{{4}}d_{{5}} \right) 
 \left( d_{{1}}d_{{4}}+d_{{3}}d_{{5}} \right)  \left( d_{{1}}d_{{4}}-d
_{{3}}d_{{5}} \right) X-d_{{2}}d_{{4}} \left( {d_{{1}}}^{2}{d_{{4}}}^{
2}-2\,d_{{1}}d_{{3}}d_{{4}}d_{{5}}-2\,d_{{1}}{d_{{4}}}^{2}d_{{5}}+{d_{
{3}}}^{2}{d_{{5}}}^{2}-2\,d_{{3}}d_{{4}}{d_{{5}}}^{2} \right) Z+4\,d_{
{2}}{d_{{4}}}^{3}{Z}^{3}-2\,d_{{2}}{d_{{4}}}^{2} \left( d_{{1}}d_{{4}}
+d_{{3}}d_{{5}}+2\,d_{{4}}d_{{5}} \right) {Z}^{2} \}.
\end{dmath}

\begin{figure}[t]
	\centering
	\includegraphics[width=0.8\linewidth]{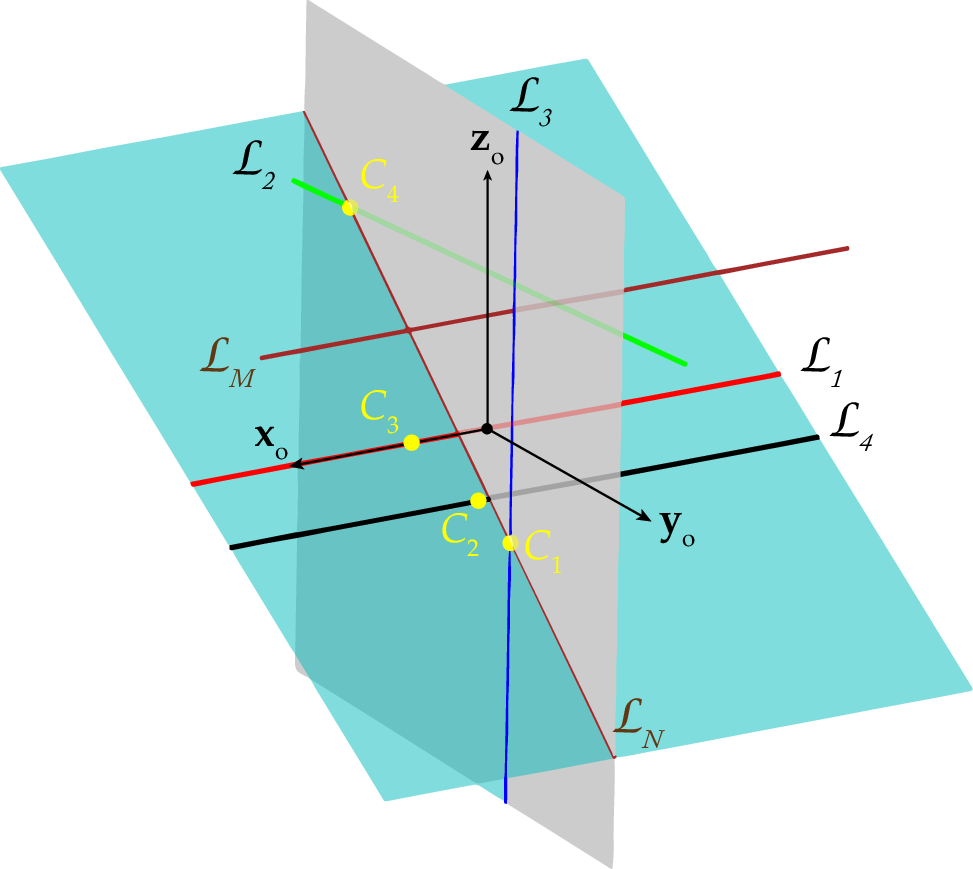}
	\caption[simple]{Singularities in P4L with orthogonality and parallelism constraints: Four observed lines $\mathcal{L}_i,i=1,2,3,4$ and their traversals $\mathcal{L}_M$ and $\mathcal{L}_N$. \label{fig:orthpar}}
\end{figure}
{The variety of $gb_{HK}$ is zero dimensional with degree 4 (see the attached Maple file). It implies that it is made up of 4 points in $\mathbb{C}[X,Y,Z]$ as shown in Fig.~\ref{fig:orthpar} whose coordinates are as follows:
\begin{align}
    &C_1=\left( 0, \dfrac{d_1d_4}{d_5},d_1 \right) \; ; \;  C_2=\left( \dfrac{(d_1-d_5)d_2d_4}{d_1d_4+d_3d_5-2d_4d_5},d_4,d_5 \right); \nonumber \\
     &C_3=\left(\dfrac{d_1d_2d_4}{d_1d_4+d_3d_5},0,0 \right) \; ; \; C_4=\left(d_2,d_3, \dfrac{d_3d_5}{d_4} \right).
\end{align}
}As mentioned in Section~\ref{subsec:VH}, some or all of these points might lie on the four observed lines or their two singular transversals, which we know for sure belong to the singularity loci. To acknowledge that, the saturation ideal $\mathcal{F}$ of~\eqref{eq:idealF} can be determined to check if $C_i \in \mathbf{V}(\mathcal{F})$, because only then, $C_i$ is a singularity locus. However, the Gr\"obner basis of $\mathcal{F}$ yields $\{ 1 \}$. By Hilbert's Nullstellensatz, when the Gr\"obner basis of an ideal is $\{ 1 \}$, its generators do not have a common solution~\citep{cox2013ideals} and hence, $\mathbf{V}(\mathcal{F})=\emptyset$. Thus, $C_i \not \in \mathbf{V}(\mathcal{F})$ for any $i=1,2,3,4$ implying that these points indeed lie on the four observed lines ($\mathcal{L}_1,\mathcal{L}_2,\mathcal{L}_3,\mathcal{L}_4$) or their transversals ($\mathcal{L}_M,\mathcal{L}_N$) as shown in Fig.~\ref{fig:orthpar} for some randomly chosen parameters $d_1,..,d_5$.}

{\bf Therefore, for the special case of P4L considered in this section, the only singularities are when $C$ lies on the four observed lines or their two transversal lines.}

{In fact, the results in this section remain valid under any permutation of the four observed lines and hence the analysis entails singularities in the case of four lines subject to any other orthogonality and parallelism constraints.} 

Let us now deal with the case of five lines observed by a camera.

\section{Singularities in P5L} 
\label{sec:singP5L}

\subsection{Parametrization}
{Let us consider the first three lines defined by the parametrization~\eqref{eq:param} and the fourth one by~\eqref{eq:param_l4}. In the same vein, the fifth line is defined using its two points of intersections $M_5$ and $N_5$ with the planes $y_o=0$ and $z_o=0$, respectively:  
\begin{align} 
\overrightarrow{OM_5}=[d_6,0,d_7]^T,\;\overrightarrow{ON_5}=[r_7,r_8,0]^T.
\end{align}
Thus, the direction vector of the fourth line is given by $\mathbf{u}_5=\overrightarrow{OM_5}-\overrightarrow{ON_5}$. After changing the variables $-r_8=s_8$ and $d_6-r_7=s_7$, we have
\begin{align} \label{eq:param_l5}
\overrightarrow{OM_5}=[d_6,0,d_7]^T,\;\mathbf{u}_5=[s_7,s_8,d_7]^T.
\end{align}
As mentioned in Section~\ref{subsec:paramP3L}, the orientation of the camera frame $\mathcal{F}_c$ and of the object frame $\mathcal{F}_o$ is considered the same also for the following analysis.}

\subsection{Singularity analysis}

From Section~\ref{subsec:newbasis}, we know that the rows of the interaction matrix $\mathbf{M}_{(5)}$ associated with the five lines represent a system of Pl\"ucker lines and that a basis for this system of lines can be expressed as $\boldsymbol{\xi}_{(5)} = [\boldsymbol{\xi}_{11}^T \ \boldsymbol{\xi}_{12}^T \ \dots \ \boldsymbol{\xi}_{51}^T \ \boldsymbol{\xi}_{52}^T]^T \in \mathbb{R}^{10\times 6}$ with $\boldsymbol{\xi}_{i1}$ and $\boldsymbol{\xi}_{i2}$ given by~\eqref{eq:basiseta}.

Let us then consider the ideal $\mathcal{I}_{210}$ generated by the maximal minors of $\boldsymbol{\xi}_{(5)}$, which in this case forms a system of $210$ polynomials $p_i$ in the variables $X,Y,Z$ and the parameters $\boldsymbol{\eta}=\{s_1,\ldots,s_8, d_1,\ldots,d_7\}$: $\mathcal{I}_{210} = \langle p_1, \dots, p_{210} \rangle$. 

Following a similar analysis of Section~\ref{sec:singP4L}, the 210 minors can be divided into three categories. $55$ of them are zero since there are $55$ submatrices containing four or five ideal lines. These matrices always have a rank at most 5 since only three ideal lines can be independent.

The second category are the determinants of block triangular matrices, composed of row vectors that represent three affine lines and three ideal lines. There are $\binom{5}{3} \times \binom{5}{3} =100$ of them.
Let these minors generate a subideal $\mathcal{I}_{100} \subseteq \mathcal{I}_{210}$. The generators of $\mathcal{I}_{100}$ are of the form $p_i = G_{ijk} H_{lmn}$, that are the products of the polynomials in~\eqref{eq:Gijk}, with the indices $\{i,j,k\}$ and $\{l,m,n\}$ ranging all triplets of numbers in $\{1,2,3,4,5\}$. Therefore, this ideal is the product of two smaller ideals: $\mathcal{I}_{100} = \mathcal{G} \times \mathcal{H}$, generated by $10$ polynomials each:
\begin{equation}
\label{eq:GijkHlmn_5lines}
\mathcal{G} = \langle G_{123},\dots, G_{235} \rangle, \quad \mathcal{H} = \langle H_{123},\dots, H_{235} \rangle.
\end{equation}

The generators of $\mathcal{G}$ and $\mathcal{H}$ are too long to be given here but $\mathcal{G}$ (resp. $\mathcal{H}$) is generated by polynomials of degree $2$ (resp. $3$) in $\Q[X,Y,Z]$ since its variety corresponds to the singular hyperboloids (resp. cubic surfaces) of P3L~(see Section~\ref{subsec:geomP3L}).

The remaining $55$ minors with degree 5 each in $\{X,Y,Z\}$ constitute the last category. Let them generate a subideal $\mathcal{K}_{55} \subseteq \mathcal{I}_{210}$.

{As before, we deduce that the solution set of the polynomials in $\mathcal{I}_{210}$ is contained in a larger variety which is the union of two varieties (see~\eqref{eq:KuGiH}):
\begin{align}
    \label{eq:KuGiHP5P}
 \mathbf{V}(\mathcal{I}_{210}) =  \left(\mathbf{V}(\mathcal{G}) \cap \mathbf{V}(\mathcal{K}_{55})\right) \cup \left(\mathbf{V}(\mathcal{H}) \cap \mathbf{V}(\mathcal{K}_{55})\right).
\end{align}
Our strategy will be to use the geometrical interpretation of previous sections wherever possible or else to use Gr\"obner Bases computations and multivariate polynomial division to analyse the varieties $\left(\mathbf{V}(\mathcal{G}) \cap \mathbf{V}(\mathcal{K}_{55})\right)$ and $\left(\mathbf{V}(\mathcal{H}) \cap \mathbf{V}(\mathcal{K}_{55})\right)$ separately.} Again, all mathematical derivations can be followed on the attached Maple file.

\subsection{{Analysis of the variety $\left(\mathbf{V}(\mathcal{G}) \cap \mathbf{V}(\mathcal{K}_{55})\right)$}} \label{subsec:VG5}

Let $gb_G = \langle g_1, \dots, g_s \rangle$ be a Gr\"obner Basis for the ideal
$\mathcal{G}$ with respect to the aforementioned lexicographical monomial ordering. We are again treating the parameters $\boldsymbol{\eta} = \{d_1,\dots, d_7, s_1,\ldots,s_8\}$ as variables here; that is, we are considering the polynomials in the ring $\mathbb{Q}[X,Y,Z,\boldsymbol{\eta}]$. Then, we can compute the normal form of the polynomials $p_i \in \mathcal{K}_{55}$ with respect to this basis to find $\overline{p_i}^{\langle g_1, \dots, g_s
\rangle}$ for all $i$. {We find $p_i=0$ for any $i$, implying that $\left(\mathbf{V}(\mathcal{G}) \cap \mathbf{V}(\mathcal{K}_{55})\right) = \mathbf{V}(\mathcal{G})$.} We already know from Theorem~\ref{thm:transversals} that $\mathbf{V}(\mathcal{G})$ should be a line that intersects the five observed lines. We cannot always find a line that intersects the given five lines unless they belong to a singular linear line complex~\citep{Pottmann_Wallner}. Thus, there must be a condition on the parameters such that $\mathbf{V}(\mathcal{G}) \neq \emptyset$. To find it, we consider a matrix with rows consisting of the Pl\"ucker coordinates of the observed lines $[\mathbf{u}_i,\ \overrightarrow{{OM}_i} \times \mathbf{u}_i]$ whose moment vectors are defined by considering the point $Q$ in~\eqref{eq:basiseta} as the origin of the object frame $O$:
\begin{equation}
     \mathbf{L}_5= \left[ \begin {array}{cccccc} 1&0&0&0&0&0\\ \noalign{\medskip}s_{{1}}
&s_{{2}}&0&d_{{1}}s_{{2}}&-d_{{1}}s_{{1}}&0\\ \noalign{\medskip}d_{{2}
}&s_{{3}}&s_{{4}}&-s_{{4}}d_{{3}}&s_{{4}}d_{{2}}&d_{{2}}d_{{3}}-d_{{2}
}s_{{3}}\\ \noalign{\medskip}s_{{5}}&d_{{4}}&s_{{6}}&d_{{4}}d_{{5}}-d_
{{4}}s_{{6}}&-s_{{5}}d_{{5}}&s_{{5}}d_{{4}}\\ \noalign{\medskip}s_{{7}
}&s_{{8}}&d_{{7}}&s_{{8}}d_{{7}}&d_{{6}}d_{{7}}-d_{{7}}s_{{7}}&-s_{{8}
}d_{{6}}\end {array} \right] .
\end{equation}
A line intersects the observed five lines only if the kernel $\mathbf{k}=[k_1,k_2,k_3,k_4,k_5,k_6]^T$ of the matrix $\mathbf{L}_5$ satisifies the Pl\"ucker relation $k_1k_4+k_2k_5+k_3k_6=0$. The first row of $\mathbf{L}_5$ imposes $k_1=0$. Eliminating $k_i$ from the remaining four equations $\mathbf{L}_5\mathbf{k}=0$ and the Pl\"ucker relation leads to a polynomial $h$ of degree 13 solely in terms of parameters $\boldsymbol{\eta}$ (see the attached Maple file).


{\bf It implies that there are no one dimensional singularities for P5L when the four observed lines are generic. Some conditions on the relative configurations of the lines must be satisfied for a transversal to exist so that the incidence of the camera center $C$ on it makes it the singularity locus.}

The results of this section are substantiated through an example in Appendix~B.

\subsection{{Analysis of the variety $\mathbf{V}(\mathcal{H}) \cap \mathbf{V}(\mathcal{K}_{55})$}} \label{subsec:VH5}

Due to the complexity of equations, it was not possible to compute a Gr\"obner Basis $gb_H = \{h_1, \dots, h_t\}$ for the ideal generated by 
$\mathcal{H}$ in $\mathbb{Q}(\boldsymbol{\eta})[X,Y,Z]$. Thus, the normal forms of the minors $p_i \in \mathcal{K}_{55}$, $\overline{p_i}^{\langle h_1, \dots, h_3\rangle} = f_i$ with respect to the basis $gb_H$ could not be calculated. {Hence, $\mathbf{V}(\mathcal{H}) \cap \mathbf{V}(\mathcal{K}_{55})$ is analysed in Appendix~B for some arbitrarily chosen parameters $\boldsymbol{\eta}$.} As in the end of Section~\ref{sec:singP4L}, this indicates that there exists, in the 
space of parameters, the zero set of some polynomial, depending on the parameters, such that, specializing the parameters outside this zero set yields no isolated singularity {as substantiated by an example in Appendix~A}.



Like P4L, it applies to P5L as well that, often, the observed lines in an environment are constrained with orthogonality and/or parallelism. One of these special cases is considered in the next section and the singularities are determined with the proposed approach without specializing any parameters.

\subsection{Singularities in P5L with orthogonality and parallelism} \label{subsec:P5Lop}

As a continuation of Section~\ref{subsec:P4Lop}, let us consider a fifth line $\mathcal{L}_5$, which is assumed to be collinear with axis $\mathbf{y_o}$. This way, we have three orthogonal lines $\mathcal{L}_1$, $\mathcal{L}_2$ and $\mathcal{L}_3$, a line $\mathcal{L}_4$ parallel to $\mathcal{L}_1$, and a line $\mathcal{L}_5$ parallel to $\mathcal{L}_2$. The location of lines $\mathcal{L}_1$ to $\mathcal{L}_4$ relative to frame $\mathcal{F}_o$ is still given by~\eqref{eq:param2}, and we parametrize $\mathcal{L}_5$ by
\begin{equation}
\label{eq:param_5lines}
\vv{OM_5} = [ d_6, \ 0, \ d_7 ]^T, \, \mathbf{u}_5 = [0, \ 1, \ 0]^T.
\end{equation}
As before, if the line $\mathcal{L}_5$ was instead given parallel to line $\mathcal{L}_3$, these parametrization will still be valid upon a redefinition of the object frame $\mathcal{F}_o$ and the renaming of the lines.


{Following the analysis done in the preceding section, the varieties $\mathbf{V}(\mathcal{G})$ and $\mathbf{V}(\mathcal{H}) \cap \mathbf{V}(\mathcal{K}_{55})$ will be analysed separately.} {In this context, the ideal $\mathcal{G}$ in~\eqref{eq:GijkHlmn_5lines} is calculated as follows:}
\begin{dmath}\label{eq:Gv5}
    \mathcal{G}  = \langle -XYd_{{5}}+XYd_{{7}}-XZd_{{3}}+XZd_{{4}}+Xd_{{3}}d_{{5}}-Xd_{{4}}d_{{
7}}+YZd_{{2}}-YZd_{{6}}-Yd_{{2}}d_{{7}}+Yd_{{5}}d_{{6}}-Zd_{{2}}d_{{4}
}+Zd_{{3}}d_{{6}}+d_{{2}}d_{{4}}d_{{7}}-d_{{3}}d_{{5}}d_{{6}}, \left( 
-d_{{5}}+Z \right)  \left( Xd_{{1}}-Xd_{{7}}+Zd_{{6}}-d_{{1}}d_{{6}}
 \right) ,- \left( -d_{{2}}+X \right)  \left( Xd_{{1}}-Xd_{{7}}+Zd_{{6
}}-d_{{1}}d_{{6}} \right) ,XYd_{{1}}-XYd_{{5}}-XZd_{{3}}+XZd_{{4}}-Xd_
{{1}}d_{{4}}+Xd_{{3}}d_{{5}}+YZd_{{2}}-Yd_{{1}}d_{{2}}-Zd_{{2}}d_{{4}}
+d_{{1}}d_{{2}}d_{{4}}, \left( -d_{{7}}+Z \right)  \left( Yd_{{5}}-Zd_
{{4}} \right) ,-XYd_{{7}}+XZd_{{3}}-YZd_{{2}}+YZd_{{6}}+Yd_{{2}}d_{{7}
}-Zd_{{3}}d_{{6}}, \left( -d_{{3}}+Y \right)  \left( Yd_{{5}}-Zd_{{4}}
 \right) ,-Z \left( Xd_{{1}}-Xd_{{7}}+Zd_{{6}}-d_{{1}}d_{{6}} \right) 
,- \left( -d_{{1}}+Z \right)  \left( Yd_{{5}}-Zd_{{4}} \right) ,XYd_{{
1}}-XZd_{{3}}+YZd_{{2}}-Yd_{{1}}d_{{2}} \rangle.
\end{dmath}

We expect that the one dimensional singularity corresponding to $\mathbf{V}(\mathcal{G})$ must be the transversal that intersects the five observed lines according to Theorem~\ref{thm:transversals}. However, we know from Section~\ref{subsec:VG5} that the parameters used to define the five observed lines must satisfy a condition for this transversal line to exist. As before, we can determine it by imposing the Pl\"ucker relation on the kernel of the matrix whose rows are the Pl\"ucker coordinates of the observed lines $[\mathbf{u}_i,\ \overrightarrow{OM_i} \times \mathbf{u}_i]$ parametrized by~\eqref{eq:param2} and~\eqref{eq:param_5lines}. This condition leads to the following polynomial in terms of the parameters that should be zero.
\begin{align}
    h(\boldsymbol{\eta})=d_{{5}} \left( d_{{1}}-d_{{7}} \right)  \left( d_{{1}}d_{{2}}d_{{4}}-d
_{{1}}d_{{4}}d_{{6}}-d_{{2}}d_{{4}}d_{{7}}+d_{{3}}d_{{5}}d_{{6}}
 \right). 
\end{align}
This polynomial can also be derived by finding the Gr\"obner basis $gb_{G}$ of $\mathcal{G}$ in~\eqref{eq:Gv5} w.r.t. the ordering $Z\succ_{lex} Y \succ_{lex} X \succ_{lex} d_1 \succ_{lex} \ldots \succ_{lex} d_7$. It consists of ten elements of which the first element is exactly $h(\boldsymbol{\eta})$. We look for conditions when $h=0$ so that $\mathbf{V}(\mathcal{G}_v) \neq \emptyset$.

When the first factor of $h$ vanishes, i.e. $d_5=0$, we get the line as the variety of the following ideal:
\begin{align}
    \mathcal{M}_{1}=\langle Z, X-d_2 \rangle.
\end{align}
When $d_1-d_7=0$, we get
\begin{align}
    \mathcal{M}_{2}=\langle Z-d_1, Y-d_3 \rangle.
\end{align}
Finally, when the third factor of $h$ vanishes, we have
\begin{align}
    \mathcal{M}_{3}=\langle Yd_5-Zd_4, X(d_1-d_7)+Zd_6-d_1d_6 \rangle.
\end{align}

\begin{figure}[t]
	\centering
	\includegraphics[width=0.8\linewidth]{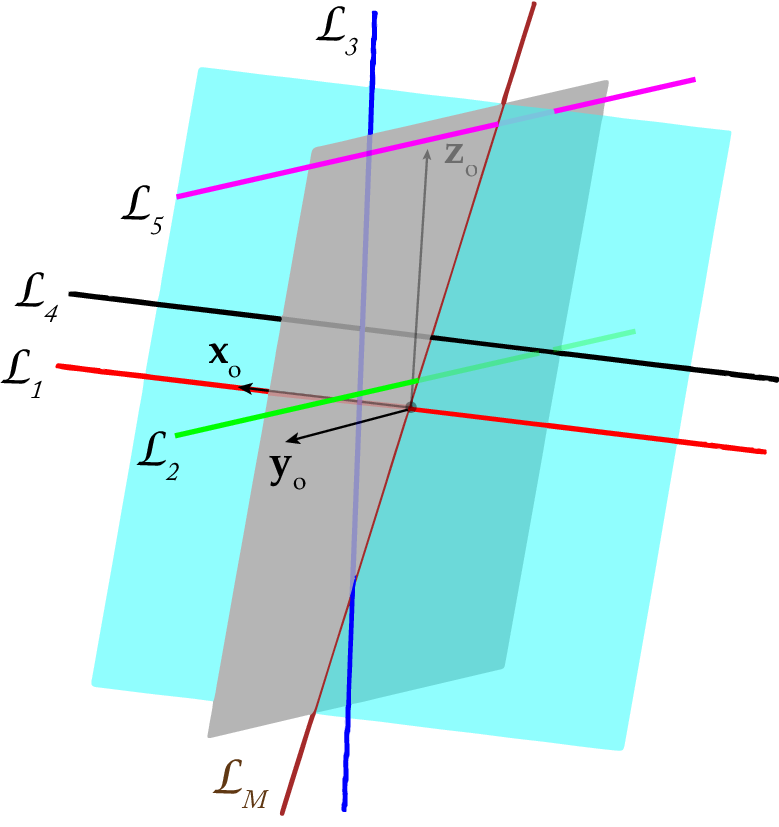}
	\caption[simple]{Singularities in P5L with orthogonality and parallelism constraints: Five observed lines $\mathcal{L}_i,i=1,2,3,4,5$ and their traversal $\mathcal{L}_M$. \label{fig:orthpar5}}
\end{figure}

Figure~\ref{fig:orthpar5} shows the third case where a line $\mathbf{V}(\mathcal{M}_3)$ intersects all five observed lines. 

To analyse the remaining singularities, we need to determine $\mathbf{V}(\mathcal{H}) \cap \mathbf{V}(\mathcal{K}_{55})$. {It amounts to analysing the ideal $\langle \mathcal{H}, \mathcal{K}_{55} \rangle$ (too large to be displayed here)}. Its Gr\"obner basis yields $\{1\}$ implying that the variety is null and hence there are no isolated singularities in the generic case. {However, it is possible that there are special relative configurations of the five lines for which $\mathbf{V}(\mathcal{H}) \cap \mathbf{V}(\mathcal{K}_{55}) \neq \emptyset$. As mentioned in Section~\ref{subsec:VH5}, it was not possible to find these configurations of the observed lines when they are generic, due to the computational complexity. Nonetheless, in this context, the constraints on the observed lines reduce the complexity and hence we are able to find the conditions on the parameters $d_i,\,i=1,\dots,7$ such that $\mathbf{V}(\mathcal{H}) \cap \mathbf{V}(\mathcal{K}_{55}) \neq \emptyset$. To do so, $\mathbf{V}(\mathcal{H})$ is first calculated by finding the Gr\"obner basis $gb_{H}$ of $\mathcal{H}$ w.r.t. the monomial order $Z\succ_{lex} Y \succ_{lex} X$. It contains 3 elements and is of dimension $0$. So, its variety is a point:}
\begin{align}
   &X={\frac {d_{{1}}d_{{2}}d_{{4}}+d_{{1}}d_{{4}}d_{{6}}-d_
{{2}}d_{{4}}d_{{7}}-d_{{3}}d_{{5}}d_{{6}}}{2 \left( d_{{1}}-d_{{7}}
 \right) d_{{4}}}}, \nonumber \\
 & Y=-{\frac {d_{{1}}d_{{2}}d_{{4}}-d_{{1}}d_{{4}
}d_{{6}}-d_{{2}}d_{{4}}d_{{7}}-d_{{3}}d_{{5}}d_{{6}}}{2d_{{5}}d_{{6}}}}
, \nonumber \\
 & Z=-{\frac {d_{{1}}d_{{2}}d_{{4}}-d_{{1}}d_{{4}}d_{{6}}-d_{{2}}d_
{{4}}d_{{7}}-d_{{3}}d_{{5}}d_{{6}}}{2d_{{4}}d_{{6}}}} .
\end{align}

{For this point to be a singularity, it should also belong to the variety $\mathbf{V}(\mathcal{H}) \cap \mathbf{V}(\mathcal{K}_{55})$ and hence it should absolutely lie in the variety $\mathbf{V}(\mathcal{K}_{55})$. Substituting 
the values of ${X,Y,Z}$ in $\mathcal{K}_{55}$ leaves 36 non-zero polynomials solely in terms of parameters $d_1,\dots,d_7$ (see the Maple file). They constitute the conditions for $\mathbf{V}(\mathcal{H}) \cap \mathbf{V}(\mathcal{K}_{55}) \neq \emptyset$.} 

{\bf Therefore, for the special case of P5L considered in this section, generically, there are no singularities. The singularities appear as a line and/or as a point for some special relative configurations of the five lines.}

{Again, the results in this section remain valid under any permutation of the five observed lines and hence the analysis entails singularities in the case of five lines subject to any other orthogonality and parallelism constraints.}
\begin{table}[t]
\caption[]{Different cases of singularities in P4L and P5L.}	\label{table:cases}
	\begin{tabular}{p{0.5cm}|p{2.6cm}|p{4cm}}
		\toprule
		Cases & Subcases & Singularity configurations\\
		\midrule%
		P3L & Three skew lines & 
		\begin{tabular}{p{4cm}}%
		$C$ lies on the hyperboloid of one sheet uniquely defined by the observed lines or on a cubic surface that contains the three lines
		\end{tabular} \\ 
		\hline
		\multirow{4}{*}[-1.2em]{P4L} & \begin{tabular}{p{2.5cm}}%
		 Four lines in a hyperbolic congruence%
		\end{tabular} & \begin{tabular}{p{4cm}}%
		$C$ lies on two affine lines intersecting the four observed lines and up to 10 real points
		\end{tabular} \\ \cline{2-3}
		& \begin{tabular}{p{2.5cm}}%
		 Four lines in a parabolic congruence%
		\end{tabular} & \begin{tabular}{p{4cm}}%
		$C$ lies on an affine line intersecting the four observed lines and up to 10 real points
		\end{tabular} \\ \cline{2-3}
		& \begin{tabular}{p{2.5cm}}%
		 Four lines in an elliptic congruence%
		\end{tabular} & \begin{tabular}{p{4cm}}%
		Up to 10 real points
		\end{tabular} \\ \cline{2-3}
    & \begin{tabular}{p{2.5cm}}%
		 With orthogonality and parallelism constraints%
		\end{tabular} & \begin{tabular}{p{4cm}}
		$C$ lies on the two affine lines intersecting the four observed lines
		\end{tabular}\\
		\hline
		\multirow{2}{*}[-1.2em]{P5L} & \begin{tabular}{p{2.5cm}}%
		 Five lines in a regular linear line complex%
		\end{tabular} & \begin{tabular}{p{4cm}}%
		No singularities
		\end{tabular} \\ \cline{2-3}
		& \begin{tabular}{p{2.5cm}}%
		 Five lines in a singular linear line complex%
		\end{tabular} & \begin{tabular}{p{4cm}}%
		$C$ lies on the line intersecting the five observed lines 
		\end{tabular} \\ \cline{2-3}
    & \begin{tabular}{p{2.5cm}}%
		 With orthogonality and parallelism constraints%
		\end{tabular} & \begin{tabular}{p{4cm}}
		No singularities; \\
		Special case: A line and/or a point
		\end{tabular}\\
		\bottomrule
	\end{tabular} 
\end{table}

{All cases corresponding to the singular configurations when observing three, four and five lines are summarized in Table~\ref{table:cases}.}

\section{Simulation results}
\label{sec:sim}

This section illustrates the impact that the exposed singularities have on the behavior of basic image-based visual servoing {and pose determination algorithms}.


From the results of Sections~\ref{sec:singP4L} and~\ref{sec:singP5L}, we designed a series of experiments where a free-flying camera is simulated and controlled in visual servoing from the observation of lines. {Another set of tests show the result of classical algorithms for pose estimation from lines when the camera is controlled in open-loop (the camera motion is specified beforehand).} All the simulations were performed using the visual servoing library ViSP~\citep{Marchand05b}.

\subsection{Singularities in P4L}
\label{sec::sims_P4L}

Let us consider four lines defined by a point and a direction~\eqref{eq:param}, and let us specialize the parameters arbitrarily as follows:
\begin{align} \label{eqs:params}
& s_1=4,\ s_2=-5, \ s_3=7, \ s_4=3, \ s_5=-2,\ s_6=13, \nonumber \\
& \ d_1=2, \ d_2=3, \ d_3=5, \ d_4=1, \ d_5=7.
\end{align}

For this configuration, as shown in Section~\ref{subsec:VG}, the four lines are in a hyperbolic congruence (see Fig.~\ref{fig:cong}) and there exist two lines $\mathcal{L}_M$ and $\mathcal{L}_N$ transversal to $\mathcal{L}_1, \dots, \mathcal{L}_4$, which have the following Pl\"ucker coordinates:
\begin{equation} \label{eq::sims_P4L_transversal_lines}
\begin{aligned}
    &\mathcal{L}_M = [0.0830 \ 0.4989 \ -0.8627 \ 0 \
    -0.9641 \ -0.5575];\\
    &\mathcal{L}_N = [0.2067 \ -0.03902 \ -0.9776 \ 0 \ -0.3509 \ 0.01401].\\
\end{aligned}
\end{equation}

The singularity loci of the interaction matrix consists of the two lines $\mathcal{L}_M$, $\mathcal{L}_N$ plus $6$ isolated points which belong in the variety $\mathbf{V}(\mathcal{H}, \mathcal{K})$ (see Section~\ref{sec::sims_4lines_isolated_points}).\\


\subsubsection{VS towards a desired pose near singularity line $\mathcal{L}_M$}
\label{section:simulation_line_LM}

\begin{figure}[t]
	\centering
	\includegraphics[width=90mm, height=64mm]{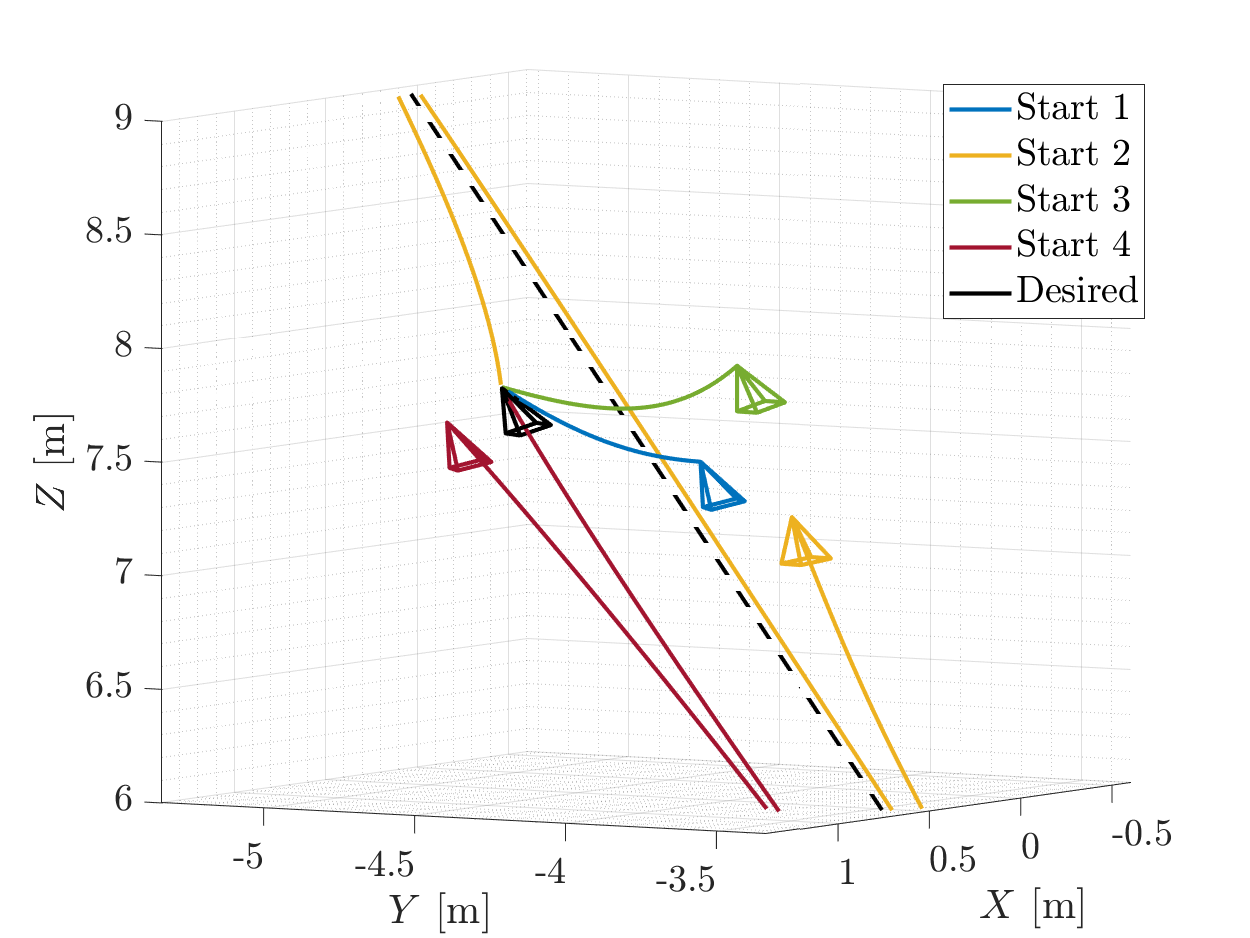}
	\caption{Visual servoing using four image lines, starting from four initial poses (coloured). The desired end pose is in black. The black dashed line is the singularity line $\mathcal{L}_{M}$ that intersects the four observed lines.\label{fig:simulations_traj}}
\end{figure}

\begin{figure}[t]
	\centering
	\includegraphics[width=90mm, height=45mm]{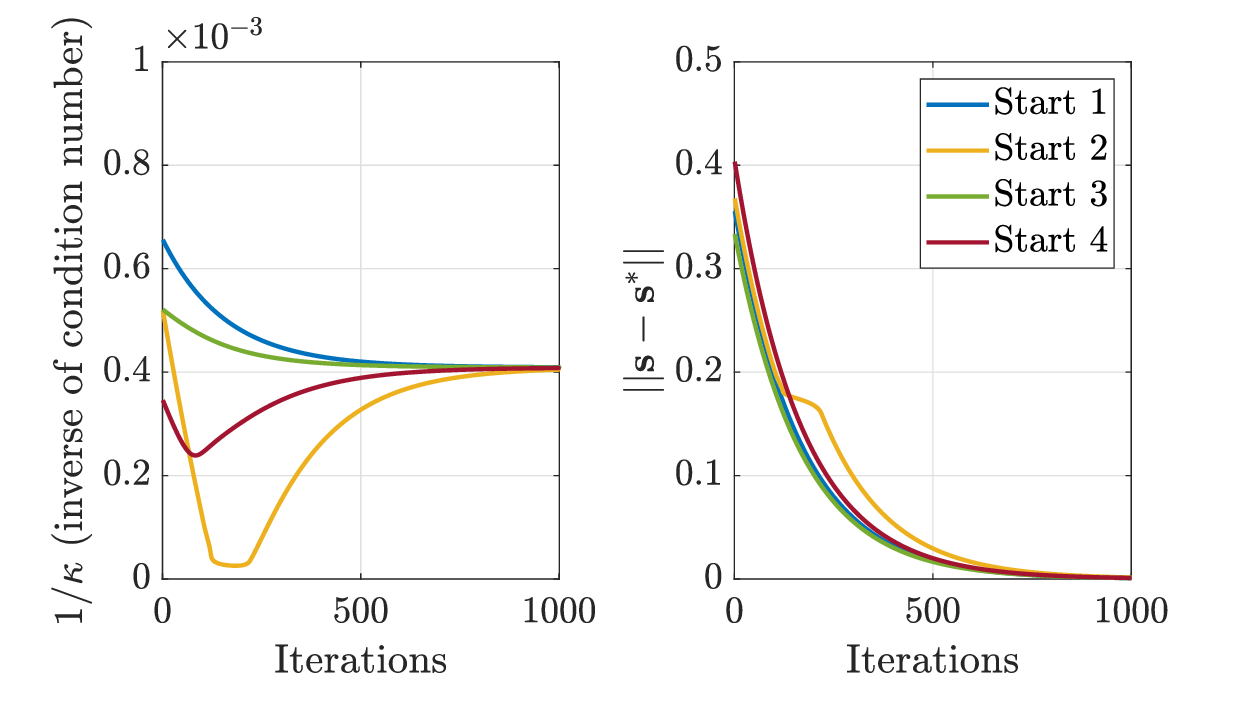}
	\caption{Inverse of the condition number $\kappa$ of the interaction matrix $\mathbf{M}_{(4)}$ (left) and norm of the error vector $||\mathbf{s}-\mathbf{s}^*||$ (right). \label{fig:errors_condnum}}
\end{figure}

\begin{figure}[t]
	\centering
	\includegraphics[width=90mm, height=45mm]{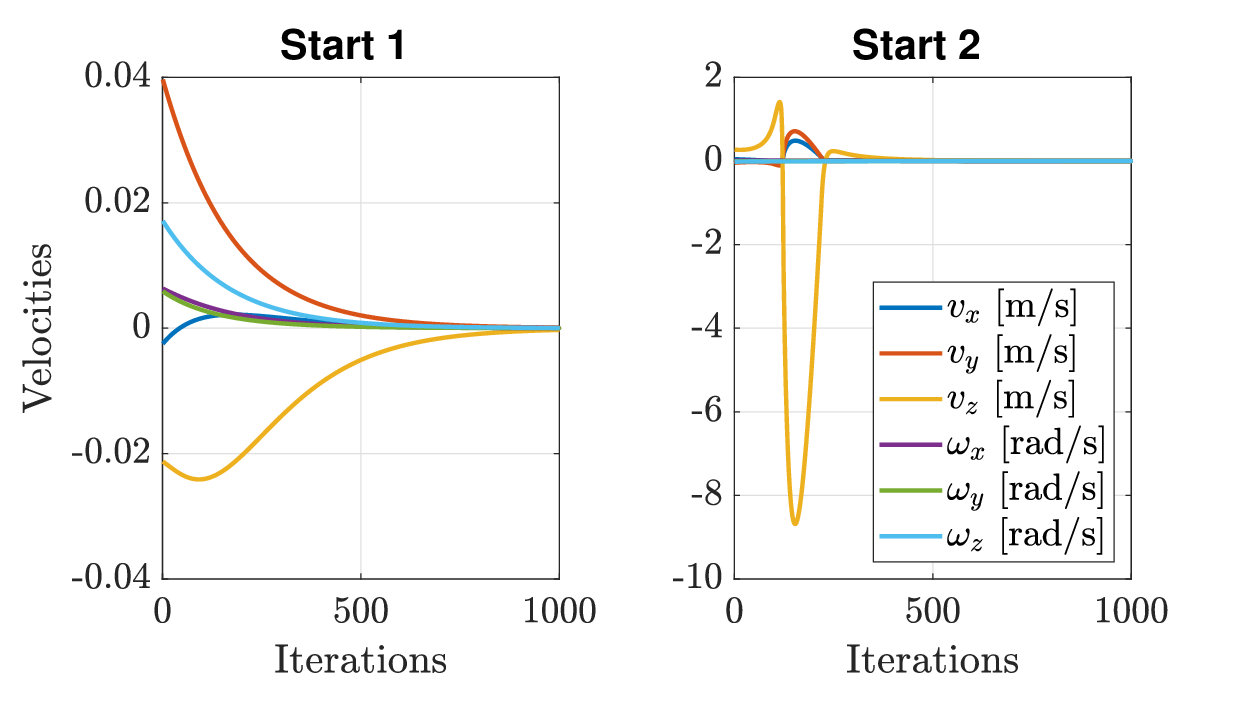}
	\caption{Velocity inputs $\boldsymbol{\tau}_c$ for the camera in a stable control scheme (left), and when crossing a singularity (right). \label{fig:vels}}
\end{figure}

\begin{figure}[t]
	\centering
	\includegraphics[width=90mm, height=45mm]{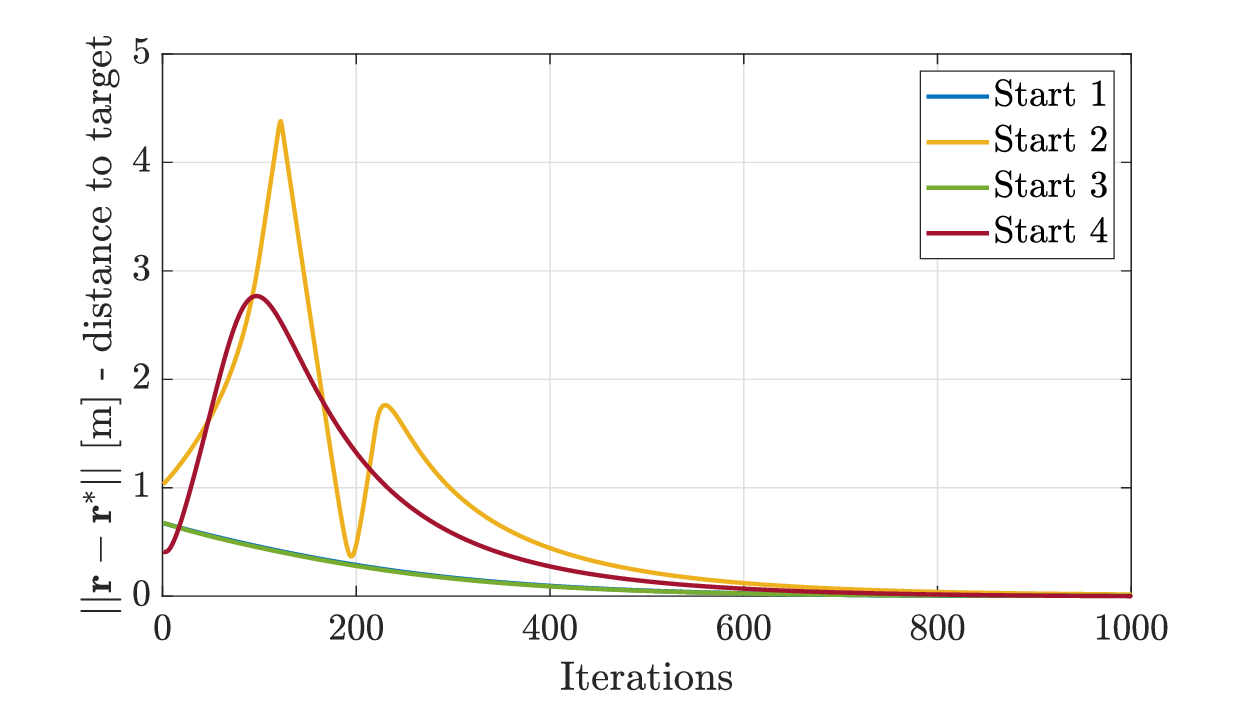}
	\caption{Distance to the target during the visual servo. \label{fig:distance}}
\end{figure}

We selected four initial camera positions in the proximity of a point $\mathcal{P}_M = [0.3962~\text{m} \ -4.337~\text{m} \ 7.50~\text{m}]$ that lies on the singularity line $\mathcal{L}_M$. From each starting point, we attempt to control the camera towards a desired position which is always the same. The coordinates of the start and target positions are defined relative to $\mathcal{P}_M$, and are shown in Table~\ref{tab:simulation_points} along with a commentary on the choice of points. The camera orientations were chosen such that the four lines are clearly visible from all locations.


At each iteration, the controller tries to minimize an error function $\mathbf{s}(t)-\mathbf{s}^*$. The vector of features $\mathbf{s} = [\mathbf{l}_1^{'T}, \dots, \mathbf{l}_4^{'T}]^T\in \mathbb{R}^{12}$ is composed of the coordinates $l'_{xi}$, $l'_{yi}$ and $l'_{zi}$ for each projected line $1 \leq i \leq 4$, while $\mathbf{s}^*$ contains the values of the features at the desired position. In order to achieve an exponential decrease of the error, the velocity input to the camera is 
\begin{equation}
\label{eq:velocity_input}
\boldsymbol{\tau}_c = -\lambda \, \mathbf{M}^+_{(4)} \left(\mathbf{s}(t)-\mathbf{s}^*\right),
\end{equation}
where $\mathbf{M}^+_{(4)}$ is the Moore-Penrose pseudoinverse of the interaction matrix~\eqref{eq:interaction_matrix_full} and $\lambda$ is a gain factor which was set to $0.1$. For these first simulations, we assume that all the parameters appearing in the matrix $\mathbf{M}_{(4)}$ are known and hence we can always obtain a perfect estimate of its pseudoinverse $\mathbf{M}^+_{(4)}$. Note that no noise was added to the visual data, and hence all the instabilities are due uniquely to the determinant of the interaction matrix vanishing at a singularity. 

Figure~\ref{fig:simulations_traj} displays the camera trajectories starting from each initial position. A normal behaviour is achieved from Starts $1$ and $3$: the camera describes an almost straight line towards the desired pose, and the magnitude of the error vector decreases exponentially (see Fig.~\ref{fig:errors_condnum}). 

Start $2$ is located opposite from the desired position relative to the point $P_{M}$. The camera reaches the target point eventually, but it diverges along the singularity line as it approaches it (see Fig.~\ref{fig:simulations_traj}). The inverse of the condition number $\kappa$ of $\mathbf{M}_{(4)}$, shown in Fig.~\ref{fig:errors_condnum}, reaches almost zero in the vicinity of the singularity, which means that the system in~\eqref{eq:velocity_input} is ill-conditioned and, as a result, the controller produces very high velocity commands causing instability. Figure~\ref{fig:vels} compares the velocity input profiles for Starts $1$ and $2$ throughout the simulation. The velocity inputs from Start $2$ are two orders of magnitude higher than those produced in a stable situation. Note that although the distance to the desired point increases during the undesirable motion (Fig.~\ref{fig:distance}), the magnitude of the error $||\mathbf{s} - \mathbf{s}^*||$ remains approximately constant (Fig.~\ref{fig:errors_condnum}). 

Start $4$ is located slightly further away from the singularity line, and closer to the desired position. The trajectory converges but the camera is again subjected to considerably high velocities and does not approach the target monotonically (see Figs.~\ref{fig:simulations_traj} and~\ref{fig:distance}). This illustrates that the impact of the singularities is not limited to the trajectories that directly cross a singular point, but instead, that there is an area of influence in the vicinity of a singularity where the behaviour can be affected by the high condition number of the interaction matrix.
{\tiny
	\begin{table}[t]
		\centering
		\caption{Initial and desired positions relative to a point on the singularity line $\mathcal{L}_{M}$ (all units are in meters). \label{tab:simulation_points}}
		\begin{tabular}{lrrrl}
			& $\Delta X$ & $\Delta Y$ & $\Delta Z$ & Note                           \\ \hline \hline
			Desired & $0.30$    & $-0.30$     & $0.30$    & Target end position $\mathbf{s}^*$ \\ \hline
			Start 1 & $0.20$     & $0.30$    & $0$    & Near to singularity.           \\ \hline
			Start 2 & $-0.30$     & $0.30$    & $-0.30$     & Opposed to desired.            \\ \hline
			Start 3 & $0$    & $0.30$     & $0.40$     & Near to singularity.         \\ \hline
			Start 4 & $0.10$    & $-0.60$    & $0.10$   & Away from singularity.        
		\end{tabular}
	\end{table}
}

\subsubsection{Trajectory following across singularity line $\mathcal{L}_N$}
\begin{figure}[t]
	\centering
	\includegraphics[width=90mm, height=64mm]{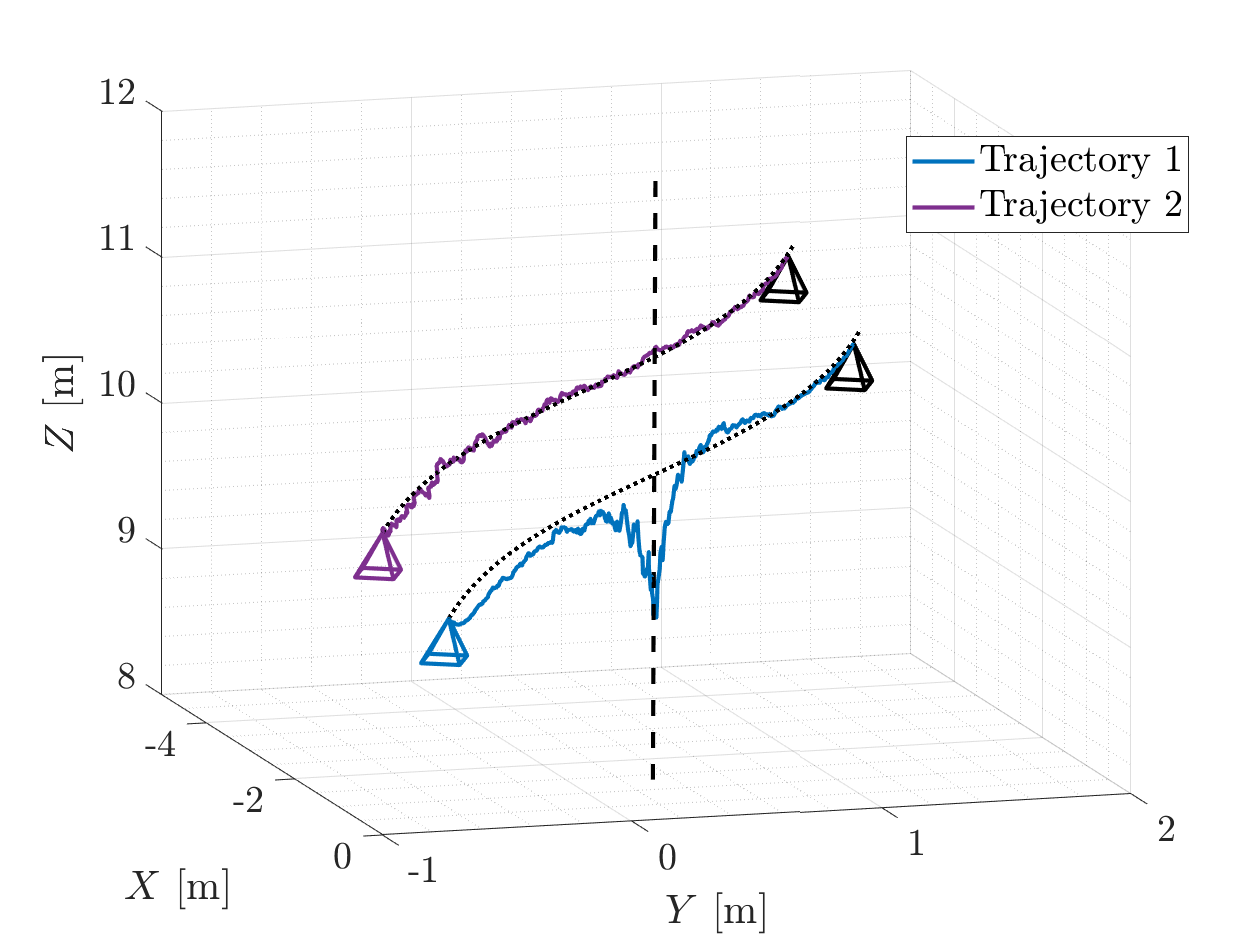}
	\caption{VS from four image lines: Trajectory described by the camera when controlling it along a prescribed path (thin dotted line). The control becomes unstable along the trajectory that crosses the singularity line $\mathcal{L}_N$ (black dashed line). \label{fig:traj_line2}}
\end{figure}

\begin{figure}[t]
	\centering
	\includegraphics[width=90mm, height=90mm]{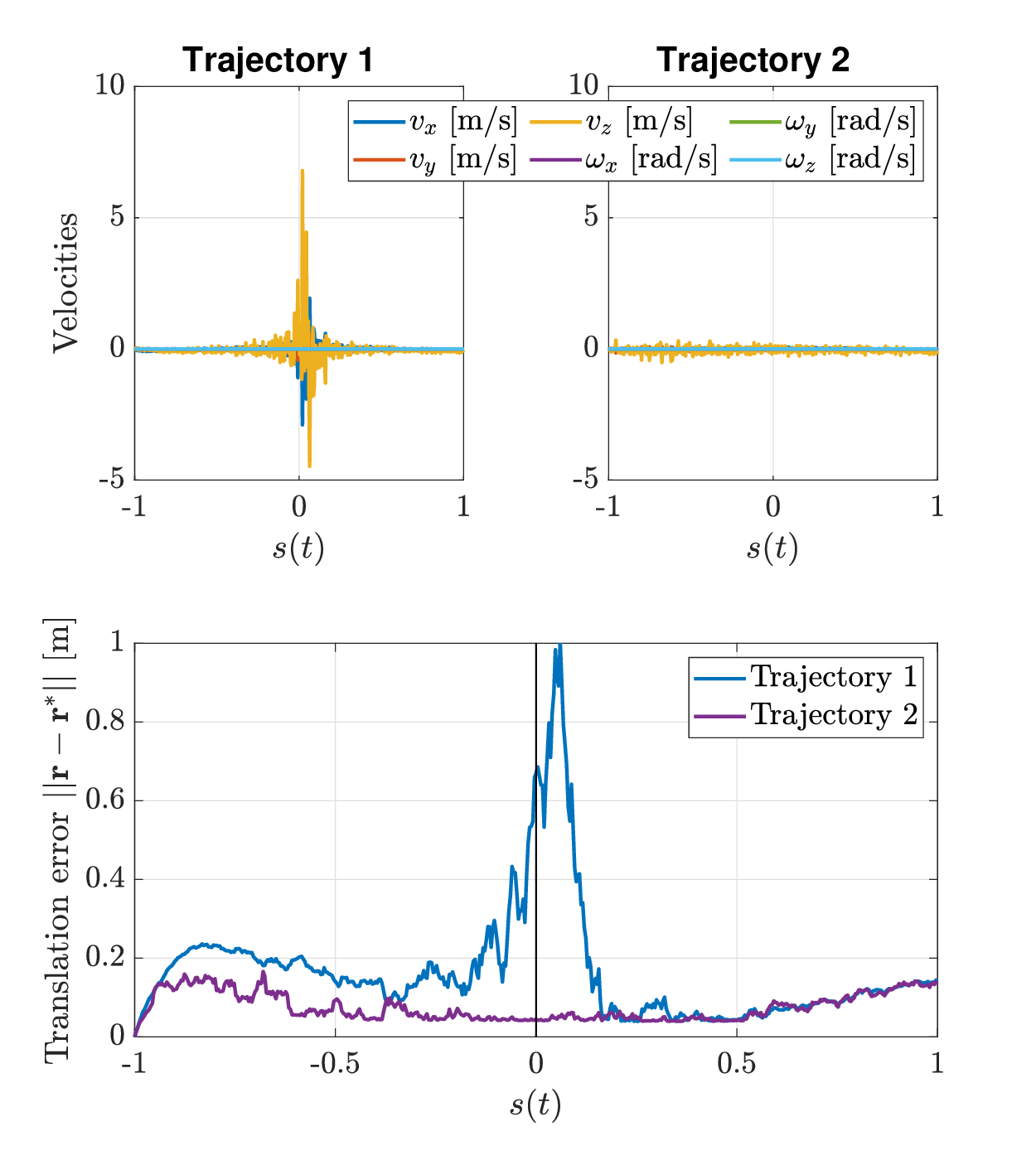}
	\caption{Camera velocity inputs $\boldsymbol{\tau}_c$ (top) and position error $||\mathbf{r}(t) - \mathbf{r}^*(t)||$ (bottom) along the trajectories. The vertical step in the bottom figure indicates where Trajectory 1 crosses the singularity line $\mathcal{L}_N$. \label{fig:vels_line2}}
\end{figure}

With the same four lines defined by the parameters~\eqref{eqs:params}, a simple trajectory along a cubic curve defined by $\mathbf{r}^*_1(t) = [X^*_1(t), Y^*_1(t), Z^*_1(t)]^T$ with 
\begin{equation}
\label{eq:traj1_sing}
\begin{aligned}
&X^*_1(t) = -s(t)^3 - 1.7551~\text{m},
&Y^*_1(t) = s(t) + 0.3992~\text{m}\\
&Z^*_1(t) = 0.7 s(t) + 10.0~\text{m}.
\end{aligned}
\end{equation}
for $s(t) = 0.04 t -1$ and $0 \leq t \leq 50$~s, was designed to cross the singularity line $\mathcal{L}_N$ at $s(t)=0$. A second, very similar trajectory $\mathbf{r}^*_2(t) = [X^*_2(t), Y^*_2(t), Z^*_2(t)]^T$ with 
\begin{equation}
\label{eq:traj2_nosing}
\begin{aligned}
&X^*_2(t) = -s(t)^3 - 3.2551~\text{m},
&Y^*_2(t) = s(t) + 0.3992~\text{m}\\
&Z^*_2(t) = 0.7 s(t) + 10.30~\text{m}.
\end{aligned}
\end{equation}
should not cross any singular points. 

Trajectory following can be performed using visual servoing by introducing a time-dependent vector $\mathbf{s}^*(t)$ of desired visual features in the control law~\eqref{eq:velocity_input}, which can be computed by forward projection~\eqref{eq:plucker2dline} of the 3D line coordinates in the camera frame $\mathcal{F}_c$ as the camera moves along the trajectory. 

As before, we assume that the parameters in the interaction matrix $\mathbf{M}_{(4)}$ can be measured such that we can obtain an estimate of its pseudoinverse $\mathbf{M}^+_{(4)}$ to use in the control law~\eqref{eq:velocity_input}. However this time we added white Gaussian noise of standard deviation $\sigma = 2 \cdot 10^{-3}$ to the 3D coordinates of the observed lines, in order to simulate the impact of errors in the measurements. The gain factor $\lambda$ was set to $1$. 

The camera behaviour is shown in Fig.~\ref{fig:traj_line2}. For trajectory $\mathbf{r}^*_2$~\eqref{eq:traj2_nosing}, away from the singularity line, the camera follows the prescribed path with relative accuracy. The velocity inputs are mild and there is a small, approximately constant tracking error of about $0.1$~m (Fig.~\ref{fig:vels_line2}), which could be reduced with a more sophisticated controller, for example by introducing an integral term in the control law~\eqref{eq:velocity_input}. 

For the trajectory $\mathbf{r}^*_1$~\eqref{eq:traj1_sing}, the camera is unable to follow the desired path accurately in the vicinity of the singularity line. Around $s(t)=0$, the velocity commands become very high in magnitude, inducing instability, and the translation error becomes as high as $1$~m (Fig.~\ref{fig:vels_line2}).

\subsubsection{VS around the isolated point singularities}
\label{sec::sims_4lines_isolated_points}

For the configuration~\eqref{eqs:params}, there also exist $6$ isolated points $\mathcal{P}_{i}$ which are zero-dimensional singularities of the interaction matrix. They are those solutions of the ideal $\langle \mathcal{H}, \mathcal{K} \rangle$ which are not contained in the variety $\mathbf{V}(\mathcal{G}) = \mathcal{L}_M \cup \mathcal{L}_N$ (see Appendix A) and have the following coordinates (in meters):
\begin{equation}
\label{eq:4linessims_isolated_points}
\begin{aligned}
&\text{Ex. }1: \quad \mathcal{P}_{1} = [-9.858 \  -2.473 \  -1.841],\\
&\text{Ex. }2: \quad \mathcal{P}_{2} = [0.0541  \  0.0092 \   1.8422],\\
&\text{Ex. }3: \quad \mathcal{P}_{3} = [-0.3203  \   0.0105  \  0.2205],\\
&\text{Ex. }4: \quad \mathcal{P}_{4} = [1.0113  \  0.7947 \  -0.8850],\\
&\text{Ex. }5: \quad \mathcal{P}_{5} = [0.9387  \  0.5681 \  -2.0225],\\
&\text{Ex. }6: \quad \mathcal{P}_{6} = [65.09 \ -96.57 \ -0.03639].
\end{aligned}
\end{equation}
Another point $\mathcal{P}_0$ is chosen arbitrarily and away from any singularities:
\begin{equation}
    \text{Ex. }0: \quad \mathcal{P}_{0} = [-8.858 \  -2.473 \  -1.841].
\end{equation}

For each of these locations, we simulated a trajectory with the shape of a \textit{quadrifolium} centered at $\mathcal{P}_i$, given by the following equations:
\begin{equation}
\label{eq:quadrifolium}
\begin{aligned}
&X^*(t) = \mathcal{P}_{ix}, \quad \quad
Y^*(t) = \mathcal{P}_{iy} + 0.3\cos(s)\cos(2s),\\
&Z^*(t) = \mathcal{P}_{iz} + 0.3\sin(s)\cos(2s).
\end{aligned}
\end{equation}
{where $\mathcal{P}_{ix}$, $\mathcal{P}_{iy}$, $\mathcal{P}_{iz}$ are the coordinates of each point}, and $s = 2\pi \,t/20$ with $0 \leq t \leq 20$~s, and we applied the control law~\eqref{eq:velocity_input} with $\lambda=5$. Once again, Gaussian noise of standard deviation $\sigma = 10^{-3}$ was added to the Pl\"ucker coordinates of the lines to simulate measurement errors.

\begin{figure}[t]
	\centering
	\includegraphics[width=90mm, height=64mm]{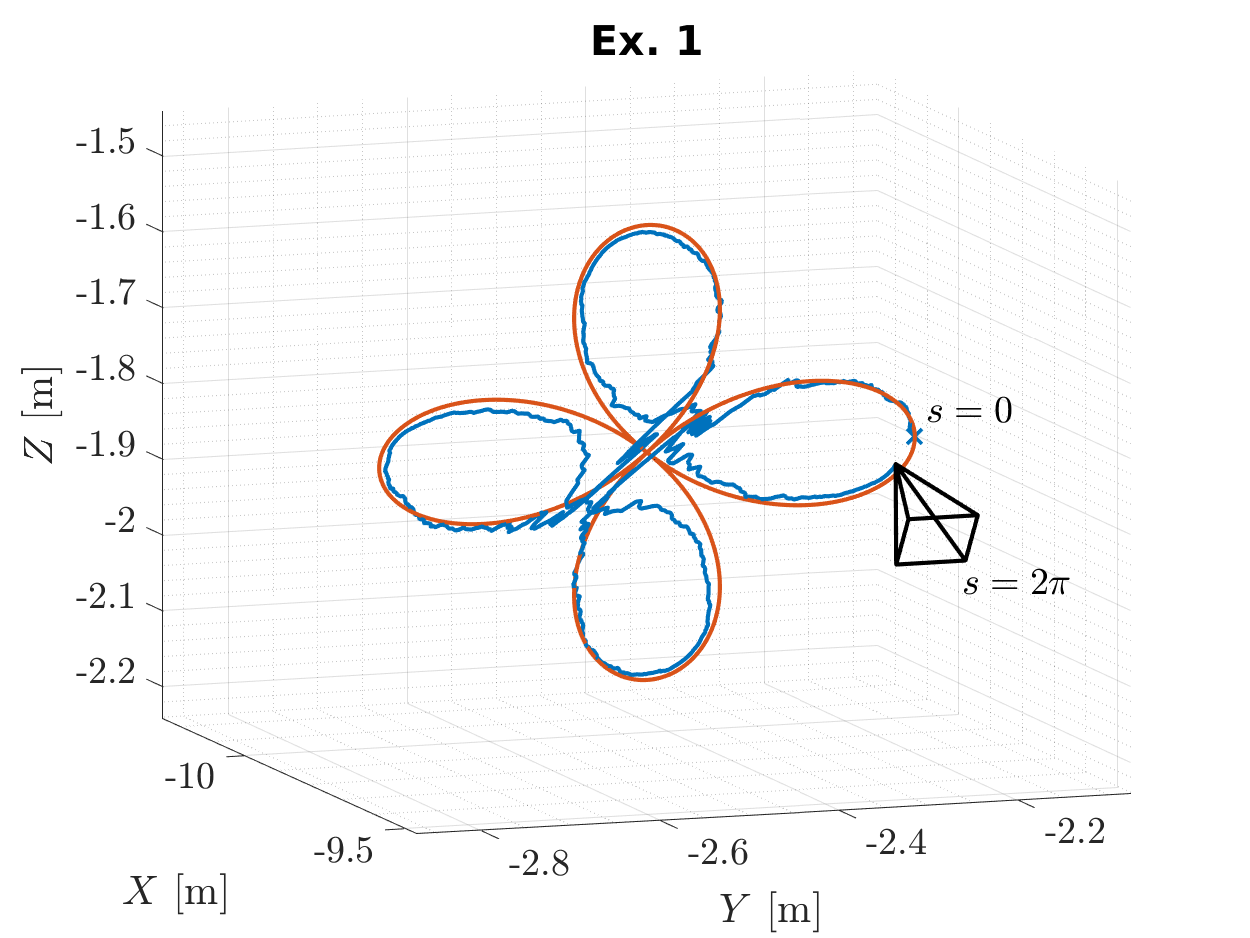}
	\caption{Visual servoing along a trajectory with the shape of a quadrifolium (red) centered at the singularity point $\mathcal{P}_1$. The true end camera position is drawn in black. A large translation error occurs 
	every time the camera approaches the singularity. \label{fig:4lines_sim3_clover}}
\end{figure}

\begin{figure}[t]
	\centering
	\includegraphics[width=90mm, height=90mm]{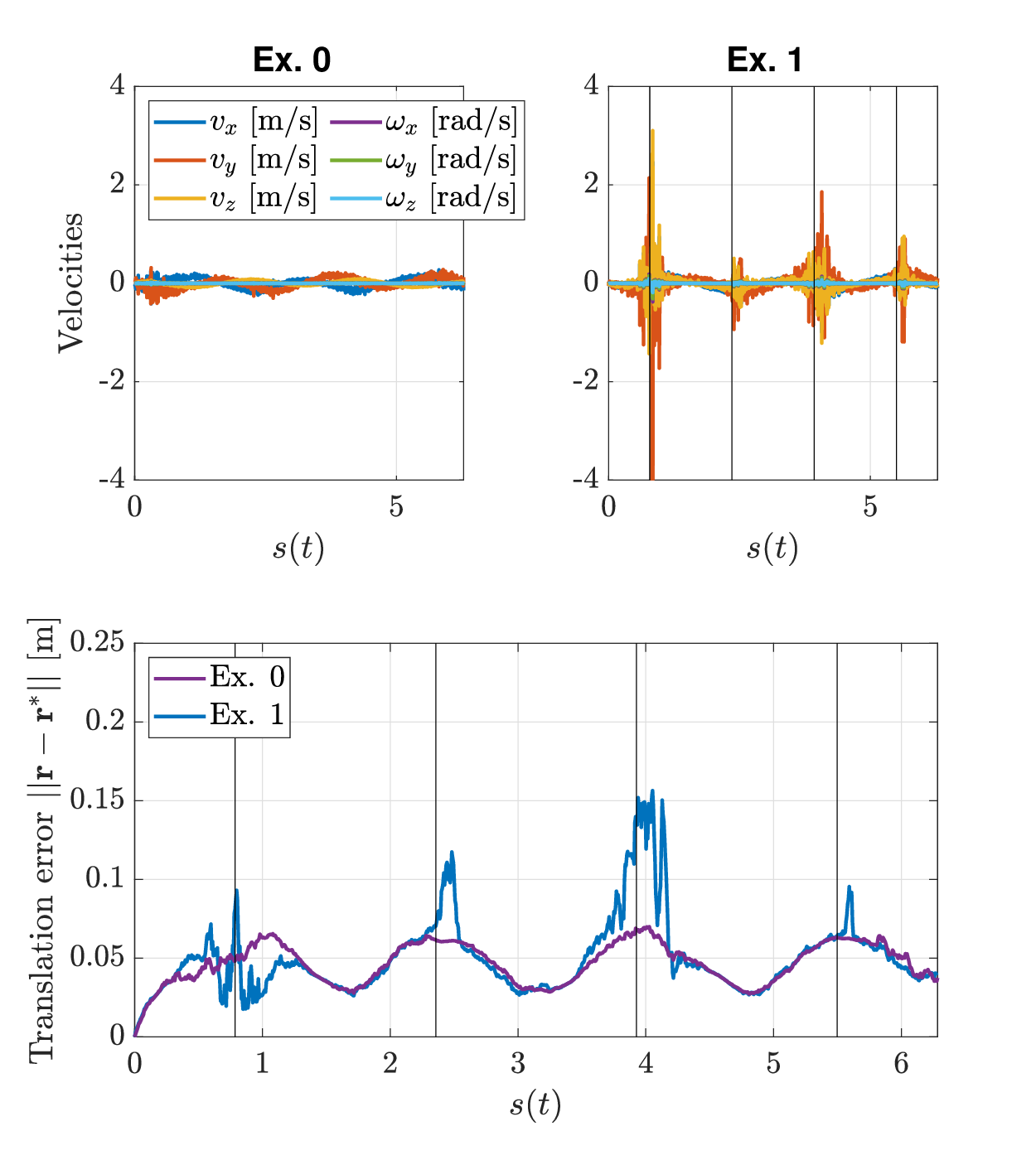}
	\caption{Velocity inputs $\boldsymbol{\tau}_c$ (top) and translation error $||\mathbf{r}-\mathbf{r}^*||$ during experiments $0$ and $1$. The vertical steps indicate where the trajectory in Ex. $1$ traverses the singularity point $\mathcal{P}_1$. \label{fig:4lines_sim3_velocities}}
\end{figure}

In Fig.~\ref{fig:4lines_sim3_clover} we show the results for the trajectory centered at the first of these points. The translation error along this trajectory is compared in Fig.~\ref{fig:4lines_sim3_velocities} with that of Ex. $0$, which does not come near a singularity. An oscillating tracking error is present in both cases due to the delay of the camera position relative to the desired point at a given time. However the presence of the singularity in Ex. $1$ results in destabilizing velocity commands as the camera approaches the center of the quadrifolium (for $s = \frac{\pi}{4} + n \frac{\pi}{2}$), as shown in Fig.~\ref{fig:4lines_sim3_velocities}. As a consequence, a significantly larger deviation occurs around this point.
 

The maximum translation errors for all the experiments~\eqref{eq:4linessims_isolated_points}, displayed in Fig.~\ref{fig:4lines_sim3_max_error}, occur always when the camera approaches the singularity point. Meanwhile the difference between the maximum and median errors indicate that the greatest part of the trajectory is completed with relative accuracy.

\begin{figure}[t]
	\centering
	\includegraphics[width=90mm, height=38mm]{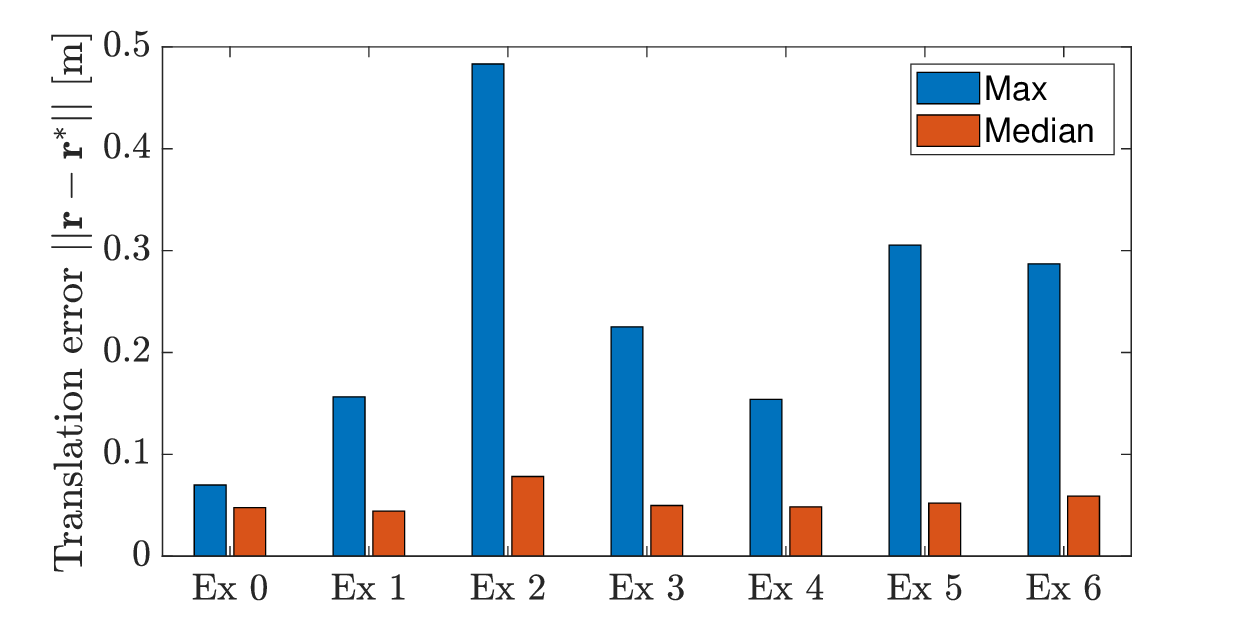}
	\caption{Maximum and median error along the quadrifolium trajectory for all experiments. \label{fig:4lines_sim3_max_error}}
\end{figure}

A particularly large error occurs in Ex. $2$. In this case, the point of singularity $\mathcal{P}_2$ is located very near both the singularity line $\mathcal{L}_M$ (see Section~\ref{section:simulation_line_LM}) and one of the observed lines $\mathcal{L}_2$ (we recall that the camera is at a singularity when it lies on $\mathcal{L}_i$ because it loses visibility of the line). When the camera approaches point $\mathcal{P}_2$, it diverges and is pulled towards the unstable regions around $\mathcal{L}_2$ and $\mathcal{L}_M$, resulting in a very large translation error.

\subsubsection{Pose Estimation along a quadrifolium trajectory}

This section illustrates the impact that the singularities have when performing pose estimation in their neighbourhood. Typically, pose computation algorithms can be classified in \textit{iterative} and \textit{non-iterative methods}. Iterative methods are usually more efficient and accurate than the non-iterative ones but, in contrast to them, they require the estimated pose to be initialized and their convergence is very sensitive to a bad initialization. 

For the following results we used our own implementation of the Robust PnL (RP$n$L) algorithm~(\cite{zhang2012robust}), regarded as one of the state-of-the-art non-iterative solvers for pose estimation from $n$ lines, which combines several classical methods for the solution of P3L~(\cite{wang2019camera}). The RP$n$L algorithm solves the P3L problem for $(n-2)$ triplets of lines and then selects the solution that yields the smallest reprojection error.

The accuracy of the RP$n$L method improves when redundant sets of lines ($n \geq 6$) are used~(\cite{wang2019camera}). Nevertheless, we consider here $n=4$, the minimal number of lines for which the pose estimation problem has a unique solution, in order to test the behaviour of pose estimation in the vicinity of the exposed singularities. In addition, the result of RP$n$L can be refined using a first-order iterative solver. Here we use Virtual Visual Servoing (VVS)~(\cite{Marchand_CGF_2002}), which minimizes the reprojection error of the lines by performing visual servoing on a virtual camera such that the desired image matches the image recorded by the real camera.

We consider four lines in the same configuration used in Section~\ref{sec::sims_P4L}, defined by a point and a direction~\eqref{eq:param} and with the parameters fixed as in~\eqref{eqs:params}. Three experiments were performed, based on the three points whose coordinates are shown in Table~\ref{tab:PE_points}. Centered at each of these, we simulated an open-loop trajectory, {defined thereafter}, along which the pose computation methods were assessed. For Ex. $1$ a generic point far from any singularities was chosen as a benchmark for the efficiency of the pose estimation algorithms. Ex. $2$ corresponds to the point singularity $P_1$ in~\eqref{eq:4linessims_isolated_points} - we demonstrate here only the behaviour of pose estimation near one of the points in~\eqref{eq:4linessims_isolated_points} because in practice the results are very similar around the other five isolated singularities. Finally, the point for Ex. $3$ lies on the singularity line $\mathcal{L}_M$. 



{\tiny
	\begin{table}[t]
		\centering
		\caption{Point coordinates used for simulations of pose determination (all coordinates given in meters). \label{tab:PE_points}}
		\begin{tabular}{lrl}
Ex. &  Coordinates & Note \\ \hline \hline
$1$ & $[5.0 \ 2.0 \ 3.0]$ & Away from singularities. \\ \hline
$2$ & $[ -9.858 \ -2.473 \ -1.841 ]$ & Isolated point singularity. \\ \hline
$3$ & $[ 0.7809 \ -2.024 \ 3.50 ]$ & On singularity line $\mathcal{L}_M$. 
		\end{tabular}
	\end{table}
}


{In all three cases the prescribed trajectory has the shape} of a \textit{quadrifolium} or four-leaved clover, defined by~\eqref{eq:quadrifolium} - for Ex. $2$, the pattern is rotated by $90$ degrees around the $Y$ axis, such that the ``leaves'' of the quadrifolium do not lie too close to the line $\mathcal{L}_M$. A constant camera orientation was chosen such that there is good visibility of the lines at all times (with the focal axis roughly pointing towards the origin).

Since the pose computation methods should be very sensitive to numerical noise in the proximity of a singularity, we added Gaussian noise with standard deviation $\sigma = 10^{-4}$ to the Pl\"ucker coordinates of the 3D lines.

Two parameters are measured from the simulations: the \textit{translation error} $t_e$, defined as the Euclidean distance between the true and estimated camera positions, and the \textit{rotation error}: the absolute value of the error angle
\begin{equation}
\label{eq:PE_error_angle}
\theta_e = \Bigg| \arccos \Bigg( \frac{1}{2} \bigg[ \text{tr} \bigg( {}^c\mathbf{R}_o \, \widehat{{}^c\mathbf{R}_o)}^T \bigg) - 1 \bigg] \Bigg) \Bigg|,
\end{equation}
where ${}^c\mathbf{R}_o$ and $\widehat{{}^c\mathbf{R}_o}^T$ are respectively the rotation matrices representing the true and the estimated orientation of the camera frame.

\begin{figure*}[t]
	\centering
\begin{subfigure}[b]{0.33\textwidth}
		\centering
		\includegraphics[width=55mm]{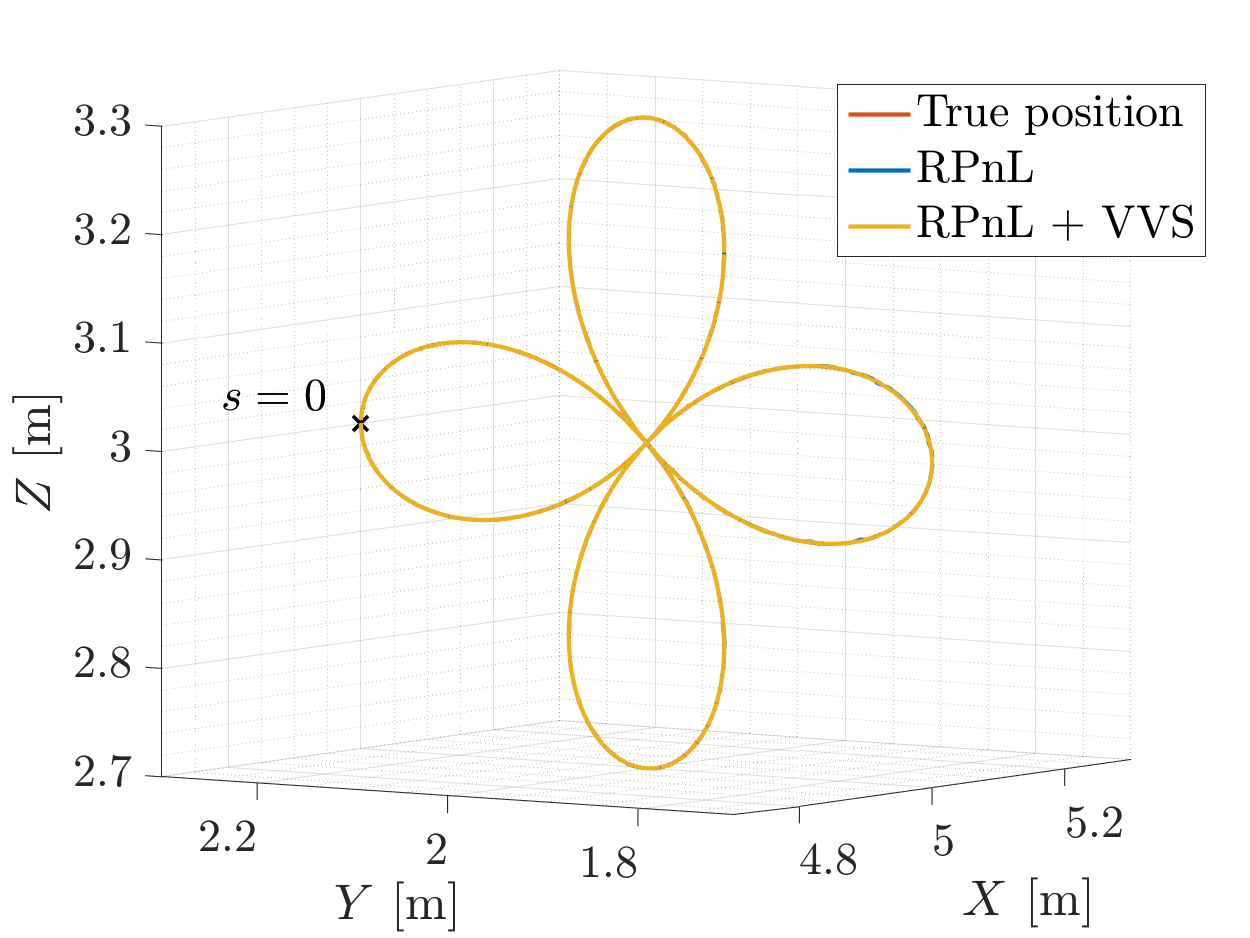}%
		\label{fig:sim_PE_P0_traj}%
\end{subfigure}
\hfill
\begin{subfigure}[b]{0.33\textwidth}
	\centering
		\includegraphics[width=55mm]{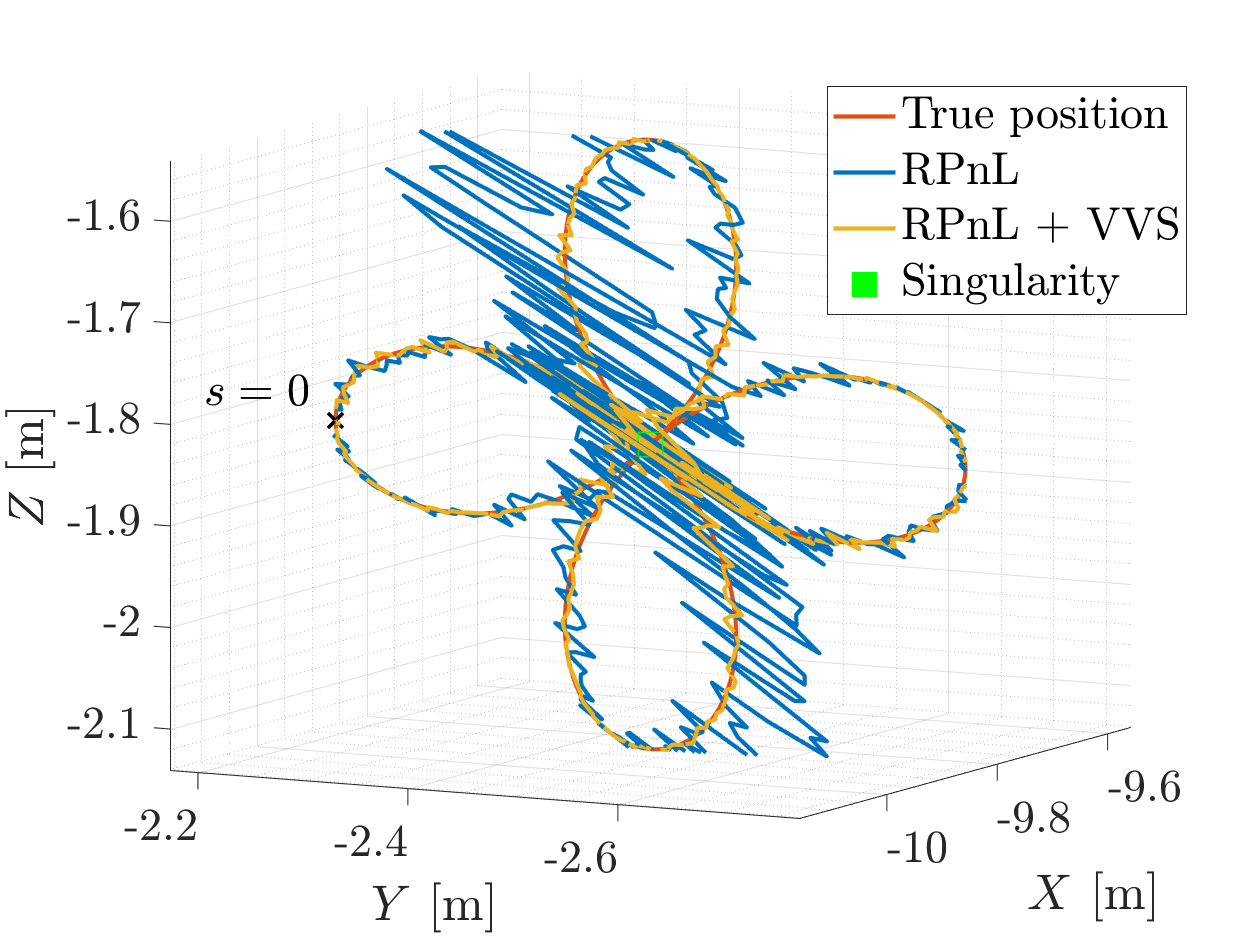}%
		\label{fig:sim_PE_P1_traj}%
\end{subfigure}
\hfill
\begin{subfigure}[b]{0.33\textwidth}
	\centering	\includegraphics[width=55mm]{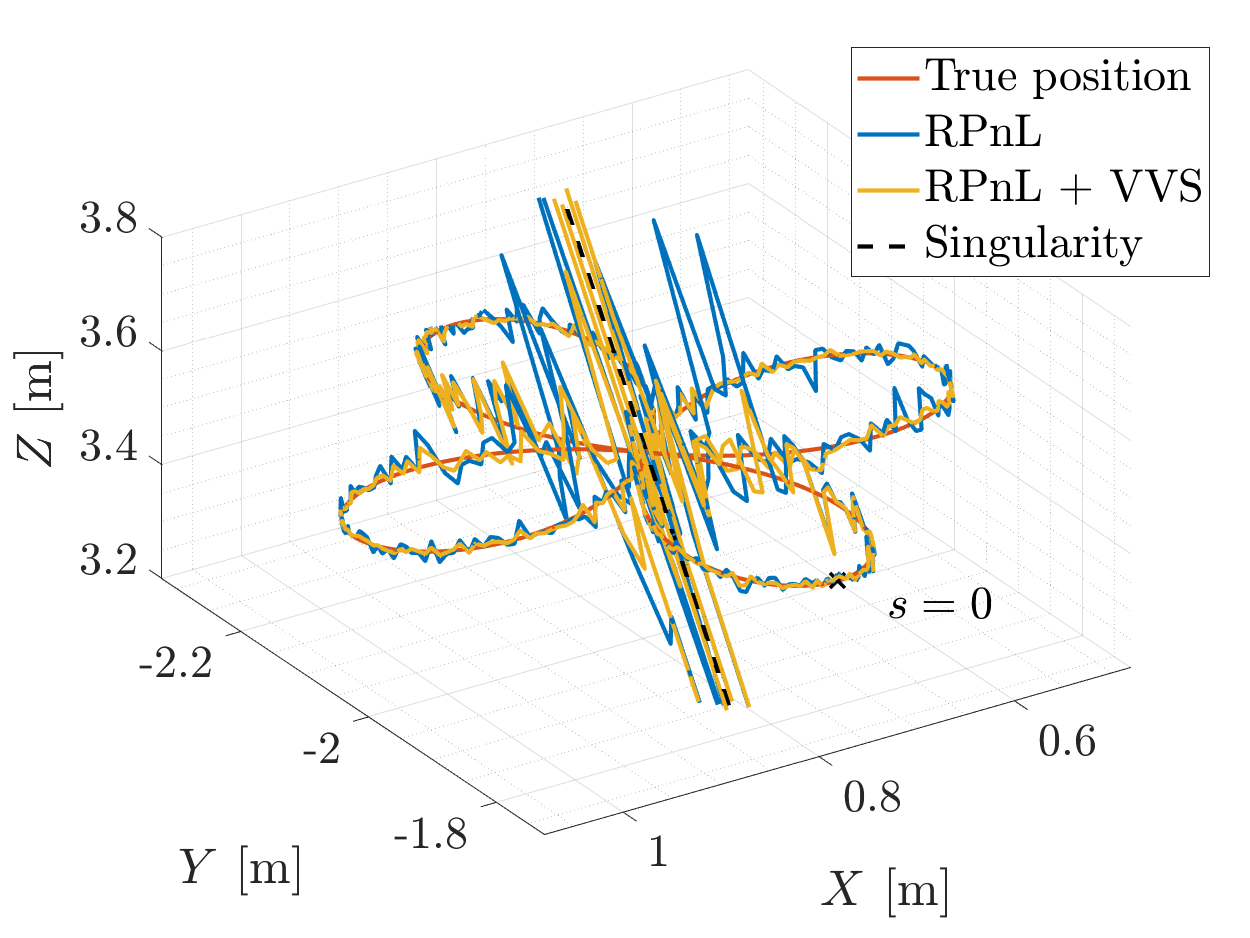}
\label{fig:sim_PE_PM1_traj}%
\end{subfigure}%

	\begin{subfigure}[b]{0.33\textwidth}
		\centering
		\includegraphics[width=55mm]{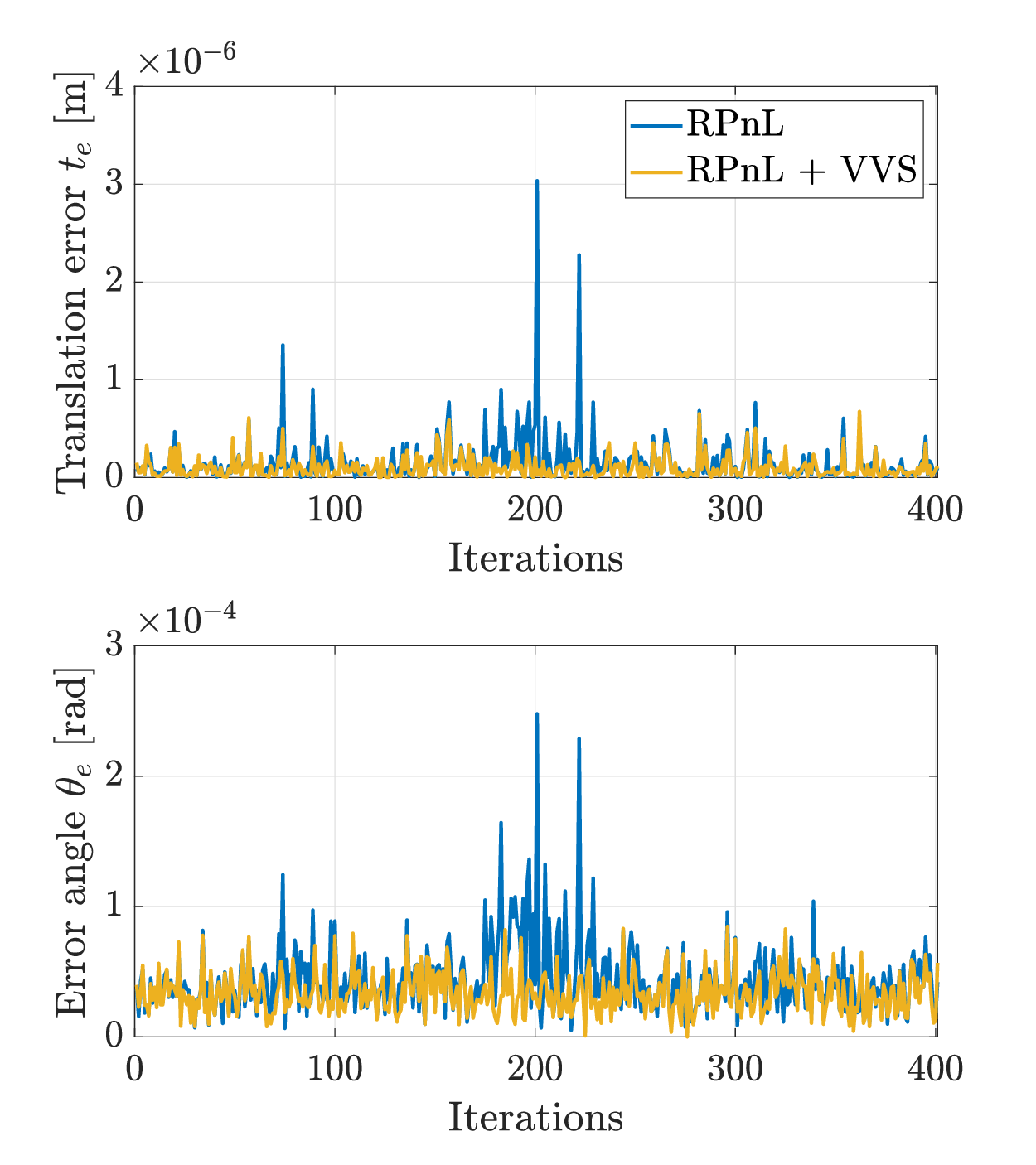}%
		\caption{Ex. $1$: Away from singularities.}
		\label{fig:sim_PE_P0_err}%
		\end{subfigure}
	\hfill
	\begin{subfigure}[b]{0.33\textwidth}
		\centering
		\includegraphics[width=55mm]{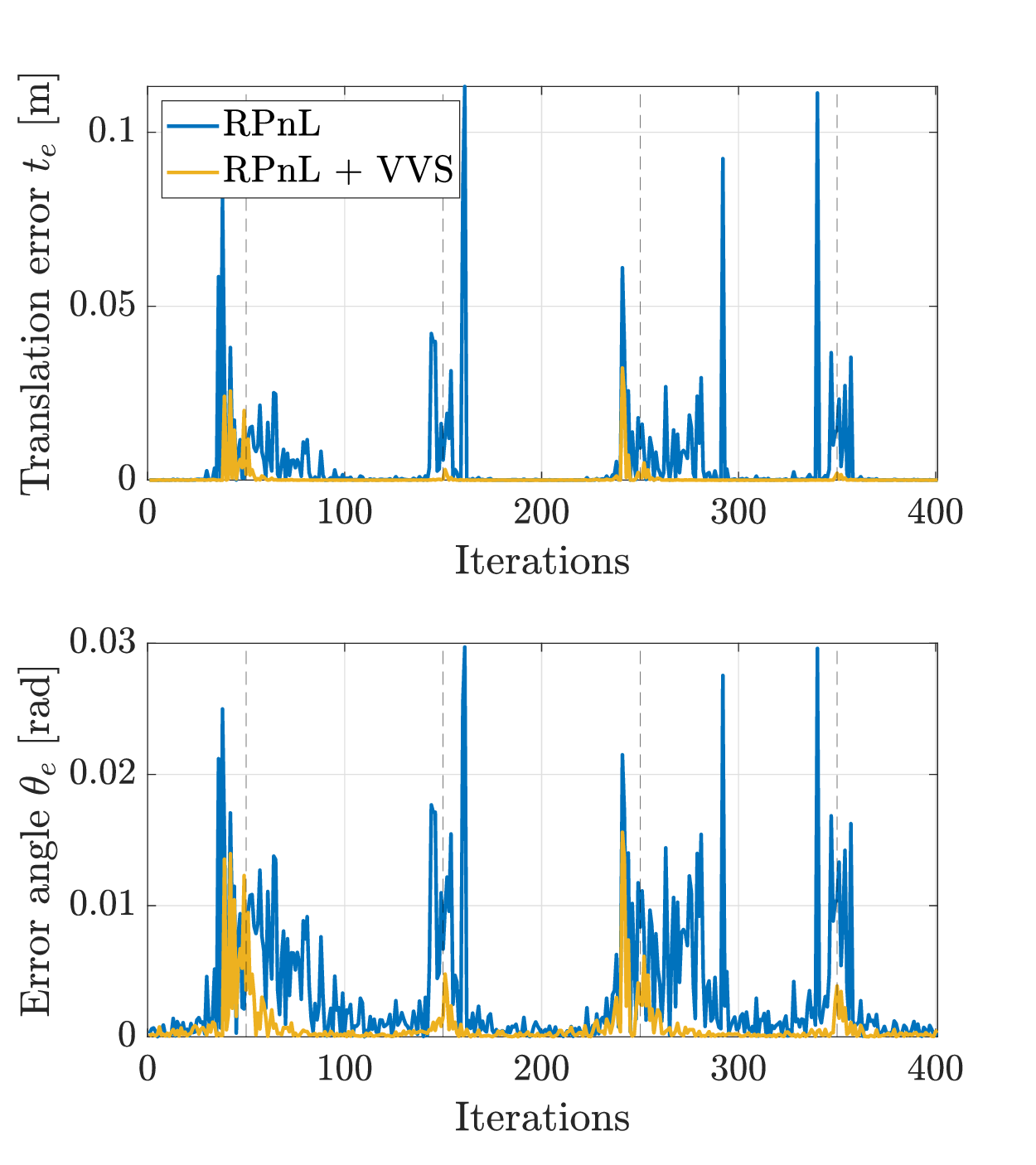}%
		\caption{Ex. $2$: Centered at an isolated singularity.}
		\label{fig:sim_PE_P1_err}%
	\end{subfigure}
	\hfill
	\begin{subfigure}[b]{0.33\textwidth}
		\centering	\includegraphics[width=55mm]{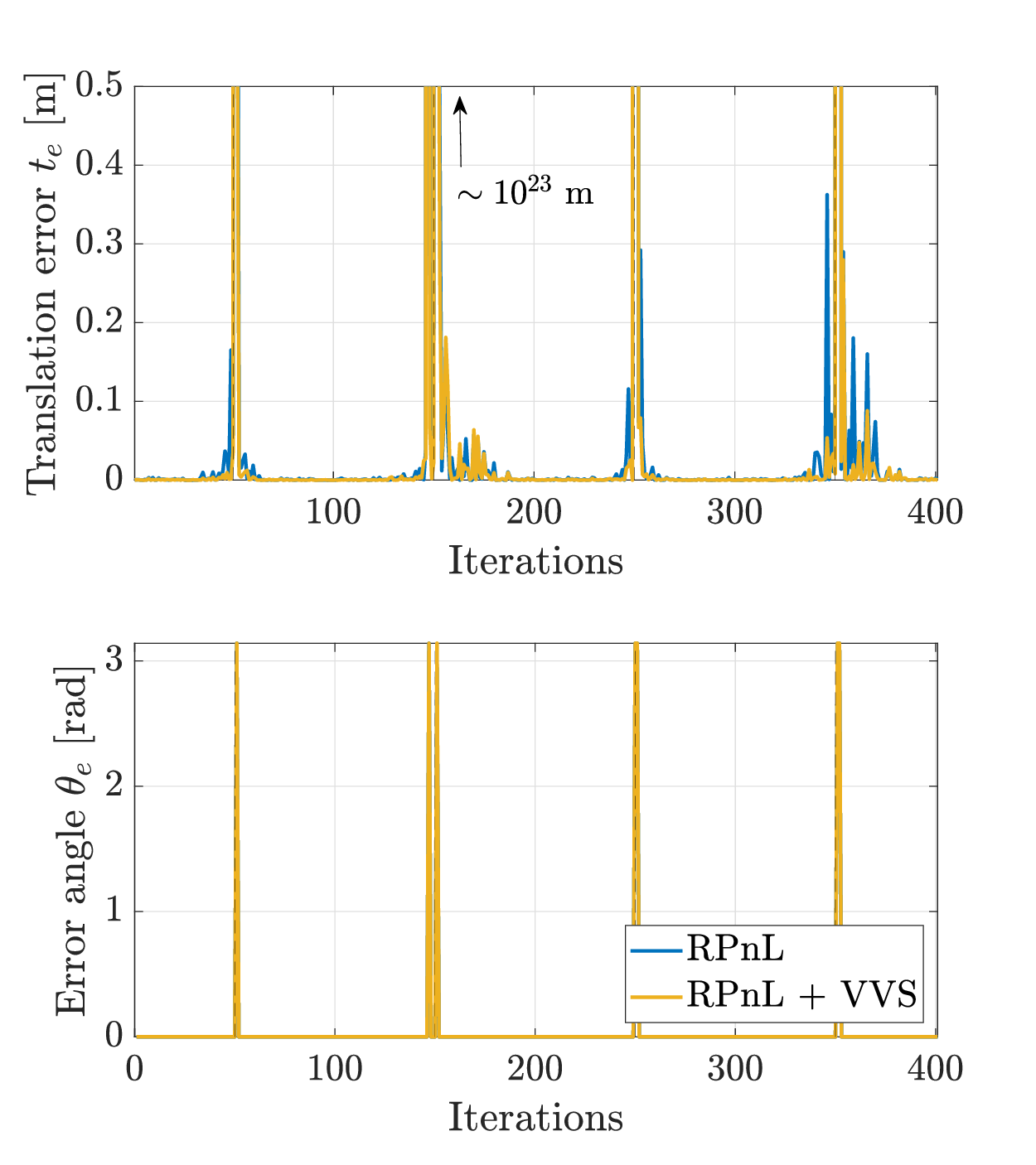}
\caption{Ex. $3$: Around the singularity line $\mathcal{L}_M$.}
		\label{fig:sim_PE_PM1_err}%
			\end{subfigure}%
	\caption{Pose estimation from four image lines along a trajectory with the shape of a quadrifolium centered at different points: a generic point away from singularities (left), an isolated point singularity (center) and a point on a line singularity (right). The top images show the true camera position (red), the estimation from the non-iterative RP$n$L algorithm (blue), and its refinement by VVS (yellow). In the bottom are displayed the translation error and the absolute error angle~\eqref{eq:PE_error_angle}. The vertical steps indicate the points where the camera passes through the singularity. Far away from any singularities both methods have a near-zero error; only the yellow plot is visible in the left image because the three trajectories overlap. In a large area around a singularity, RP$n$L becomes very sensitive to noise in the data, while VVS is quite effective in refining the result from RP$n$L except when very near the singularity. In the near proximity of the line singularity, both methods output an abhorrent estimation, with the errors tending to infinity. \label{fig:pose_estimation}}
\end{figure*}

The results from the RP$n$L algorithm and its refinement by VVS as the camera moves along each of the trajectories are depicted in Fig.~\ref{fig:pose_estimation}, with the corresponding error metrics displayed below. Away from the singularities, both methods yield near perfect estimations for both position and orientation (see Fig.~\ref{fig:sim_PE_P0_err}). 

In the second experiment (Fig.~\ref{fig:sim_PE_P1_err}), centered at an isolated singularity, the RP$n$L method becomes very inaccurate. A large translation error reaching up to $10$~cm occurs particularly as the camera crosses the singularity at the center of the pattern, but also in the vertical direction of the clover.
This is explained by the fact that RP$n$L solves the P3L problem for two triplets of lines (in this case $\mathcal{L}_1$, $\mathcal{L}_2$ and $\mathcal{L}_3$ on one hand, and $\mathcal{L}_1$, $\mathcal{L}_2$ and $\mathcal{L}_4$ on the other), and that for the P3L problem the singularity loci is a surface. The refinement from VVS generally allows reducing the translation error to below $0.1$~mm, except very near the singular point, where a persistent error of about $3$~cm remains.

In Ex. $3$, both the RP$n$L and VVS fail catastrophically to give an acceptable estimation near the singularity (see Fig.~\ref{fig:sim_PE_PM1_err}). As the camera approaches the line $\mathcal{L}_M$, the translation error blows up by several orders of magnitude ($\sim 10^{23}$~m).

These results show that a rank-defficiency of the interaction matrix can significantly impact the accuracy of pose computation methods, even in the case of direct solvers such as RP$n$L which do not explicitly involve the interaction matrix.

\subsection{Singularities in P5L}

\begin{figure}[t]
	\centering
	\includegraphics[width=90mm, height=64mm]{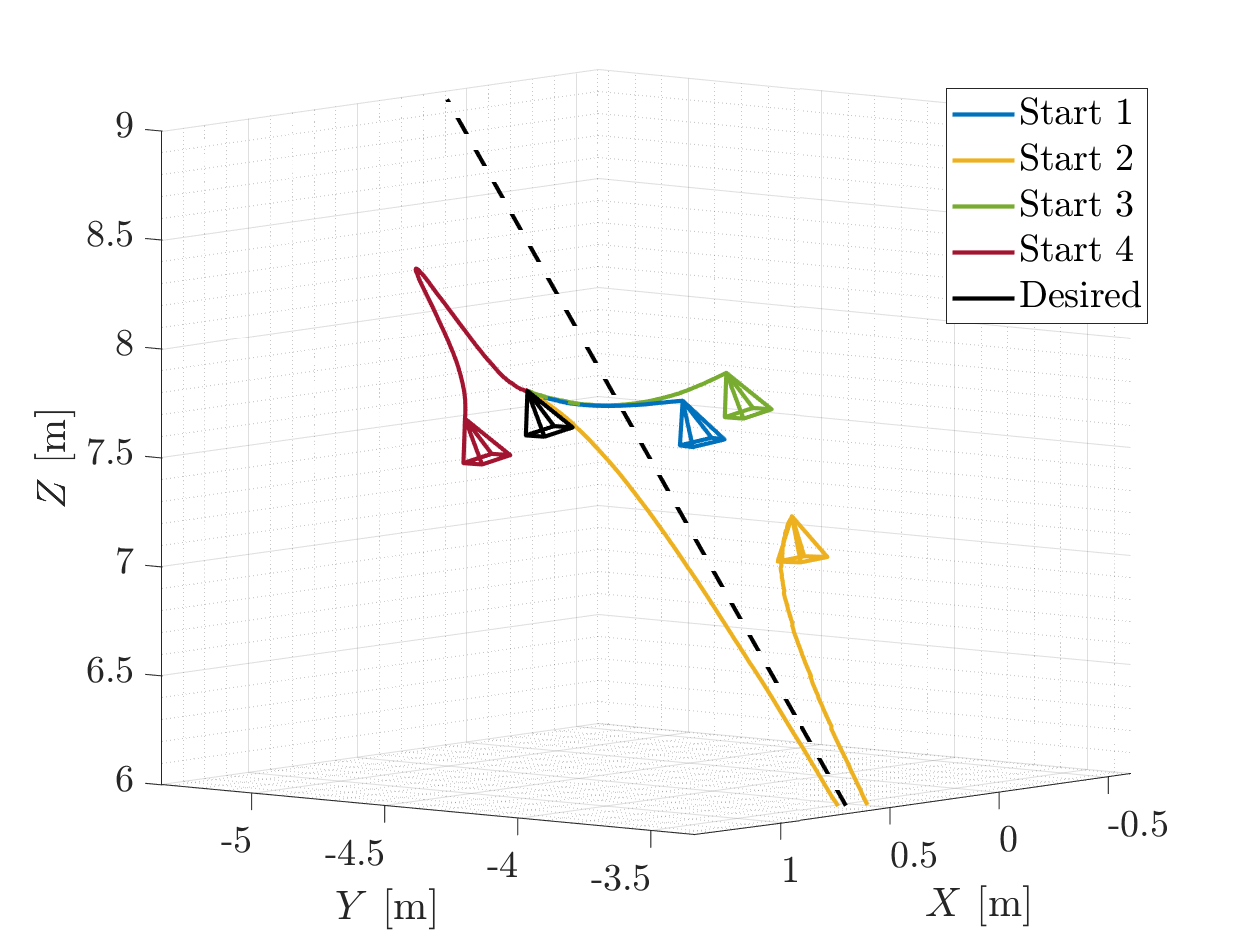}
	\caption{Visual servoing from five lines starting from different positions (coloured) towards a desired pose (black). The line $\mathcal{L}_M$ (dashed line) that intersects all five lines is a singularity of the interaction matrix.\label{fig:5lines_simulations_traj}}
\end{figure}

\begin{figure}[t]
	\centering
	\includegraphics[width=90mm, height=45mm]{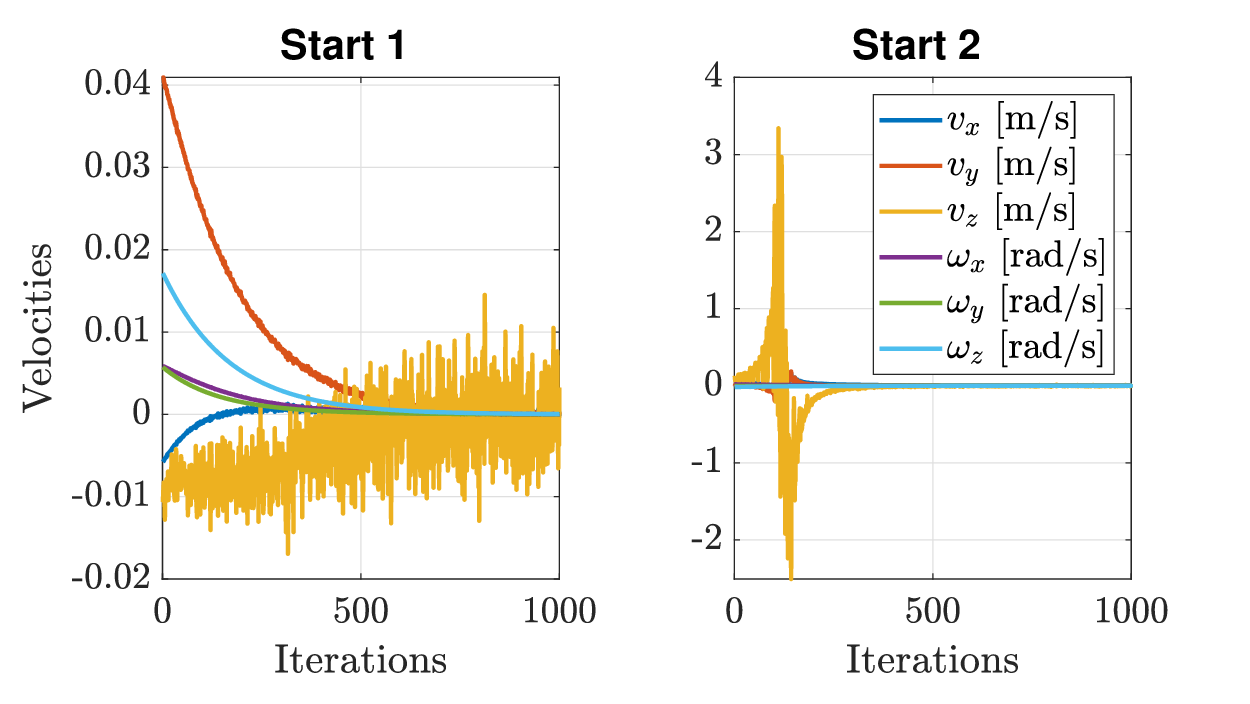}
	\caption{Camera velocity inputs $\boldsymbol{\tau}_c$ in a stable situation (left), and when crossing a singularity (right). \label{fig:5lines_vels}}
\end{figure}

\begin{figure}[t]
	\centering
	\includegraphics[width=90mm, height=45mm]{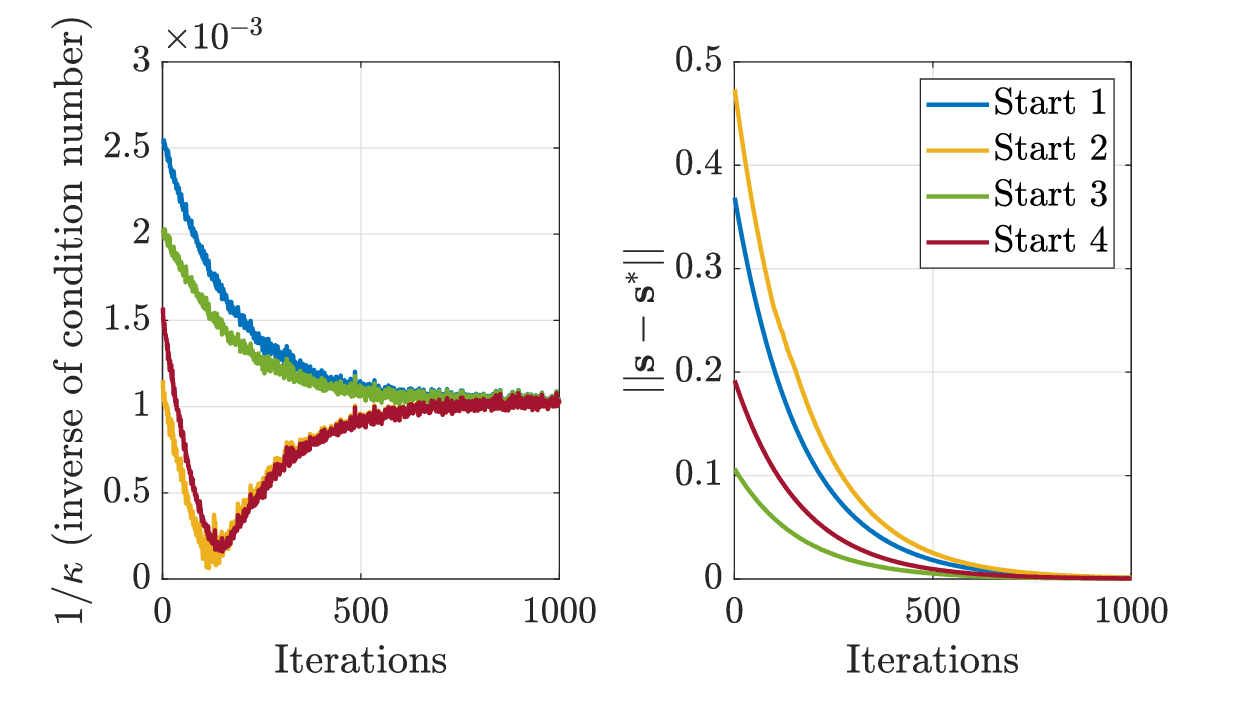}
	\caption{Inverse of the condition number $\kappa$ of the interaction matrix $\mathbf{M}_{(4)}$ (left) and norm of the error vector $||\mathbf{s}-\mathbf{s}^*||$ (right). \label{fig:5lines_errors_condnum}}
\end{figure}

We now consider the case of five lines determined by the following parameters:
\begin{equation} \label{eqs:sims_5lines_params}
\small
\begin{aligned} 
& s_1=4,\ s_2=-5, \ s_3=7, \ s_4=3, \ s_5=-2,\ s_6=13, s_7=-11,\ s_8=6,\nonumber\\
& \ d_1=2, \ d_2=3, \ d_3=5, \ d_4=1, \ d_5=7, \ d_6=-4,  \nonumber\\
& d_7= {\frac {128893236}{7630285}}-{\frac {24\,\sqrt {8508173023861}}{
		7630285}} \approx 7.7178.
\end{aligned}
\end{equation}

For this configuration, there is only one transversal line $\mathcal{L}_M$ that intersects all five lines (see Appendix B), defined by its Pl\"ucker coordinates:
\begin{equation}
\mathcal{L}_M = [0.0830 \ 0.4989 \ -0.8627 \ 0 \
-0.9641 \ -0.5575].
\end{equation}

From the results in Section~\ref{subsec:VG}, we know that a singularity will occur when the camera lies on this line. Note that the lines $\mathcal{L}_1, \dots, \mathcal{L}_4$ are defined identically as in Section~\ref{sec::sims_P4L}, and that $\mathcal{L}_5$ is chosen so as to intersect the first of the lines of singularity~\eqref{eq::sims_P4L_transversal_lines}.

\subsubsection{VS near the singularity line}

We considered a point $\mathcal{P}_M = [0.3962~\text{m} \ -4.337~\text{m} \ 7.50~\text{m}]$ that lies on the transversal $\mathcal{L}_M$ and defined four starting camera positions in its surroundings. From each of these points, we performed a simulation of visual servoing towards a target point. The initial and desired positions are the same as those considered in Section~\ref{section:simulation_line_LM}, whose coordinates relative to point $\mathcal{P}_M$ are given in Table~\ref{tab:simulation_points}. The camera velocity inputs are computed according to~\eqref{eq:velocity_input} with the factor $\lambda = 0.1$. Gaussian noise of standard deviation $\sigma = 2 \cdot 10^{-3}$ was added to the 3D coordinates of the lines.


In all four experiments, the magnitude of the error vector decreases exponentially (Fig.~\ref{fig:5lines_errors_condnum}). However, the trajectories displayed in Fig.~\ref{fig:5lines_simulations_traj} show very different behaviours. Starts $1$ and $3$ converge rapidly, with small camera velocities and displacements.

On the contrary, Starts $2$ (directly opposed to desired point) and $4$ (further away from the singularity), lead to unstable motion and large deviations. The velocity input profiles for Starts $1$ (stable) and $2$ (unstable) are compared in Fig.~\ref{fig:5lines_vels}. In the latter, the inputs become very high in magnitude as the inverse of the condition number of the interaction matrix drops near zero, when the camera crosses $\mathcal{L}_M$. Again we notice that there exists a region of instability around the singularity line, particularly strong in one direction which affects the trajectory of Start $4$, where the interaction matrix becomes ill-conditioned.

\subsubsection{Pose Estimation near the singularity line}

\begin{figure}[t]
	\centering
	\includegraphics[width=90mm, height=64mm]{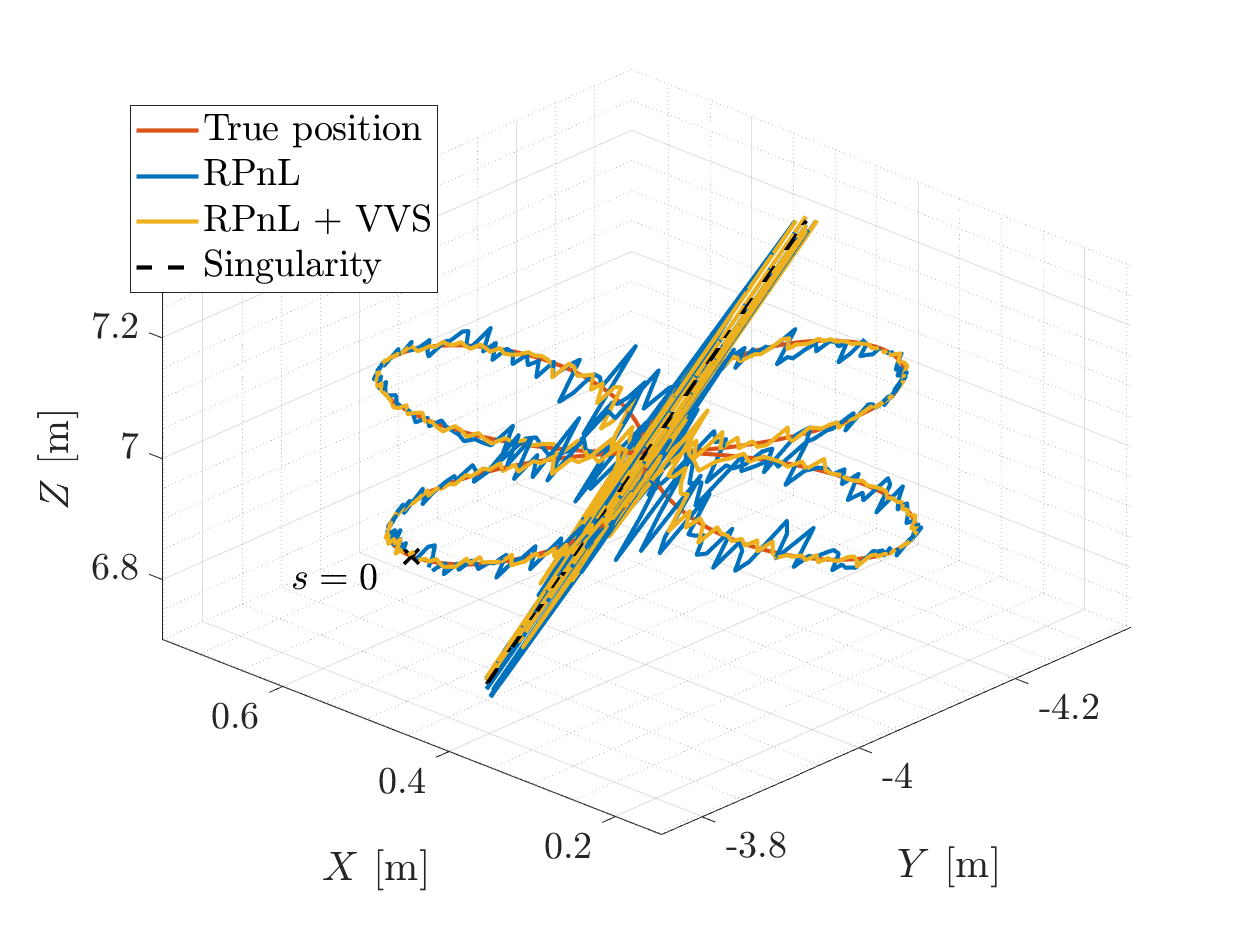}
	\caption{Pose computation from RP$n$L using five image lines along a quadrifolium trajectory centered at a point on the line singularity $\mathcal{L}_M$, and its refinement from VVS. In the proximity of the singularity the error in the estimation grows unbounded.}
	\label{fig:PE_5lines_traj}
\end{figure}

\begin{figure}[t]
	\centering
	\includegraphics[width=90mm, height=77mm]{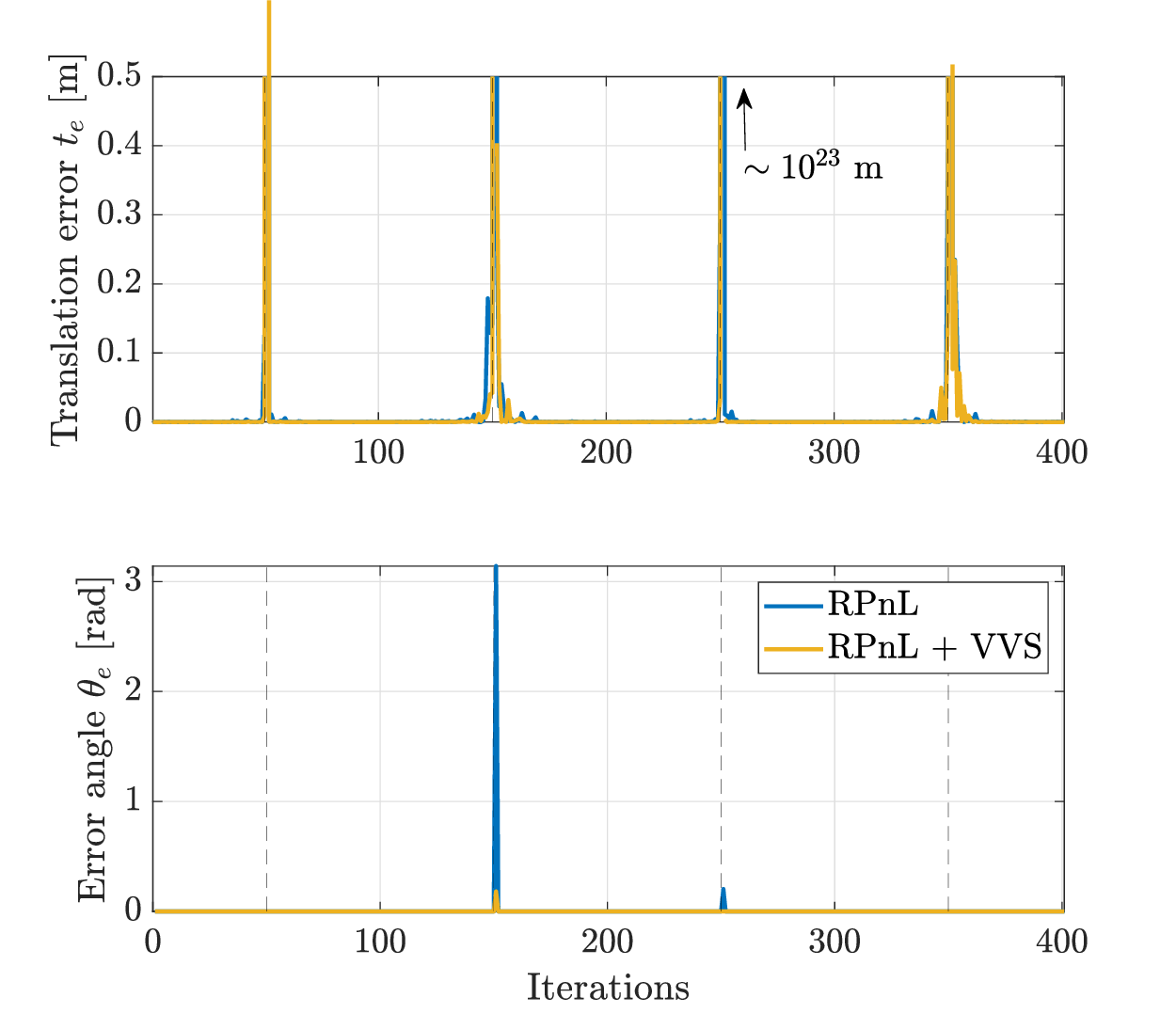}
	\caption{Translation (top) and rotation (bottom) errors in pose estimation from five lines along a quadrifolium trajectory. The vertical steps indicate where the camera crosses the singularity.}
	\label{fig:PE_5lines_error}
\end{figure}

For completeness, we conclude by demonstrating the behaviour of pose estimation from the observation of five lines near a singularity. We simulated a trajectory along a pattern with the shape of a quadrifolium (with equations~\eqref{eq:quadrifolium}) centered at a point with coordinates $[0.7809~\text{m} \, -2.024~\text{m} \, 3.50~\text{m}]$, which lies on the singularity line $\mathcal{L}_M$. Along this trajectory, the camera pose was computed using the RP$n$L algorithm~(\cite{zhang2012robust}) and the estimation was further refined using VVS~(\cite{Marchand_CGF_2002}). We considered Gaussian noise of standard deviation $\sigma = 5 \times 10^{-5}$ on the $3D$ coordinates of the observed lines.

The true pose is shown along the estimations from both methods in Fig.~\ref{fig:PE_5lines_traj}, while the translation error $t_e$ and the rotation error $\theta_e$ defined by~\eqref{eq:PE_error_angle}, are depicted in Fig.~\ref{fig:PE_5lines_error}. The observed behaviour is very similar to the case of four lines when pose reconstruction is performed near the singularity line: along the leaves of the quadrifolium, the RP$n$L presents a certain sensitivity to numerical noise, which is mitigated by the refinement through VVS. However, very near the singularity the errors blow up in magnitude and both methods prove totally unsuccessful. {For comparison, along a similar trajectory, but centered at coordinates $[0.5~\text{m} \, 4.0~\text{m} \, 1.0~\text{m}]$, far from any singularities, both methods are accurate up to $0.1$~mm in the estimation of position and up to $2\cdot 10^{-3}$~rad in orientation throughout.}


\section{{Conclusions}}
\label{sec:conclu}

In this paper, the singularities in the perspective four and five line problems were determined. Finding these singularities is crucial since they lead to controllability issues in visual servoing of image-lines and in large errors in pose estimation for P$n$L. To do so, a basis of the interaction matrix was found such that its rows are Pl\"ucker coordinates of $n$ affine and $n$ ideal lines for P$n$L.

First, it was recalled that the singularities in P$3$L are due to the vanishing of the determinant of the $(6 \times 6)$ interaction matrix which factors as a quadric and a cubic surface in terms of the position coordinates of the optical center of the camera. It was then proved that the quadric surface is essentially the hyperboloid of one sheet uniquely defined by the three observed lines. 

This fact was further used in the case of P4L to understand different cases such that the 28 principal minors of the $(8 \times 6)$ interaction matrix vanish simultaneously, leading to singularities. One of those cases is when the camera lies on two transversals intersecting the observed four lines. It was also shown that this one dimensional singularity could be avoided by choosing the fourth line to not intersect the hyperboloid defined by the other three lines. Additionally, Gr\"obner basis computations were used to expect up to ten real isolated singularities of P$4$L. 

Furthermore, in the case of P$5$L, no singularities were found for a generic choice of five lines. Nonetheless, some conditions of the relative configurations of the five lines were shown to yield a transversal line of singularities that intersects the five observed lines. 

The same analysis was done for four and five lines that are constrained to be orthogonal or parallel to each other to corroborate the results for the generic case. It turned out that for P$4$L, the singularities are two transversal lines and 10 isolated points whereas for P$5$L, it is just a transversal line.

{The results are supported with experimental simulations of Visual Servoing control of a camera and of pose determinations algorithms from the observation of lines in the proximity of the singularities. As expected, the ill-conditioning of the interaction matrix near a singular point results in unstable behaviour of the control law from VS and in significant losses of accuracy for the pose estimation methods.}

The geometric interpretation of the one dimensional singularities of P4L and P5L was provided, by extending the result that a hyperboloid of one sheet is a singularity in the P3L case. These one dimensional singularities appeared when the affine lines in the interaction matrix are linearly dependent by being parallel to the same plane. In the future, we will try to obtain the isolated singularities in P$n$L for four and five generic lines along with their geometric interpretation.

This work has been partially funded by the French ANR project SESAME (funding ID: ANR-18-CE33-0011).

\begin{acknowledgements}
This work has been partially funded by the French ANR project SESAME (funding ID: ANR-18-CE33-0011).
\end{acknowledgements}

\appendix
\appendix
\section*{Appendices}
\subsection*{A. Analysis of varieties $\mathbf{V}(\mathcal{G})$ and $\mathbf{V}(\mathcal{H}) \cap \mathbf{V}(\mathcal{K})$ in P4L for an example.}
\label{app:appendixA}
The singularities in P4L are determined for lines whose Pl\"ucker coordinates are arbitrarily chosen according to the following parameters:
\begin{align} \label{params}
    & s_1=4,\ s_2=-5, \ s_3=7, \ s_4=3, \ s_5=-2,\ s_6=13, \nonumber \\
    & \ d_1=2, \ d_2=3, \ d_3=5, \ d_4=1, \ d_5=7.
\end{align}
Then,
\begin{dmath*}
    \mathcal{G}= \langle -1765\,{X}^{2}-587\,XY-878\,XZ+660\,{Y}^{2}+232\,YZ-122\,{Z}^{2}+2606
\,X+216\,Y+598\,Z-708,\; -177\,XY-27\,XZ+75\,{Y}^{2}+17\,YZ-8\,{Z}^{2}+
156\,Y+6\,Z, \; -130\,XY+40\,XZ-104\,{Y}^{2}-62\,YZ+10\,{Z}^{2}+188\,Y-20
\,Z, \;
-30\,XY+145\,XZ-24\,{Y}^{2}+11\,YZ+30\,{Z}^{2}+210\,Y-60\,Z \rangle.
\end{dmath*}
The Gr\"obner basis $gb_G$ of $\mathcal{G}$ w.r.t. $Z\succ_{lex} Y \succ_{lex} X$ consists of four elements $\{g_1,g_2,g_3,g_4\}$. The first element factors (using the command \verb|evala(AFactor)| in Maple) as follows:
\begin{align*}
    g_1=-2166Y^2-1166YZ & +50Z^2\\
    & =\ \frac{1}{2166}\left( Z (\sqrt {448189}+583)+2166Y \right) \\
    & \ \left( Z(\sqrt 
{448189}-583)-2166Y \right).
\end{align*}
By substituting the factors into the other elements of $gb_G$, $\mathcal{G}$ can be decomposed into two subideals whose varieties correspond to the two transversal lines intersecting all the four observed lines:
\begin{align*}
    \mathcal{G} & =\mathcal{M} \cap \mathcal{N}, \textrm{ where }\\
    \mathcal{M} & =\langle Z(\sqrt{448189}+583)+2166Y, \\
    & \quad Z(329\sqrt {448189}-587953)+2166\sqrt {
448189} \\
  & \quad -3822990X+2822298  \rangle, \\
      \mathcal{N} & =\langle Z(\sqrt{448189}-583)-2166Y, \\
    & \quad Z(329\sqrt {448189}+587953)+2166\sqrt {
448189} \\
  & \quad +3822990X-2822298  \rangle.
\end{align*}
Following Section~\ref{subsec:VG}, it is straightforward to verify that $\overline{p_i}^{\langle g_1, \dots, g_4\rangle} = 0 \  \forall \ p_i \in \mathcal{K}$. So, the corresponding positive dimensional singularities are $\mathbf{V}(\mathcal{M})$ and $\mathbf{V}(\mathcal{N})$. 

The ideals $\mathcal{H}$ and $\mathcal{K}$ are determined whereas only $\mathcal{H}$ is displayed here since $\mathcal{K}$ is too large:
\begin{dmath}
   \mathcal{H}=\langle  -73960\,{X}^{3}-46428\,Y{X}^{2}+320426\,{X}^{2}Z+88867\,X{Y}^{2}+
163934\,YXZ+184389\,X{Z}^{2}+62940\,{Y}^{3}-356381\,{Y}^{2}Z-32282\,Y{
Z}^{2}+27183\,{Z}^{3}+210018\,{X}^{2}-721747\,XY-416981\,XZ-146898\,{Y
}^{2}-118097\,YZ-116973\,{Z}^{2}+111106\,X+377082\,Y+153504\,Z-56580,-
3038\,Y{X}^{2}+3686\,{X}^{2}Z-2288\,X{Y}^{2}+16544\,YXZ+3344\,X{Z}^{2}
-315\,{Y}^{3}-4111\,{Y}^{2}Z-157\,Y{Z}^{2}+663\,{Z}^{3}+27166\,XY-4168
\,XZ+2769\,{Y}^{2}-13942\,YZ-1527\,{Z}^{2}-25806\,Y+690\,Z,-650\,Y{X}^
{2}-3350\,{X}^{2}Z-195\,X{Y}^{2}+8450\,YXZ-845\,X{Z}^{2}+260\,{Y}^{3}+
3410\,{Y}^{2}Z+390\,Y{Z}^{2}-13390\,XY+1630\,XZ+276\,{Y}^{2}-8372\,YZ+
15088\,Y,225\,Y{X}^{2}-1685\,{X}^{2}Z+705\,X{Y}^{2}+450\,YXZ-345\,X{Z}
^{2}+420\,{Y}^{3}-1325\,{Y}^{2}Z-645\,Y{Z}^{2}-8056\,XY+705\,XZ-918\,{
Y}^{2}-1527\,YZ+5658\,Y \rangle
\end{dmath}
The Gr\"obner basis $gb_{H}$ of $\mathcal{H}$ w.r.t. $Z\succ_{lex} Y \succ_{lex} X$ contains 5 elements $\{h_1,h_2,h_3,h_4,h_5\}$. As proposed in Section~\ref{subsec:VH}, the normal forms can be calculated as $\overline{p_i}^{\langle h_1, \dots, h_5\rangle} = f_i$. The residuals $f_i$ do not vanish. Thus, the whole variety $\mathbf{V}(\mathcal{H}) \cap \mathbf{V}(\mathcal{K})$ is considered and it turns out to be of dimension $0$ and degree $22$ with 16 real solutions (see the attached Maple file):
\begin{table}[ht!]
			\caption{Elements of the variety $\mathbf{V}(\mathcal{H}) \cap \mathbf{V}(\mathcal{K})$.} \label{table:exVHK}
\centering
			\begin{tabular}{c| ccc}
				\toprule
			 & X & Y & Z\\
			 \midrule
			 *1 &-9.858 & - 2.473 & - 1.841 \\
			 \hline
			 2 & -0.720 &  0.0 &  0.0 \\
			 \hline
		 *3 & -0.320 &  0.010 &  0.220 \\
			 \hline
			 4 &-0.007 &  0.009 &  2.0 \\
			 \hline
		 *5 & 0.054 &  0.009 &  1.842 \\
			 \hline
			 6 & 0.328 &  0.0 &  0.0 \\
			 \hline
			 7 & 0.918 &  0.141 & - 2.082 \\
			 \hline
		 *8 & 0.938 &  0.568 & - 2.023 \\
				\hline
				9 & 0.972 & - 1.215 &  2.0 \\
				\hline
				 *10 & 1.011 &  0.794 & - 0.885 \\
				\hline
				11 & 1.016 &  0.371 & - 1.984 \\
				\hline
				12 & 1.218 &  0.390 & - 0.918 \\
				\hline
				13 & 3.231 &  0.0 &  0.0 \\
				\hline
				14 & 3.880 &  7.054 &  0.880 \\
				\hline
				 *15 & 65.09 & - 96.57 & - 0.036 \\
				\hline
				16 & 90.31 & - 112.9 &  2.0 \\
				\bottomrule
			\end{tabular}
\end{table}

However, it can be verified that some of these points lie on the observed four lines or their transversals $\mathbf{V}(\mathcal{M})$ and $\mathbf{V}(\mathcal{N})$. Since any point incident with these lines leads to a singularity, we are interested in singular points that do not lie on them. They can be calculated by determining the ideal $\mathcal{F}$ using~\eqref{eq:idealF}. The Gr\"obner basis $gb_F$ of $\mathcal{F}$ w.r.t. $Z\succ_{lex} Y \succ_{lex} X$ has the following nice structure~(called the \emph{shape position}):
\begin{align*}
    gb_F=\{f_a(Z), \ f_b(Z)+Y, \ f_c(Z)+X \},
\end{align*}
where $f_a(Z)=\Sigma_{i=1}^{10} a_iZ^i$ is a degree 10 univariate polynomial in $Z$, $f_b(Z)$ and $f_c(Z)$ are also univariate polynomials in $Z$. It follows that $\mathbf{V}(\mathcal{F})$ is of degree $10$. It consists of $6$ real points marked with an asterisk each in Table~\ref{table:exVHK}. Thus, the singularity loci for this example include the four observed lines, their two transversals and 6 points. 

Since the parameters were chosen randomly for this analysis, this 
indicates that for 
values of parameters, outside the zero set of some polynomial depending on the 
parameters (hence of measure zero), the real singular points in P4L can be up to 
$10$.

\subsection*{B. Analysis of varieties $\mathbf{V}(\mathcal{G})$ and $\mathbf{V}(\mathcal{H}) \cap \mathbf{V}(\mathcal{K})$ in P5L for an example.}
\label{app:appendixB}
The singularities in P5L are determined for lines whose Pl\"ucker coordinates are arbitrarily chosen as follows:
\begin{align} \label{paramsP5L}
    & s_1=4,\ s_2=-5, \ s_3=7, \ s_4=3, \ s_5=-2,\ s_6=13, s_7=-11, s_8=6, \nonumber \\
    & \ d_1=2, \ d_2=3, \ d_3=5, \ d_4=1, \ d_5=7, d_6=-4, d_7=11.
\end{align}
The Gr\"obner basis $gb_G$ of $\mathcal{G}$ w.r.t. $Z\succ_{lex} Y \succ_{lex} X$ is $\langle 1 \rangle$. By Hilbert's Nullstellensatz, $\mathbf{V}(\mathcal{G})=\emptyset$. 

If we choose the parameters according to~\eqref{paramsP5L} except $d_7$, which is chosen such that the parameters satisfy $h=0$ in Section~\ref{subsec:VG5}:
\begin{equation*}
   d_7= {\frac {128893236}{7630285}}-{\frac {24\,\sqrt {8508173023861}}{
7630285}},
\end{equation*}
then, calculating the Gr\"obner basis leads to:
\begin{align*}
    gb_G &=\{ 2166Y+ \left( \sqrt {448189}+583 \right) Z, \  -2166\sqrt {448189}-
2822298 \\ & \quad +3822990X+ \left( -329\sqrt {448189}+587953 \right) Z \}.
\end{align*}
It is the equation of a line intersecting the five observed lines.

The Gr\"obner basis $gb_H$ of $\mathcal{H}$ w.r.t. $Z\succ_{lex} Y \succ_{lex} X$ contains 3 elements $\{h_1,h_2,h_3,h_4,h_5\}$. As proposed in Section~\ref{subsec:VH5}, the normal forms can be calculated as $\overline{p_i}^{\langle h_1, \dots, h_3\rangle} = f_i,i=1,\dots,55$. The residues $f_i$ do not vanish. Therefore, he Gr\"obner basis of the ideal $\langle \mathcal{H}, \mathcal{K}_{55} \rangle$ is calculated and it turns out to be $\langle 1 \rangle$. Hence, for a generic choice of parameters, there are no singularities in P5L. 


\bibliographystyle{spbasic}      

\bibliography{bibreport}

\end{document}